\documentclass{article} 
\usepackage{iclr2027_conference,times}

\usepackage{amsmath,amsfonts,bm}

\def\eqref#1{equation~\ref{#1}}

\def\1{\bm{1}}

\DeclareMathAlphabet{\mathsfit}{\encodingdefault}{\sfdefault}{m}{sl}
\SetMathAlphabet{\mathsfit}{bold}{\encodingdefault}{\sfdefault}{bx}{n}

\usepackage{url}
\usepackage[table]{xcolor}

\definecolor{burgundy}{RGB}{128, 0, 32}
\definecolor{hotpink}{RGB}{10,  78,  139}

\usepackage[colorlinks,linkcolor=burgundy, anchorcolor=blue, citecolor=burgundy, urlcolor=hotpink]{hyperref}

\usepackage{amsmath}
\usepackage{amssymb}
\usepackage{amsfonts}
\usepackage{amsthm}
\usepackage{mathtools}
\usepackage{bm}
\usepackage{booktabs} 
\usepackage{tcolorbox}
\usepackage[svgnames]{xcolor}
\usepackage{multirow}
\usepackage{tabularx}
\usepackage{pifont}
\usepackage{adjustbox}
\usepackage{tablefootnote}
\usepackage{graphicx}   
\usepackage{subcaption} 
\usepackage{wrapfig}   
\usepackage{minitoc}

\usepackage{fontawesome5}
\usepackage{xcolor}

\definecolor{titlelink}{HTML}{1F5A8A}
\newcommand{\titlehref}[2]{%
  \href{#1}{\textcolor{titlelink}{#2}}%
}

\theoremstyle{plain}
\newtheorem{theorem}{Theorem}[section]
\newtheorem{proposition}[theorem]{Proposition}
\newtheorem{lemma}[theorem]{Lemma}

\theoremstyle{definition}
\newtheorem{definition}[theorem]{Definition}

\theoremstyle{remark}

\title{Do ResNets Route? Sparse Interaction Experts in Residual Networks}

\author{%
\makebox[0.96\textwidth][c]{%
  \textbf{Liang Yan}\textsuperscript{1}%
  \thanks{Corresponding author: \texttt{yanliangfdu@gmail.com}}%
  \hspace{0.7em}
  \textbf{Siying Chen}\textsuperscript{2}
  \hspace{0.7em}
  \textbf{Kaijie Chen}\textsuperscript{3}
  \hspace{0.7em}
  \textbf{Bo Li}\textsuperscript{1}
  \hspace{0.7em}
  \textbf{Jinghao Zhang}\textsuperscript{4}
  \hspace{0.7em}
  \textbf{Mu Miao}\textsuperscript{5}%
}
\\[0.35em]
\makebox[0.96\textwidth][c]{%
  {\normalfont\small
  \textsuperscript{1}Fudan University
  \qquad
  \textsuperscript{2}University of Washington
  \qquad
  \textsuperscript{3}Tongji University}%
}
\\[-0.05em]
\makebox[0.96\textwidth][c]{%
  {\normalfont\small
  \textsuperscript{4}Shandong University
  \qquad
  \textsuperscript{5}Datacanvas}%
}
\\[0.4em]
\makebox[0.96\textwidth][c]{%
  {\normalfont\small
  \titlehref{https://divinyan.com/resnets-route-without-routers}
  {\faGlobe\enspace Project Page}%
  \hspace{1.1em}%
  \titlehref{https://github.com/yanliang3612/resnets-route-without-routers}
  {\faGithub\enspace Code}}%
}
}

\iclrfinalcopy 
\begin{document}
\doparttoc 
\faketableofcontents

\maketitle

\begin{abstract}
Residual networks execute every block for every input, yet their functional
contributions need not be input independent. We formulate a trained ResNet as
a set function over binary residual-branch masks and apply M\"obius inversion
to decompose its output exactly into individual residual corrections and
higher-order interactions. For smooth residual stacks, we show that each fixed
$k$-way interaction scales as $\mathcal{O}(\lambda^k)$ under residual scaling.
Exhaustive analysis of ImageNet-pretrained ResNet-18 and ResNet-34 reveals
that interaction mass peaks at orders five and ten, respectively, rather than
at low orders. Reducing the residual scale shifts both spectra toward lower
orders, but also changes model predictions. The interaction coefficients are
concentrated in magnitude but not hard sparse, and prediction-preserving
sparsity weakens with depth. Crucially, the dominant interactions vary across
inputs and predicted classes around a shared global core, while their overall
complexity changes little with sample difficulty. These results show that
dense ResNets implement an implicit form of soft routing: every block is
executed, but different inputs rely on different residual interaction experts.
Routing can therefore emerge at the level of functional contribution without
an explicit router or sparse execution.
\end{abstract}

\vspace{-0.6cm}
\section{Introduction}
\vspace{-0.3cm}
\label{sec:introduction}

Standard ResNets execute every residual block for every input and contain no
explicit router, yet we show that their functional contributions exhibit
implicit input-dependent routing. Residual networks enable the reliable
optimization of deep models through identity shortcuts and residual
updates~\citep{he2016deep}. The influential path-ensemble view
of~\citet{veit2016residual} interprets each block as offering a shortcut and a
residual branch, yielding exponentially many structural paths of different
lengths. This perspective helps explain network redundancy, robustness to
block removal, and training methods such as stochastic
depth~\citep{huang2016deep}.

The path-ensemble view is structural rather than functional. A masked-network
response is the cumulative output of a surviving nonlinear subnetwork, not the
additive contribution of an isolated path. Path counting therefore does not
reveal how residual branches individually or jointly support a prediction, or
whether their dominant contributions vary across inputs. These questions
connect ResNets to conditional computation. Dynamic residual models explicitly
gate blocks~\citep{wang2018skipnet,wu2018blockdrop}, while sparse
Mixture-of-Experts models activate selected experts for each
input~\citep{shazeer2017outrageously,lepikhin2020gshard,fedus2021switch}.
Any routing in a standard ResNet must instead arise at the level of functional
contribution. We therefore ask: \emph{Do standard ResNets exhibit implicit
functional routing over residual interactions, despite having no explicit
router?} We characterize such routing by contribution concentration and
systematic input dependence of the dominant interactions.

To answer this question, we treat a trained ResNet as an input-conditioned set
function over binary residual-branch masks. Boolean M\"obius inversion yields
a unique, exact decomposition into a shortcut-only baseline, individual
residual corrections, and higher-order
interactions~\citep{rota1964foundations,grabisch2000equivalent,kang2024mobius}.
We call high-contribution terms \emph{implicit residual experts}. Ranking these
experts defines an input-dependent mapping from $x$ to a dominant interaction
set $\mathcal{A}_K(x)$, which acts as an implicit functional router. The gates
are analytical interventions, so ordinary execution remains dense.
Figure~\ref{fig:framework-overview} schematically summarizes this framework:
the network remains densely executed, while M\"obius inversion reveals an
input-dependent distribution of functional contributions over residual
interaction experts. For an idealized smooth residual stack, we further show
that each fixed $k$-way interaction coefficient is
$\mathcal{O}(\lambda^k)$ as the residual scale $\lambda\to0$, predicting
higher-order attenuation in the small-residual regime.

We exhaustively analyze ImageNet-pretrained ResNet-18 and ResNet-34, whose
eight and sixteen residual blocks induce $255$ and $65{,}535$ non-constant
interaction terms. M\"obius inversion reconstructs the full logits to numerical
precision, whereas naive aggregation of masked responses fails. At the natural
operating point, the squared-coefficient spectra peak at orders five and ten,
showing that short structural paths do not imply low-order functional
interactions. Residual scaling shifts both spectra toward lower orders but also
substantially changes predictions.

Our central finding is that dense execution does not imply a fixed functional
pathway. In both models, each input induces a different allocation of
contribution over a shared dictionary of residual interactions. Dominant
experts vary systematically across inputs, predicted-class signatures are
strongly class-structured, and full-logit signatures retain a shared global
core. This establishes implicit functional routing: ResNets reuse common
interaction experts while changing which ones dominate each prediction. The
routing is soft rather than hard because smaller interactions remain decision
relevant, especially in ResNet-34, while expert complexity changes little with
input difficulty. Standard ResNets therefore route functional responsibility
across inputs despite executing every residual block.

Our contributions are: (1) an input-conditioned residual-mask framework that
provides a unique and exact decomposition into individual corrections and
higher-order interactions; (2) a coefficient-wise
$\mathcal{O}(\lambda^k)$ connection between residual scale and interaction
strength; and (3) empirical evidence that pretrained ResNets exhibit
class-structured, readout-dependent implicit routing, organized around a
shared global core and input-specific expert emphasis. Together, these results
show that routing can emerge at the level of functional contribution without
an explicit router or sparse execution.


\vspace{-0.3cm}
\section{Background and Theoretical Framework}
\vspace{-0.3cm}

To determine whether a trained ResNet exhibits input-dependent, routing-like specialization, we distinguish structural paths from functional contributions. The path-ensemble view identifies which computational paths exist, but not how residual branches jointly affect a particular prediction. For each input, we
therefore treat the network as a set function over binary residual-branch masks and use M\"obius inversion to decompose its output exactly into individual corrections and higher-order interactions. We then test whether a small subset faithfully preserves the output and whether the dominant interactions vary
systematically across inputs. Evidence for both properties is required to support a limited routing interpretation at the level of functional contribution rather than sparse execution.

\vspace{-0.3cm}
\subsection{Residual-Mask Responses and M\"obius Interaction Supports}
\vspace{-0.3cm}
\textbf{Residual-Mask Responses.} Let $[L]=\{1,\ldots,L\}$ index the residual blocks. For a mask $m\in\{0,1\}^L$, block $l$ computes $x_{l+1}(m)=\phi_l\!\left(s_l(x_l(m))+m_lF_l(x_l(m))\right)$, where $s_l$ is the identity or projection shortcut and $\phi_l$ is the post-addition activation. Thus, $m_l=0$ suppresses only the residual branch while preserving the shortcut. For a fixed input $x$, these interventions define the residual-mask response function $h_x:\{0,1\}^L\rightarrow\mathbb{R}^C$, where $h_x(m)$ is the resulting logit vector and $C$ is the number of classes. For scalar analysis, we define $v_x(m)=\rho_x(h_x(m))$, where the readout map $\rho_x$ is fixed across all
masks and may, for example, select the class predicted by the full model.

\begin{figure}[t]
  \centering
  \vspace{-1cm}
  \includegraphics[width=\textwidth]{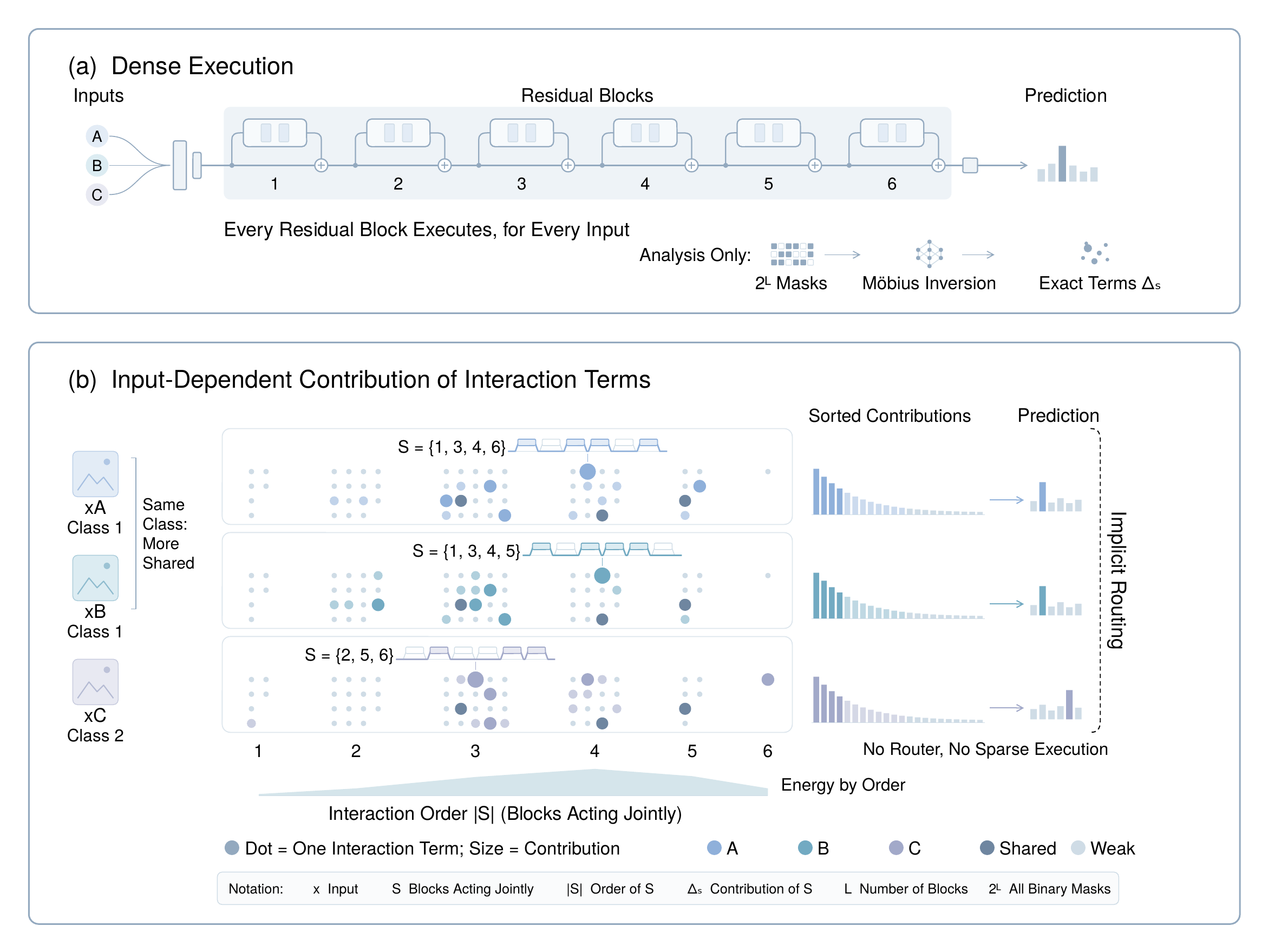}
  \vspace{-0.8cm}
\caption{From dense residual execution to input-dependent interaction
experts. (a) A standard ResNet executes every residual block for every input
and contains no explicit router. For analysis, we apply binary gates to its
$L$ residual branches, evaluate all $2^L$ masks, and use M\"obius inversion to
decompose the output exactly into interaction coefficients $\Delta_S$, one for
each branch subset $S$. These gates are analytical interventions and do not
alter ordinary dense inference. (b) Schematic of the resulting functional
decomposition for different inputs. Each dot represents an interaction,
positioned by its order $|S|$ and sized by its predicted-class contribution
$|\Delta_S v_x|$; the callouts identify the branches participating in the
strongest interactions. The dominant terms combine a shared global core with
input-dependent, class-structured variation. We refer to this concentration
and variation of functional contributions as \emph{implicit soft routing},
which is distinct from sparse execution or an explicit learned router.}
\label{fig:framework-overview}
\vspace{-0.2cm}
\end{figure}

\textbf{M\"obius Residual Interactions.} For any subset $S\subseteq[L]$, let $\mathbf{1}_S\in\{0,1\}^L$ denote the mask that activates exactly the residual branches in $S$. We apply the standard M\"obius transform on the Boolean lattice of residual-branch subsets~\citep{rota1964foundations,grabisch2000equivalent,kang2024mobius}. The resulting coefficient for $S$ removes, through inclusion--exclusion, the effects already present in its proper subsets. We refer to these coefficients as residual interaction coefficients.

\begin{definition}[M\"obius residual interaction coefficient]
\label{def:mobius-residual-interaction}
For a fixed input $x$ and scalar residual-mask response function $v_x:\{0,1\}^L \to \mathbb{R}$, the M\"obius residual interaction coefficient associated with a subset $S \subseteq [L]$ is defined as $\Delta_S v_x = \sum_{T \subseteq S} (-1)^{|S|-|T|}v_x(\mathbf{1}_T)$.
\end{definition}
Intuitively, $\Delta_S v_x$ isolates the functional contribution that is uniquely attributable to the joint activation of the residual branches in $S$, after all lower-order contributions from proper subsets of $S$ have been removed.

\begin{lemma}[M\"obius inverse relation]
\label{lem:mobius-inverse}
For every subset $U \subseteq [L]$, the scalar residual-mask response function satisfies $v_x(\mathbf{1}_U) = \sum_{S \subseteq U} \Delta_S v_x$. 
\end{lemma}

The proof of Lemma~\ref{lem:mobius-inverse} is provided in Appendix~\ref{app:mobius-proof}. Lemma~\ref{lem:mobius-inverse} shows that the coefficients in Definition~\ref{def:mobius-residual-interaction} form an exact inclusion--exclusion decomposition of the residual-mask response function. In particular, for the full network, $v_x(\mathbf{1}_{[L]}) = \sum_{S \subseteq [L]} \Delta_S v_x$. Applying the same transform to each logit coordinate gives the vector coefficients $\Delta_S h_x\in\mathbb{R}^C$, which satisfy $h_x(\mathbf{1}_{[L]})=\sum_{S\subseteq[L]}\Delta_S h_x$. We interpret first-order terms as residual corrections. For a single residual branch $l$,
$\Delta_{\{l\}} v_x = v_x(\mathbf{1}_{\{l\}}) - v_x(\mathbf{0})$, which measures the functional effect of activating branch $l$ relative to the shortcut-only baseline. Higher-order terms correspond to \emph{residual interactions}. For example, for two residual branches $i$ and $j$, $\Delta_{\{i,j\}} v_x = v_x(\mathbf{1}_{\{i,j\}}) - v_x(\mathbf{1}_{\{i\}}) - v_x(\mathbf{1}_{\{j\}}) + v_x(\mathbf{0})$. This quantity measures the non-additive effect produced by using branches $i$ and $j$ together. If $\Delta_{\{i,j\}}v_x$ is large, then the joint effect of the two residual branches cannot be explained by the sum of their individual corrections.

\textbf{Sparse Residual Experts.} Let $\mathcal{R} = \{S\subseteq[L]:S\neq\emptyset\}$ denote the dictionary of non-empty residual interactions. We define the scalar and vector contribution magnitudes of an interaction $S\in\mathcal{R}$ as $a_S^v(x)=|\Delta_S v_x|$ and $a_S^h(x)=\|\Delta_S h_x\|_2$, and write $a_S^r$ with $r\in\{v,h\}$ for either. A subset $S$ with large $a_S^r(x)$ acts as an \emph{implicit residual expert} for input $x$: it is a residual correction or interaction that substantially affects the model's output under the chosen readout. The effective expert set consists of the top-$K$ interactions, $\mathcal{A}_K^r(x)=\operatorname{TopK}_{S\in\mathcal{R}}a_S^r(x)$. We say that the residual-mask response admits a sparse M\"obius representation for input $x$ if a small set of residual interactions explains most of the contribution:
\begin{equation}
    v_x(\mathbf{1}_{[L]})
    \approx
    \Delta_{\emptyset}v_x
    +
    \sum_{S\in\mathcal{A}_{K}^v(x)}
    \Delta_S v_x,
    \qquad
    K\ll|\mathcal{R}|.
\end{equation}

This sparsity is not execution sparsity. A standard ResNet still evaluates all residual blocks in the ordinary forward pass. Rather, it is effective contribution sparsity: among exponentially many possible residual corrections and interactions, only a small subset may dominate the functional contribution for a particular input.

\textbf{Interaction-Order Statistics.} For either readout $r\in\{v,h\}$, we define the squared coefficient mass at order $k$ as $E_k^r(x)=\sum_{S\in\mathcal{R}:\,|S|=k}a_S^r(x)^2$ and the total non-constant mass as $Z^r(x)=\sum_{k=1}^{L}E_k^r(x)$. When $Z^r(x)>0$, the normalized order spectrum is $\widetilde E_k^r(x)=E_k^r(x)/Z^r(x)$, and the effective interaction order is $\kappa^r(x)=\sum_{k=1}^{L}k\,\widetilde E_k^r(x)$. These statistics summarize the concentration and order distribution of the M\"obius coefficient mass. They do not by themselves measure reconstruction fidelity or establish input-dependent routing.

\textbf{Input-dependent Implicit Routing.} In sparse MoE models, a router selects a small set of experts for each input. A standard ResNet has no explicit router. Nevertheless, the residual expert set $\mathcal{A}_K^r(x)$ is input-dependent because the residual-mask responses depend on the input representation at every block. Thus, two inputs $x_i$ and $x_j$ may have different effective expert sets: $\mathcal{A}_K^r(x_i)\neq\mathcal{A}_K^r(x_j)$. This motivates an implicit routing interpretation: the input does not explicitly choose which residual experts to execute, but it induces which residual corrections and interactions are functionally active. We can measure this input dependence by comparing top-$K$ expert sets across samples: $J_K^r(x_i,x_j)=|\mathcal{A}_K^r(x_i)\cap\mathcal{A}_K^r(x_j)|/|\mathcal{A}_K^r(x_i)\cup\mathcal{A}_K^r(x_j)|$. If samples from the same class or with similar visual structure have higher overlap than unrelated samples, then residual interactions exhibit a form of input-dependent specialization. This is the sense in which a trained ResNet may behave like an implicit soft MoE: not by sparse execution, but by sparse, input-dependent functional contribution.

\vspace{-0.3cm}
\subsection{Low-Order Bias from Residual Scaling}
\vspace{-0.3cm}
We also note an interesting theoretical phenomenon: lower-order interactions may dominate as the residual scale decreases. Consider a scaled residual network
\begin{equation}
x_{l+1}^{(\lambda)}(m)
=
x_l^{(\lambda)}(m)
+
\lambda m_l F_l \left(x_l^{(\lambda)}(m)\right),
\qquad
m\in[0,1]^L,
\quad
0 < \lambda \leq \lambda_0 .
\label{eq:scaled-residual-network}
\end{equation}
For small $\lambda$, each residual branch acts as a small perturbation to the shortcut trajectory. An interaction among $k$ branches arises from the joint effect of $k$ such perturbations, suggesting that higher-order interactions become increasingly attenuated as $\lambda$ decreases.

\begin{proposition}[Residual scaling suppresses high-order interactions]
\label{prop:residual-scaling}
Let $v_x^{(\lambda)}(m)$ be the scalar residual-mask response function induced by the scaled network in Equation~\ref{eq:scaled-residual-network}. Suppose that, for the fixed input $x$, the residual maps $F_l$ and the final scalar readout are sufficiently smooth in a neighborhood of the shortcut trajectory, and that their derivatives up to order $L$ are uniformly bounded for $0<\lambda\leq\lambda_0$. Then for every subset $S\subseteq[L]$ with $|S|=k$, there exists a constant $B_{S,x}>0$ independent of $\lambda$ such that $\left|\Delta_S v_x^{(\lambda)}\right|\leq B_{S,x}\lambda^k$. The proof is given in Appendix~\ref{app:residual-scaling-proof}. 
\label{proposition:residual-scaling-bound}
\end{proposition}
Proposition~\ref{prop:residual-scaling} suggests that residual networks are naturally biased toward low-order residual interactions when residual updates are small. Depth therefore does not merely create long computational chains; it creates a large dictionary of possible residual interactions, while the effective contribution may remain concentrated in a sparse, not high-order subset.


\vspace{-0.3cm}
\section{From Structural Paths to Functional Interactions}
\vspace{-0.3cm}

This section empirically tests functional interactions behind the structural path view. We first show that masked-subnetwork outputs are not additive path contributions, while their Möbius coefficients exactly reconstruct the logits. We then measure interaction energy across residual-branch orders and find that pretrained ResNets concentrate mass at mid-to-high orders. Finally, we scale residual branches at inference time to test whether smaller residual updates lower interaction order and how much they change predictions.

\textbf{Experimental Setup.} We use ImageNet-pretrained ResNet-18 and ResNet-34 architectures~\citep{he2016deep}, treating their residual blocks as maskable residual branches. ResNet-18 has $L=8$ such branches and is our main model; ResNet-34 has $L=16$ branches and is used to verify that the masked-output reconstruction phenomenon is not specific to ResNet-18. Each branch receives an externally specified binary mask $m_l\in\{0,1\}$, so block $l$ computes $x_{l+1}(m)=x_l(m)+m_lF_l\!\left(x_l(m)\right)$. The shortcut pathway, including projection shortcuts, is always preserved; thus, $m_l=0$ removes only the residual branch. For ResNet-18, we evaluate a class-balanced subset of the ImageNet-2012 validation set~\citep{russakovsky2015imagenet} with $10$ images per class ($N=10{,}000$) under all $2^8=256$ residual-branch masks. For ResNet-34, we evaluate $N=1{,}000$ images under all $2^{16}=65{,}536$ masks for the reconstruction diagnostic. For $S\subseteq[L]$, $\mathbf{1}_S\in\{0,1\}^L$ denotes the mask activating exactly the branches in $S$, and $h_x(\mathbf{1}_S)\in\mathbb{R}^{1000}$ denotes the corresponding pre-softmax logits. We write $h_x^{\mathrm{full}}=h_x(\mathbf{1}_{[L]})$ and $h_x^{\emptyset}=h_x(\mathbf{0})$.

\textbf{Observation~\ding{182}: Masked-subnetwork outputs are not additive path contributions.}
Enumerating structural paths does not imply that the corresponding masked-subnetwork outputs form an additive decomposition of the network prediction. If each masked output represented an independent path contribution, simple aggregations across masks should recover the full-network logits. We therefore ask: \textit{Can direct summation, averaging, or shortcut-baseline centering of masked-subnetwork outputs reconstruct the full ResNet logits?} We consider three naive aggregations: the raw subset sum $A_x^{\mathrm{raw}}=\sum_{S\subseteq[L]}h_x(\mathbf{1}_S)$, the mean subset output $A_x^{\mathrm{mean}}=2^{-L}A_x^{\mathrm{raw}}$, and the centered subset sum $A_x^{\mathrm{ctr}}=h_x^{\emptyset}+\sum_{\emptyset\neq S\subseteq[L]}(h_x(\mathbf{1}_S)-h_x^{\emptyset})$. We compare them with the M\"obius reconstruction $A_x^{\mathrm{Mobius}}=\sum_{S\subseteq[L]}\Delta_S h_x$, which is computed from the same masked evaluations and equals $h_x^{\mathrm{full}}$ in exact arithmetic. The latter therefore also provides a numerical check of our implementation. We report the norm ratio, relative $\ell_2$ error, optimally rescaled error using the least-squares scalar $\alpha_x^\star(A_x)=\langle h_x^{\mathrm{full}},A_x\rangle/\|A_x\|_2^2$, cosine similarity, and top-1 agreement. Formal definitions are provided in Appendix~\ref{app:naive-path-sum}.

\begin{wraptable}{r}{0.65\textwidth}
\vskip-0.4in
\vspace{0.5cm}
\centering
\caption{Naive additive reconstructions versus exact M\"obius reconstruction on ResNet-18 and ResNet-34. Entries are means, rounded to three decimal places. ResNet-18 uses $N=10{,}000$ ImageNet validation images; ResNet-34 uses $N=1{,}000$.}
\vspace{-0.3cm}
\label{tab:naive_vs_mobius_main}
\scalebox{0.62}{
\begin{tabular}{clccccc}
\toprule[1.5pt]
& \textbf{Method} & \textbf{Norm ratio} & \textbf{Rel.\ error} ($\downarrow$) & \textbf{Scaled error} ($\downarrow$) & \textbf{Cosine} ($\uparrow$) & \textbf{Top-1 agree} ($\uparrow$) \\
\midrule
\multirow{4}{*}{\textbf{18}}
& $A_x^{\mathrm{raw}}$  & $79.400$        & $79.200$        & $0.961$ & $0.222$ & $0.006$ \\
& $A_x^{\mathrm{mean}}$ & $0.310$         & $0.977$         & $0.961$ & $0.222$ & $0.006$ \\
& $A_x^{\mathrm{ctr}}$  & $77.400$        & $77.200$        & $0.959$ & $0.229$ & $0.006$ \\
& M\"obius\tablefootnote{ResNet-18, before rounding: relative error $8.51\times10^{-7}$, optimally scaled error $8.55\times10^{-7}$.} & \textbf{1.000} & \textbf{0.000} & \textbf{0.000} & \textbf{1.000} & \textbf{1.000} \\
\midrule[1.2pt]
\multirow{4}{*}{\textbf{34}}
& $A_x^{\mathrm{raw}}$  & $18{,}852.327$  & $18{,}851.887$  & $0.889$ & $0.440$ & $0.025$ \\
& $A_x^{\mathrm{mean}}$ & $0.288$         & $0.910$         & $0.889$ & $0.440$ & $0.025$ \\
& $A_x^{\mathrm{ctr}}$  & $17{,}750.219$  & $17{,}749.760$  & $0.879$ & $0.459$ & $0.024$ \\
& M\"obius\tablefootnote{ResNet-34, before rounding: relative error is $0$; optimally scaled error is $1.76\times10^{-16}$.} & \textbf{1.000} & \textbf{0.000} & \textbf{0.000} & \textbf{1.000} & \textbf{1.000} \\
\bottomrule[1.5pt]
\end{tabular}
}
\vspace{-0.4cm}
\end{wraptable}

\textbf{Results and Analysis.}
Table~\ref{tab:naive_vs_mobius_main} shows that all three naive aggregations fail to reconstruct the full-network logits. On ResNet-18, the raw and centered sums inflate the logit norm by factors of approximately $79.4$ and $77.4$, yielding relative errors of $79.2$ and $77.2$, respectively. Averaging prevents this magnitude explosion but still produces a relative error of $0.977$. Optimal rescaling does not resolve the discrepancy: the scaled errors remain between $0.959$ and $0.961$, the cosine similarities are only $0.222$--$0.229$, and top-1 agreement is $0.6\%$. ResNet-34 exhibits the same qualitative behavior, with even greater norm inflation for the raw and centered sums and no naive method exceeding $2.5\%$ top-1 agreement. In contrast, the M\"obius reconstruction recovers the full logits up to floating-point precision on both architectures. Thus, residual masks define meaningful interventions, but their outputs cannot be interpreted as disjoint additive path contributions. The exact additive units are the M\"obius coefficients $\Delta_S h_x$, which isolate the incremental contribution of each subset through inclusion--exclusion.

\textbf{Observation~\ding{183}: Interaction mass concentrates at mid-to-high orders.}
Having established that exact additive contributions are given by the M\"obius coefficients, we next examine how their squared mass is distributed across interaction orders. For $S\subseteq[L]$, the order $|S|$ counts the number of residual branches that contribute jointly after all lower-order effects have been removed through inclusion--exclusion. The diagnostic question is: \textit{At their pretrained operating point, are ResNets dominated by low-order residual interactions, or is their functional contribution concentrated at mid-to-high orders?} We consider two readouts: the scalar predicted-class logit $v_x(m)=h_{x,\hat y(x)}(m)$, where $\hat y(x)=\arg\max_c h_{x,c}^{\mathrm{full}}$ is fixed across masks, and the full logit vector $h_x(m)$. For notational compactness, we suppress the dependence on $x$ below and write $Z^r=\sum_{j=1}^{L}E_j^r$. For each readout $r\in\{v,h\}$, the order statistics are
{\small
\begin{equation}
E_k^r=\sum_{|S|=k}(a_S^r)^2,\quad
\widetilde E_k^r=\frac{E_k^r}{Z^r},\quad
C_{\le q}^r=\frac{\sum_{j=1}^{q}E_j^r}{Z^r},\quad
\kappa^r=\frac{\sum_{j=1}^{L}jE_j^r}{Z^r},\quad
M_k^r=\frac{E_k^r}{\binom{L}{k}}.
\end{equation}}
Here, $\widetilde E_k^r$ measures the fraction of squared coefficient mass, which we call \emph{interaction energy}, assigned to order $k$; $C_{\le q}^r$ is its cumulative mass through order $q$; and $\kappa^r$ summarizes the spectrum by its energy-weighted mean order. Because the M\"obius coefficients are not generally orthogonal, these quantities describe the distribution of coefficient mass rather than the fraction of logit norm explained. The per-interaction statistic $M_k^r$ controls for the $\binom{L}{k}$ possible subsets at each order. All reported quantities are dataset means of these per-example statistics; standard errors and additional results are provided in Appendix~\ref{app:interaction-order-spectrum}.

\begin{figure}[t]
\centering
\vspace{-1cm}
\includegraphics[width=1\linewidth]{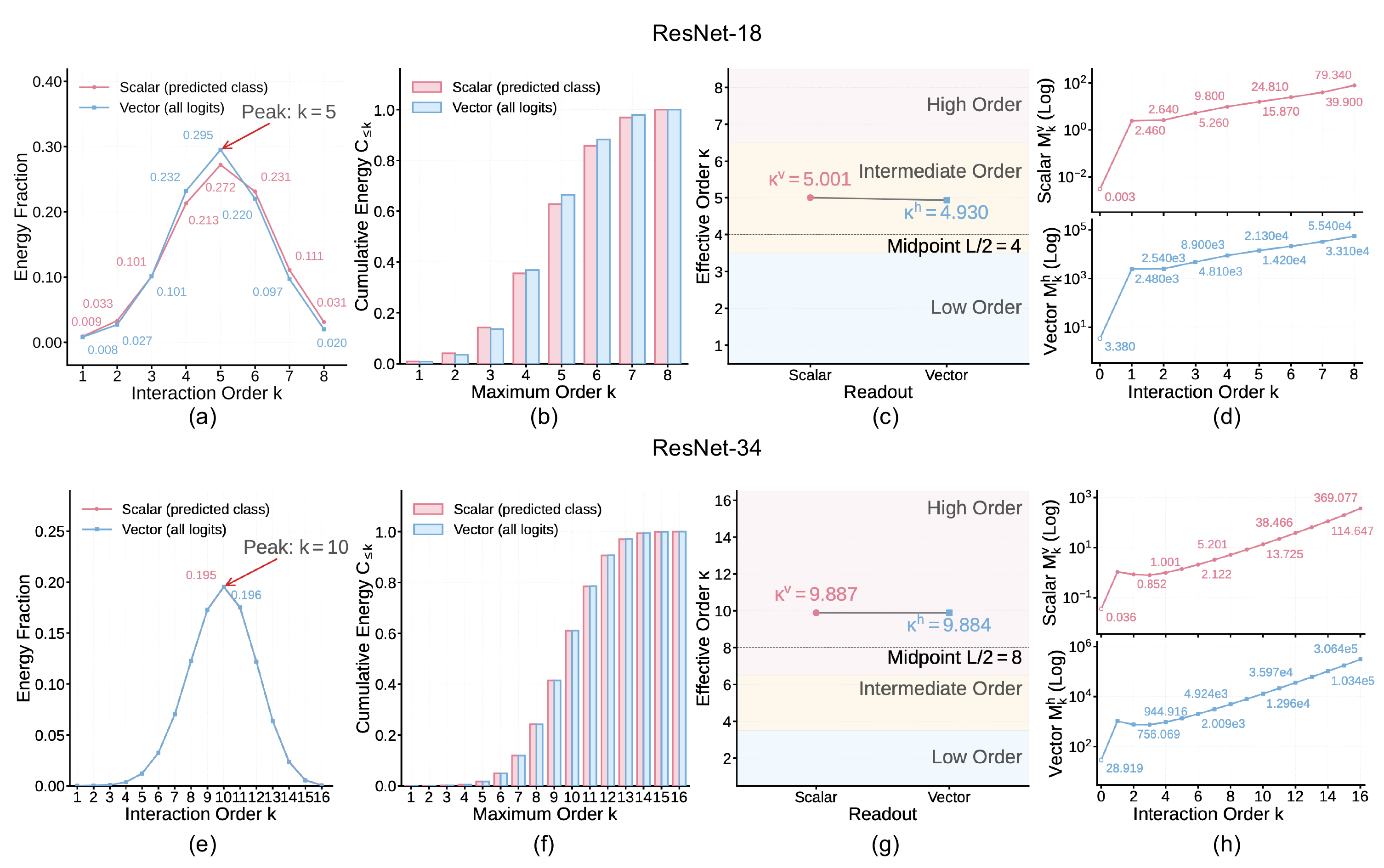}
\vspace{-0.5cm}
\caption{Residual interaction-order spectra for scalar predicted-class and vector-logit readouts. The top and bottom rows show ResNet-18 and ResNet-34, respectively. From left to right, the columns report normalized interaction energy $\widetilde E_k$, cumulative low-order energy $C_{\le k}$, effective interaction order $\kappa$, and mean per-interaction energy $M_k$ on a logarithmic scale.}
\label{fig:interaction-order-spectrum}
\end{figure}

\begin{wrapfigure}{r}{0.62\textwidth}
\vskip-0.3in
\vspace{0.3cm}
    \centering
    \includegraphics[width=\linewidth]{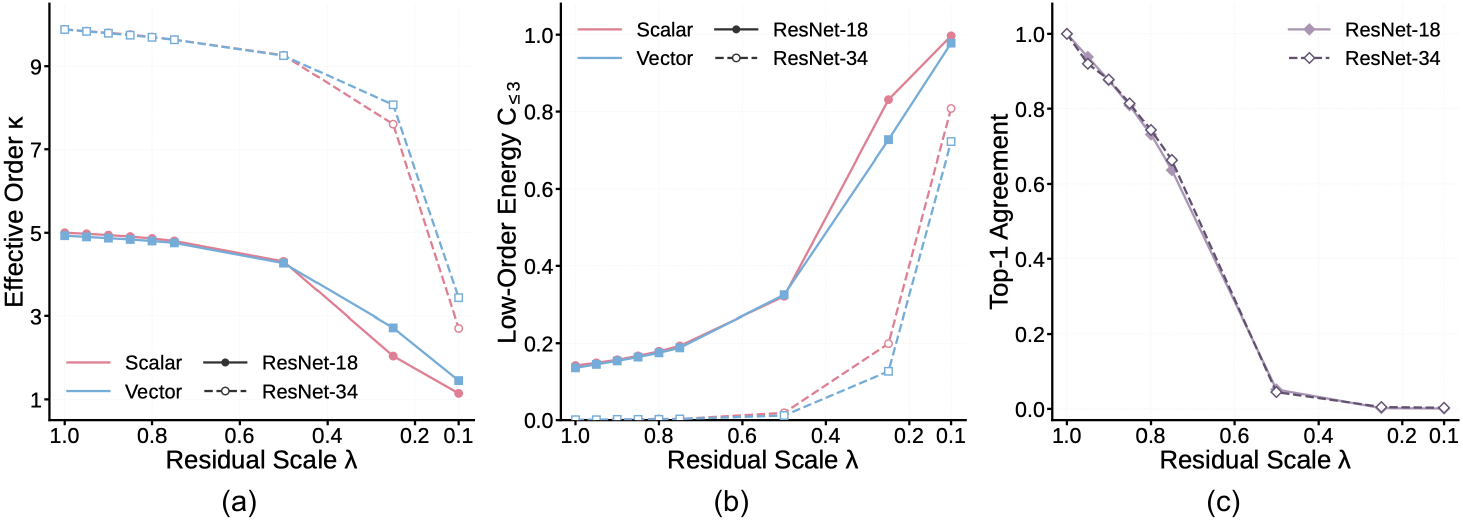}
    \vspace{-0.5cm}
    \caption{Inference-time residual-scaling sweep on pretrained ResNet-18 ($N=10{,}000$) and ResNet-34 ($N=1{,}000$), without retraining. Solid and dashed lines denote ResNet-18 and ResNet-34, respectively. (a) Effective interaction order $\kappa$ for the scalar and vector readouts. (b) Cumulative low-order energy $C_{\le3}$. (c) Top-1 agreement of each scaled model with its prediction at $\lambda=1$.}
    \label{fig:residual-scaling-sweep}
\end{wrapfigure}

\textbf{Results and Analysis.}
Figure~\ref{fig:interaction-order-spectrum} shows closely aligned scalar and vector spectra for both architectures. In ResNet-18, both peak at $k=5$: orders $4$--$6$ account for $71.6\%$ and $74.7\%$ of scalar and vector energy, whereas orders up to three contribute only $14.17\%$ and $13.56\%$. The effective orders, $\kappa^v=5.001$ and $\kappa^h=4.930$, also exceed the midpoint $L/2=4$. ResNet-34 shows the same pattern, with both spectra peaking at $k=10$ and effective orders $\kappa^v=9.887$ and $\kappa^h=9.884$, above the midpoint $L/2=8$. For both architectures, $M_k$ increases toward higher orders and is largest at $k=L$, showing that the mid-order peaks are not merely caused by subset counts. Thus, at $\lambda=1$, pretrained ResNets operate in a mid-to-high-order interaction regime: short structural paths do not imply low-order functional interactions.

\textbf{Observation~\ding{184}: Residual scaling lowers interaction order but changes predictions.}
The mid-to-high-order spectra above are measured at the natural operating point $\lambda=1$. Proposition~\ref{proposition:residual-scaling-bound} predicts that a $k$-way interaction is attenuated as $O(\lambda^k)$ in the small-update regime, suggesting that decreasing the residual scale should shift interaction mass toward lower orders. We therefore ask: \textit{Does inference-time residual scaling causally lower functional interaction order, and can it do so without substantially changing the model's predictions?} Without retraining, we sweep $\lambda\in\{1.0,0.95,0.9,0.85,0.8,0.75,0.5,0.25,0.1\}$ for both ResNet-18 and ResNet-34. At each scale, we recompute the effective order $\kappa$ and cumulative low-order energy $C_{\le3}$, and measure top-1 agreement with each model's prediction at $\lambda=1$. The denser sampling over $\lambda\in[0.8,1.0]$ tests whether interaction order decreases before prediction behavior changes.

\textbf{Results and Analysis.}
Reducing $\lambda$ shifts both scalar and vector interactions toward lower orders in both architectures (Figure~\ref{fig:residual-scaling-sweep}a--b). For ResNet-18, as $\lambda$ decreases from $1$ to $0.1$, $\kappa^v$ falls from $5.001$ to $1.144$ and $\kappa^h$ from $4.930$ to $1.452$, while $C_{\le3}^v$ rises from $0.142$ to $0.997$ and $C_{\le3}^h$ from $0.136$ to $0.978$. ResNet-34 shows the same trend: its effective orders fall from approximately $9.9$ to below $3.5$, while its scalar and vector low-order mass rises to approximately $0.80$ and $0.72$. These results support the qualitative prediction of Proposition~\ref{proposition:residual-scaling-bound}. However, the shift also changes the model's predictions (Figure~\ref{fig:residual-scaling-sweep}c). For ResNet-18, top-1 agreement falls to $0.938$ at $\lambda=0.95$, $0.731$ at $\lambda=0.8$, and $0.052$ at $\lambda=0.5$; ResNet-34 follows a similar trajectory. Thus, residual scaling can lower interaction order, but it does not preserve model behavior. Together, Observations~\ding{183}--\ding{184} show that pretrained ResNets are not dominated by low-order interactions at $\lambda=1$. This does not exclude concentration on a small number of specific subsets, which we examine in Section~\ref{sec:sparse-routing}.


\vspace{-0.5cm}
\section{A Residual Network Routes Like an MoE,Without a Router }~\label{sec:sparse-routing}
\vspace{-0.8cm}

The preceding section studies how interaction energy is distributed across orders; here we ask whether it is concentrated on a few input-specific residual subsets. We test whether a small number of experts can reconstruct the network readouts, whether the selected experts vary across inputs and classes, and whether their complexity depends on sample difficulty. Together, these observations support a sparse, input-dependent soft-routing interpretation of residual computation. All observations use ImageNet-pretrained ResNet-18 and ResNet-34 and report both scalar and vector readouts. Observation~\ding{185} uses unlabeled ImageNet-1K test subsets with $N=10{,}000$ for ResNet-18 and $N=1{,}000$ for ResNet-34. Observations~\ding{186}--\ding{187} use a class-balanced ImageNet-1K validation subset containing $10$ images per class ($N=10{,}000$) for both architectures.

\begin{figure*}[t]
\centering
\vspace{-0.6cm}
\includegraphics[width=\linewidth]{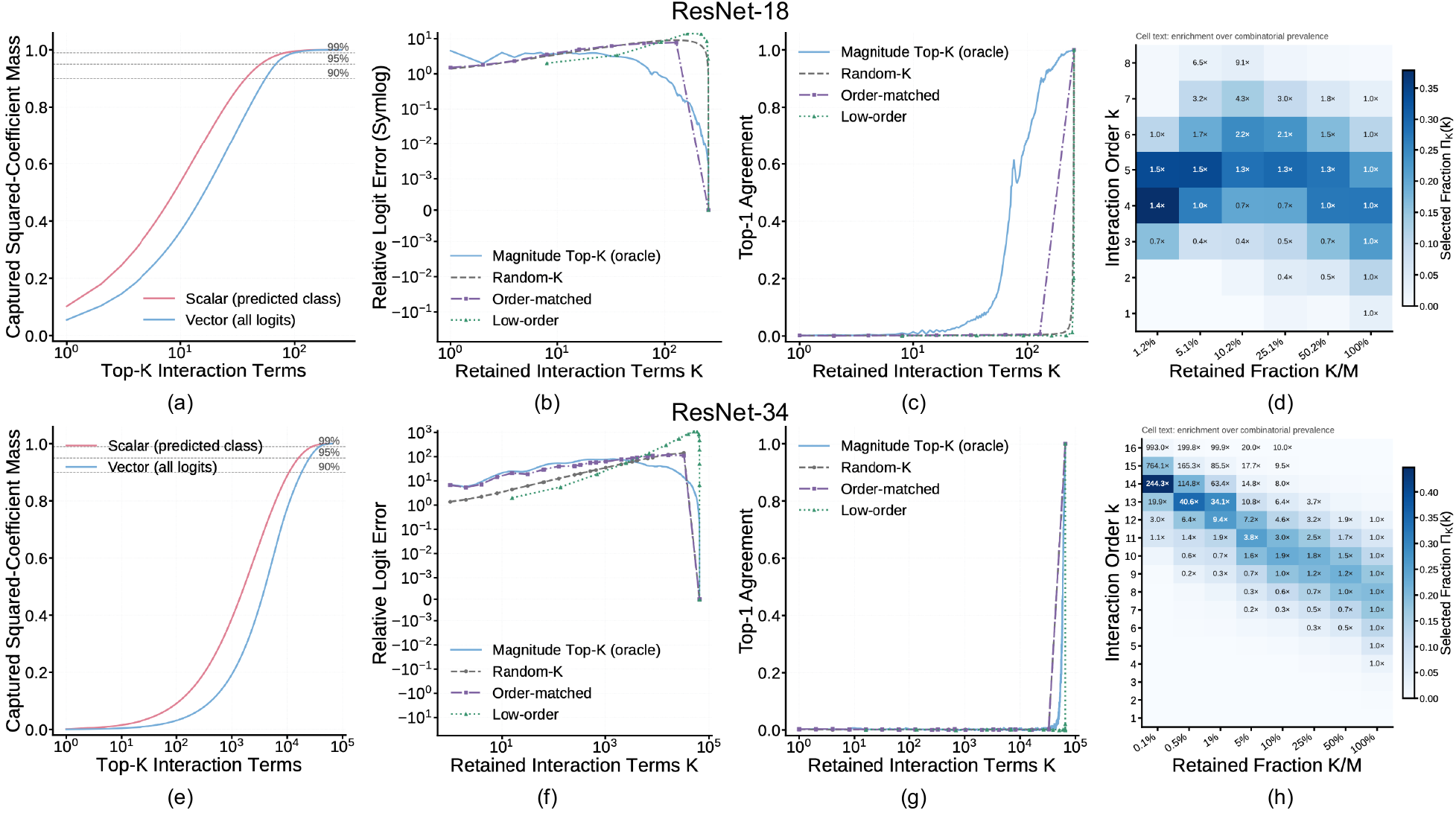}
\vspace{-0.5cm}
\caption{Top-$K$ interaction concentration and reconstruction.
ResNet-18 results are shown on top and ResNet-34 results on the bottom.
(a,e) Captured squared coefficient mass.
(b,f) Relative logit error.
(c,g) Agreement with the full-model predictions.
(d,h) Order composition and enrichment of the vector Top-$K$ supports.
Most coefficient mass is concentrated in a sparse interaction core, while a
long tail progressively refines the final prediction.}
\label{fig:topk-reconstruction}
\vspace{-0.1cm}
\end{figure*}

\textbf{Observation~\ding{185}: Residual interactions form a sparse but long-tailed expert representation.} We ask whether a compact per-input subset of residual interactions can preserve the functional contribution and prediction of the full model. Let $\mathcal{A}_K^v(x)$ and $\mathcal{A}_K^h(x)$ denote the $K$ interactions with the largest scalar and vector magnitudes, respectively. Their truncated reconstructions are
\begin{equation}
\label{eq:main-topk-reconstruction}
\widehat v_{x,K}
=
\Delta_{\emptyset}v_x+
\sum_{S\in\mathcal{A}_K^v(x)}\Delta_Sv_x,
\qquad
\widehat h_{x,K}
=
\Delta_{\emptyset}h_x+
\sum_{S\in\mathcal{A}_K^h(x)}\Delta_Sh_x .
\end{equation}
Because selection uses the complete M\"obius decomposition, these are post-hoc oracle supports that measure representational rather than execution sparsity. We enumerate all masks for ResNet-18 on $10{,}000$ ImageNet test images and ResNet-34 on $1{,}000$ images, producing $255$ and $65{,}535$ non-empty interaction experts, respectively. We evaluate captured coefficient mass, relative logit error, and agreement with the full-model prediction, and compare magnitude Top-$K$ with random-$K$, order-matched, and low-order selection.

\textbf{Results and Analysis.}
Both models exhibit substantial interaction sparsity. For ResNet-18, Top-64 contains $25.1\%$ of the dictionary but captures $93.64\%$ of the vector
coefficient mass. Top-128 captures $99.71\%$ of the mass and preserves $87.66\%$ of the full-model predictions. For ResNet-34, whose $16$ blocks
induce $2^{16}-1=65{,}535$ experts, retaining $25\%$ of the dictionary captures $95.28\%$ of the scalar mass, while $38.3\%$ captures $95\%$ of the
vector mass. Prediction reconstruction requires a broader support than mass recovery, revealing a decision-relevant long tail. For ResNet-34, retaining half of the dictionary captures $97.90\%$ of the vector mass, but smaller interactions remain important for accurate logits and decisions. Moreover, magnitude selection substantially outperforms random, order-matched, and low-order selection, showing that sparsity depends on the identities of specific
mid- and high-order interactions rather than interaction order alone. These results establish the sparsity component of soft routing; the next
experiment examines whether the dominant supports vary across inputs.

\textbf{Observation~\ding{186}: Expert supports share a global core with class-dependent variation.} We ask whether the dominant residual interactions are fixed across inputs or vary systematically with class. For each input $x$ and readout $r\in\{v,h\}$, we form the Top-$K$ interaction support $\mathcal{A}_K^r(x)$ among the $M=2^L-1$ non-empty interactions. For inputs $x_i,x_j$ with labels $y_i,y_j$, we compare binary supports using Jaccard similarity $J_K^r$ and vector mass-weighted supports using weighted Jaccard $J_K^w$. Class alignment is measured by $\Delta_{\mathrm{class}}^r(K)=\mathbb{E}[J_K^r\mid y_i=y_j]-\mathbb{E}[J_K^r\mid y_i\neq y_j]$, defined analogously for $J_K^w$. To compare architectures, we report the retained fraction $K/M$ and use label-shuffled classes and the size-matched dataset-level Top-$K$ support as controls.

\begin{figure}[t]
  \centering
    \vspace{-0.2cm}
  \includegraphics[width=\textwidth]{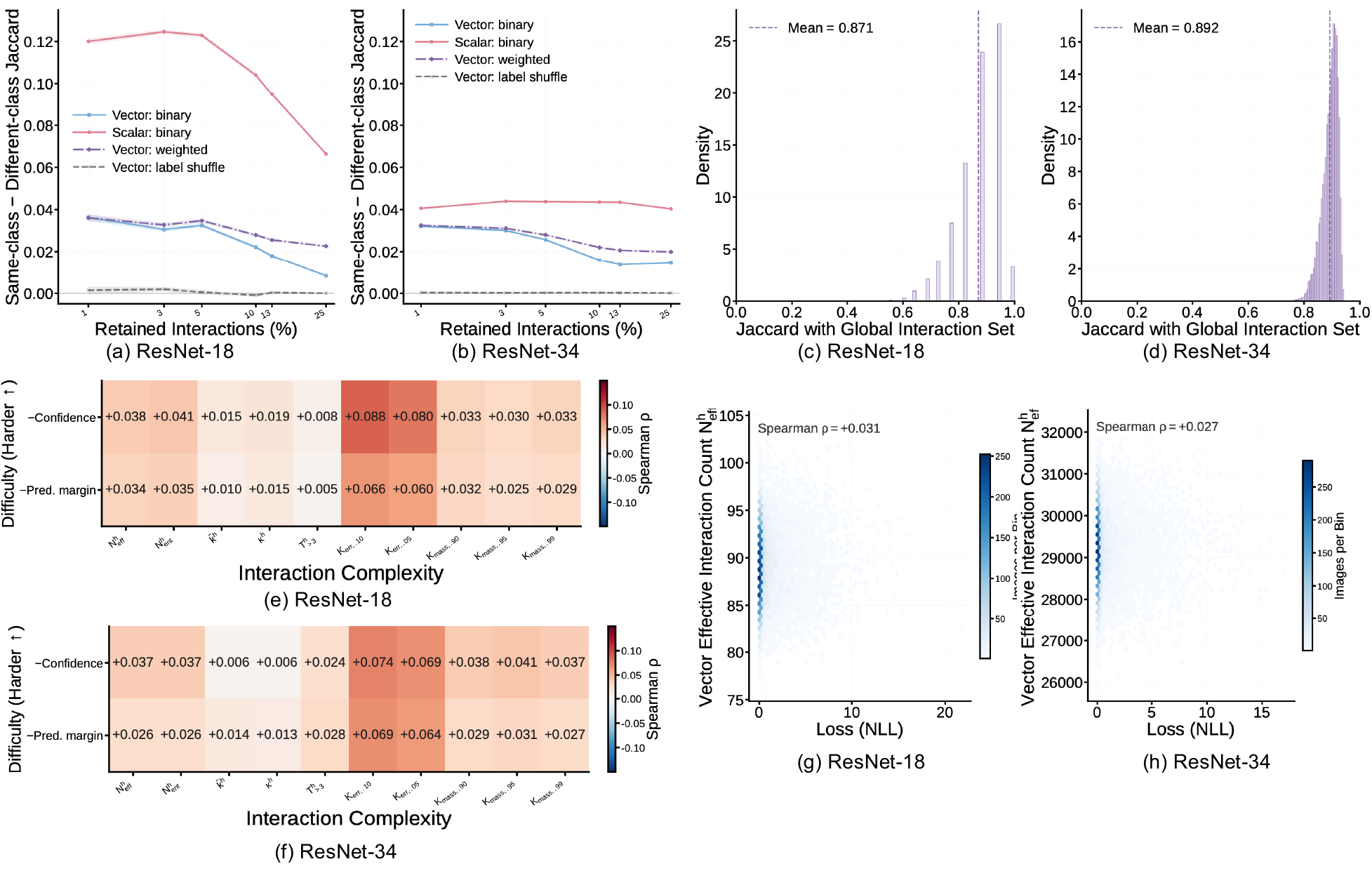}
  \vspace{-0.8cm}
    \caption{Input dependence and difficulty sensitivity of residual interaction supports. Results are shown for ResNet-18 and ResNet-34 on $10{,}000$ ImageNet validation images. (a,b) Same-class minus different-class similarity for scalar, vector, and weighted interaction signatures, with label shuffling as a null control. (c,d) Input-wise Jaccard similarity with the global vector interaction set at approximately $13\%$ retention. (e,f) Spearman correlations between label-free difficulty measures and vector interaction complexity. (g,h) Negative log-likelihood versus the vector effective interaction count $N_{\mathrm{eff}}^h$. Dominant interactions exhibit class-dependent variation within a shared global core, while their overall complexity is only weakly associated with input difficulty.}
      \vspace{-0.1cm}
    \label{fig:interaction-selection-and-difficulty}
\end{figure}

\textbf{Results and Analysis.} For both models, same-class pairs are more similar than different-class pairs at every retained fraction, whereas label shuffling reduces the gap to nearly zero (Figure~\ref{fig:interaction-selection-and-difficulty}a,b). The effect is readout-dependent: the peak class gaps for predicted-class scalar supports are $0.125$ and $0.044$ for ResNet-18 and ResNet-34, compared with $0.036$ and $0.032$ for vector supports. At approximately $13\%$ retention, however, the vector supports remain close to the global support, with mean Jaccard similarities of $0.871$ and $0.892$ (Figure~\ref{fig:interaction-selection-and-difficulty}c,d). Thus, inputs do not select disjoint expert sets; they modify a largely shared interaction core. Consistently, even at $K=32$, scalar binary signatures achieve $24.2\%$ and $9.9\%$ 1-NN class retrieval, respectively, versus $0.1\%$ chance. These results support a soft, readout-dependent routing interpretation based on class-aware variation within a shared residual-expert pool.

\textbf{Observation~\ding{187}: Interaction complexity is largely insensitive to input difficulty.} We ask whether harder inputs require broader, higher-order, or less compressible residual interactions. We measure difficulty using negative confidence, negative predicted margin, and negative log-likelihood. For the vector readout, we consider support size, interaction order, and the Top-$K$ budgets required to reach specified coefficient-mass or reconstruction-error thresholds. In particular, $N_{\mathrm{eff}}^h=(\sum_S a_S^h)^2/\sum_S(a_S^h)^2$ measures the effective number of contributing interactions. Associations are evaluated using Spearman correlation.

\textbf{Results and Analysis.} Interaction complexity is only weakly associated with difficulty. Across all metrics in Figure~\ref{fig:interaction-selection-and-difficulty}(e,f), the largest correlation is $0.088$ for ResNet-18 and $0.074$ for ResNet-34. Negative log-likelihood is also nearly uncorrelated with $N_{\mathrm{eff}}^h$, with $\rho=0.031$ and $0.027$, respectively (Figure~\ref{fig:interaction-selection-and-difficulty}g,h). We therefore find little evidence that harder inputs systematically require broader, higher-order, or less compressible interaction structure. Together with the preceding experiment, these results suggest that input dependence primarily changes \emph{which} interactions dominate, rather than \emph{how many} are engaged.

\vspace{-0.2cm}
\section{Related Work}
\vspace{-0.2cm}
Our work intersects four lines of research. (1) \textit{Residual learning and path-ensemble views.} ResNets ease deep-network optimization via identity shortcuts~\citep{he2016deep}, and the unraveled view interprets them as ensembles of exponentially many paths with ensemble-like redundancy~\citep{veit2016residual}; related work studies stochastic depth~\citep{huang2016deep}, the effective depth of the virtual ensemble~\citep{littwin2016loss}, and signal propagation~\citep{he2016identity,de2020batch}. These studies concern structure and optimization, not which residual interactions produce a particular prediction. (2) \textit{Gating and dynamic routing.} Highway networks~\citep{srivastava2015highway}, residual gates~\citep{savarese2016learning}, SkipNet~\citep{wang2018skipnet} and BlockDrop~\citep{wu2018blockdrop} equip networks with explicit, learned gates or per-input block skipping. We instead analyze a standard trained ResNet without adding a router, using gates only as an analytical probe over residual-branch masks. (3) \textit{Mixture-of-Experts.} Sparse MoE models scale capacity by routing each input to a few experts~\citep{shazeer2017outrageously,lepikhin2020gshard,fedus2021switch}, with subsequent work on routing assignment~\citep{lewis2021base,zhou2022expertchoice,dai2022stablemoe} and cross-layer expert paths~\citep{gu2026pathmoe,rajbhandari2022deepspeedmoe}. We approach MoE from the opposite direction: rather than adding explicit experts, we ask whether a trained ResNet already induces implicit, input-dependent expert structure. (4)  \textit{Attribution and interaction decompositions.} Component importance has been estimated via Taylor criteria~\citep{molchanov2019importance} and Shapley values~\citep{lundberg2017unified,ancona2020shapley}, recently extended to higher-order interactions~\citep{chen2023harsanyinet,muschalik2024shapiq,kang2024mobius}. Our players are not input features, neurons, or filters but residual branches; treating branch on/off as the game defines residual corrections and their higher-order interactions as functional contribution units, bridging residual learning and implicit expert behavior.

\vspace{-0.2cm}
\section{Conclusion}
\label{sec:conclusion}

We use M\"obius inversion to decompose trained ResNets exactly into residual
corrections and higher-order interactions. In ResNet-18 and ResNet-34,
interaction mass peaks at intermediate orders, while residual scaling shifts
it toward lower orders but also changes predictions. Crucially, dominant
interactions vary across inputs and classes around a shared global core. Dense
execution therefore does not imply fixed functional computation: every block
executes, yet different inputs rely on different interaction experts. Standard
ResNets thus exhibit implicit soft routing at the level of functional
contribution, not explicit sparse execution. Extending this analysis to deeper
networks requires efficient estimators of high-order interactions.

\clearpage
\section*{AI Use Statement}
Generative AI tools were used solely to polish the language of the manuscript. All research ideas, theoretical results, experiments, and analyses are the authors’ own, and the authors take full responsibility for the content of this paper.

\section*{Reproducibility Statement}
We have taken several steps to ensure that our results can be reproduced. All theoretical claims are stated together with their assumptions in the main text, and complete proofs are provided in the appendices. The experimental protocol is fully specified in the main text and appendices: we use a publicly available ImageNet-pretrained ResNet-18 and ResNet-34, exhaustively enumerate all $2^L$ residual-branch masks, and evaluate on the ImageNet validation images, with the M\"obius-inversion reconstruction, the interaction-order spectrum, and the residual-scaling sweep described step by step. Because the analysis relies only on public pretrained weights and the public ImageNet validation set, no model training is required, and the exhaustive mask evaluation is deterministic given the pretrained weights. We provide all model, data, and analysis details needed to reproduce our results directly from these public resources.

\section*{Ethics Statement}
This work analyzes the internal computational structure of existing, publicly available image-classification networks. It involves no human subjects and no private or personally identifiable data, using only the standard public ImageNet validation set and publicly released pretrained weights, and it releases no new datasets or trained models. The study is interpretability-oriented: it characterizes how residual branches interact functionally and does not by itself measure the fairness, safety, or reliability of deployed systems, for which additional application-specific evaluation would be required. The authors are aware of no ethical concerns raised by this work and confirm adherence to the ICLR Code of Ethics.

\bibliography{reference}
\bibliographystyle{iclr2027_conference}

\clearpage
\appendix

\addcontentsline{toc}{section}{Appendix} 
\renewcommand \thepart{} 
\renewcommand \partname{}
\part{\Large\centering{Appendix}}
\parttoc 

\makeatletter
\setlength{\@fptop}{0pt}
\setlength{\@fpsep}{14pt plus 0fil}
\setlength{\@fpbot}{0pt plus 1fil}
\makeatother
\setlength{\floatsep}{14pt plus 2pt minus 2pt}
\setlength{\textfloatsep}{16pt plus 2pt minus 4pt}
\setlength{\intextsep}{14pt plus 2pt minus 2pt}

\clearpage
\section{Notation}\label{app:notation}

Table~\ref{tab:residual-interaction-notation} summarizes the notation used throughout the paper.

\begin{table}[!ht]
\centering
\caption{Notation. Superscript $r\in\{v,h\}$ marks the scalar (predicted-class) or vector (full-logit) readout; per-input statistics are averaged over images. Masks are analytical interventions, not learned routing decisions. Symbols used within a single appendix are defined there.}
\label{tab:residual-interaction-notation}
\scriptsize
\setlength{\tabcolsep}{4pt}
\renewcommand{\arraystretch}{1.05}
\begin{tabularx}{\textwidth}{@{}p{0.24\textwidth}>{\raggedright\arraybackslash}X@{}}
\toprule[1.5pt]
\textbf{Symbol} & \textbf{Meaning} \\
\midrule
\multicolumn{2}{@{}l}{\textit{Residual network, masks, and readouts}} \\
\addlinespace[1pt]
$x,\ N,\ C$ & Input image; number of images; number of classes. \\
$L,\ [L];\ l,i,j\in[L]$ & Number and index set of residual blocks; block indices. \\
$F_l,\ s_l,\ \phi_l$ & Residual branch, shortcut, and activation of block $l$. \\
$x_l(m)$ & Hidden state at block $l$ under mask $m$. \\
$m\in\{0,1\}^L$ & Residual-branch mask; $m_l=0$ disables branch $l$ only. \\
$\mathbf{0},\ \mathbf{1}_S,\ \mathbf{1}_{[L]}$ & Shortcut-only, subset-active, and full masks. \\
$\mathbf{1}[\cdot]$ & Indicator function. \\
$h_x(m)\in\mathbb{R}^C$ & Logits under mask $m$ (vector response). \\
$h_x^{\mathrm{full}},\ h_x^{\emptyset}$ & Full-network and shortcut-only logits. \\
$\hat y(x)$ & Full-model predicted class. \\
$\rho_x,\ v_x(m)$ & Scalar readout and scalar response (predicted-class logit). \\
$y,\ y_i$ & Ground-truth labels of $x$ and $x_i$. \\
\addlinespace[2pt]
\multicolumn{2}{@{}l}{\textit{M\"obius residual interactions}} \\
\addlinespace[1pt]
$S,\ T,\ U\subseteq[L]$ & Subsets of residual branches. \\
$k=|S|$ & Interaction order. \\
$\Delta_S v_x,\ \Delta_S h_x$ & Scalar and vector M\"obius coefficients of $S$. \\
$\Delta_{\emptyset}v_x$ & Shortcut-only baseline (not counted in $K$). \\
$\Delta_{\{l\}}v_x,\ \Delta_{\{i,j\}}v_x$ & First-order correction; pairwise interaction. \\
$\mathcal{R},\ M$ & Dictionary of non-empty interactions; its size $2^L-1$. \\
\addlinespace[2pt]
\multicolumn{2}{@{}l}{\textit{Residual experts and Top-$K$ reconstruction}} \\
\addlinespace[1pt]
$a_S^v(x),\ a_S^h(x)$ & Contribution magnitudes $|\Delta_S v_x|$ and $\|\Delta_S h_x\|_2$. \\
$K,\ K/M$ & Top-$K$ budget; retained fraction. \\
$\mathcal{A}_K^r(x)$ & Input-specific Top-$K$ expert set. \\
$\widehat v_{x,K},\ \widehat h_{x,K}$ & Top-$K$ reconstructions. \\
$\Gamma_K^r(x)$ & Retained fraction of squared coefficient mass. \\
$\mathrm{Err}_K^r,\ \mathrm{Cos}_K$ & Relative reconstruction error; cosine similarity. \\
$K_\tau^{\mathrm{err}},\ K_\eta^{\mathrm{mass}}$ & Smallest $K$ with error $\le\tau$ or retained mass $\ge\eta$. \\
\addlinespace[2pt]
\multicolumn{2}{@{}l}{\textit{Interaction-order and breadth statistics}} \\
\addlinespace[1pt]
$E_k^r(x),\ Z^r(x)$ & Interaction energy at order $k$; total over $k\ge1$. \\
$\widetilde E_k^r(x)$ & Normalized order spectrum. \\
$C_{\le q}^r(x),\ T_{>3}^r(x)$ & Energy fraction up to order $q$; fraction above order $3$. \\
$\kappa^r(x)$ & Effective interaction order. \\
$M_k^r(x)$ & Mean energy per interaction at order $k$. \\
$q_S^r(x),\ p_S^r(x)$ & Normalized magnitude and squared-magnitude weights. \\
$N_{\mathrm{eff}}^r(x),\ N_{\mathrm{ent}}^r(x)$ & Participation-ratio and entropy effective counts. \\
$\bar k^r(x)$ & Magnitude-weighted mean order. \\
\addlinespace[2pt]
\multicolumn{2}{@{}l}{\textit{Input dependence of expert sets}} \\
\addlinespace[1pt]
$J_K^r(x_i,x_j)$ & Jaccard overlap of two Top-$K$ sets. \\
$w_{K,S}^h(x),\ J_K^w$ & Mass-weighted vector signature; weighted Jaccard. \\
$\Delta_{\mathrm{class}}^r(K)$ & Same-class minus different-class overlap. \\
$G_K$ & Dataset-level global Top-$K$ set. \\
\addlinespace[2pt]
\multicolumn{2}{@{}l}{\textit{Aggregation baselines and prediction agreement}} \\
\addlinespace[1pt]
$A_x^{\mathrm{raw}},\ A_x^{\mathrm{mean}},\ A_x^{\mathrm{ctr}}$ & Naive sum, mean, and centered aggregations. \\
$A_x^{\mathrm{Mob}}$ & M\"obius reconstruction of $h_x^{\mathrm{full}}$. \\
$\alpha_x^\star$ & Least-squares rescaling factor. \\
$\mathrm{Agree}$ & Top-1 agreement with the full model (not accuracy). \\
\addlinespace[2pt]
\multicolumn{2}{@{}l}{\textit{Residual scaling}} \\
\addlinespace[1pt]
$\lambda,\ \lambda_0$ & Residual scale and its upper bound. \\
$x_l^{(\lambda)},\ h_x^{(\lambda)},\ v_x^{(\lambda)}$ & Hidden state, logits, and scalar response at scale $\lambda$. \\
$\hat y_\lambda(x)$ & Predicted class of the scaled model. \\
$B_{S,x}$ & $\lambda$-independent constant in the interaction bound. \\
\bottomrule[1.5pt]
\end{tabularx}
\end{table}

\clearpage
\section{Detailed Related Work}~\label{appen:detailed_related_work}

\textbf{Residual learning and path-ensemble views.}
Residual networks were introduced to ease the optimization of very deep neural networks by reformulating layers as residual functions with respect to their inputs~\citep{he2016deep}. A residual block has the form $x_{l+1} = x_l + F_l(x_l)$, where the identity shortcut provides a stable information and gradient pathway. A prominent interpretation of this architecture is the unraveled view of~\citet{veit2016residual}, which shows that a ResNet can be represented as an exponential collection of computational paths of different lengths. Their lesion studies suggest that these paths exhibit ensemble-like redundancy, and their gradient analysis shows that relatively short paths dominate training. Stochastic depth~\citep{huang2016deep} provides a related training mechanism: during training, residual layers are randomly dropped, effectively exposing the model to shorter subnetworks while retaining the full network at test time. Our work builds on this line of thought but asks a different question. Beyond the original path-ensemble analysis of path lengths, block-removal robustness, and gradient flow~\citep{veit2016residual}, subsequent work has examined how the effective depth of the virtual ensemble evolves during training~\citep{littwin2016loss} and how identity shortcuts and residual-branch scaling affect signal propagation and trainability~\citep{he2016identity,de2020batch}. These studies illuminate network structure and optimization, but do not identify the residual interactions that contribute to a particular prediction. In contrast, we study the functional contribution structure induced by residual branches. Rather than treating paths as additive outputs, we introduce a residual-gate function view and ask whether residual corrections and their higher-order interactions form input-dependent, sparse, expert-like computational units. Thus, our focus is not merely that ResNets contain many paths, but whether their effective residual computations behave like implicit soft experts.

\textbf{Gating, dynamic routing, and residual branch selection.} Highway networks~\citep{srivastava2015highway} introduced input-dependent gates to regulate information flow between transformed and skipped representations. Residual gates~\citep{savarese2016learning} further studied scalar gates on residual architectures, showing that explicit gating can facilitate optimization and increase layer independence. Dynamic inference methods such as SkipNet~\citep{wang2018skipnet} and BlockDrop~\citep{wu2018blockdrop} learn to skip convolutional or residual blocks on a per-input basis, thereby converting static residual networks into explicitly routed conditional-computation models. These methods introduce or learn explicit gating mechanisms. Our goal is different: we analyze a standard trained ResNet without adding a router. We use gates only as an analytical probe to define a function over residual branch masks. This allows us to measure how the model output changes as residual branches are activated or removed, and to characterize whether the induced contribution structure is sparse and input-dependent.

\textbf{Mixture-of-Experts and expert-path structure.} Mixture-of-Experts (MoE) models use gating or routing functions to combine specialized expert modules~\citep{shazeer2017outrageously}. Sparse MoE architectures such as GShard~\citep{lepikhin2020gshard} and Switch Transformers~\citep{fedus2021switch} scale model capacity by activating only a small number of experts for each input. Subsequent work studies load balancing and routing assignment, including BASE layers~\citep{lewis2021base}, Expert Choice routing~\citep{zhou2022expertchoice}, and stable routing strategies~\citep{dai2022stablemoe}. More recent work has started to view MoE computation through the lens of cross-layer expert paths; for example, Path-Constrained MoE constrains the combinatorial space of expert paths by sharing router parameters across consecutive layers~\citep{gu2026pathmoe}. Residual MoE architectures, such as PR-MoE in DeepSpeed-MoE~\citep{rajbhandari2022deepspeedmoe}, also suggest that dense/shared paths and sparse expert paths can be combined. Our work approaches MoE from the opposite direction. Instead of adding explicit sparse experts to a network, we ask whether a trained residual network already induces implicit expert-like structures. The analogy is not exact: a standard ResNet does not perform sparse execution and has no explicit router. Nevertheless, if only a small subset of residual corrections or interactions dominates the functional contribution for a given input, then the network exhibits an implicit, soft, input-dependent expert structure.

\textbf{Attribution, pruning, and interaction decompositions.} A related line of work estimates the importance of neural network components for pruning or interpretation. Taylor-based pruning methods estimate the contribution of neurons or filters to the final loss~\citep{molchanov2019importance}, and Shapley-value approaches provide principled component importance measures for structured pruning~\citep{ancona2020shapley}. More broadly, Shapley values and related cooperative-game-theoretic tools have become common for explaining model predictions~\citep{lundberg2017unified}. Recent work extends these ideas to higher-order interactions: HarsanyiNet uses Harsanyi interactions to compute Shapley values in a single forward propagation~\citep{chen2023harsanyinet}, Shapiq provides tools for estimating Shapley interactions in machine learning~\citep{muschalik2024shapiq}, and the M\"obius transform has been proposed as a way to identify sparse, low-degree interactions in learned functions~\citep{kang2024mobius}. Our analysis is inspired by this interaction viewpoint, but the players in our setting are not input features, neurons, or filters. Instead, we treat residual branches as players and study the function induced by turning these branches on or off. This lets us define residual corrections and higher-order residual interactions as functional contribution units. The resulting framework provides a bridge between residual learning and implicit expert behavior.


\clearpage
\section{Proof of Lemma~\ref{lem:mobius-inverse}}
\label{app:mobius-proof}

\begin{proof}[Proof of Lemma~\ref{lem:mobius-inverse}]
For any subset $S\subseteq[L]$, define its M\"obius coefficient by
\[
\Delta_S v_x
=
\sum_{T\subseteq S}
(-1)^{|S|-|T|}v_x(\mathbf{1}_T),
\]
where $\mathbf{1}_T\in\{0,1\}^L$ activates exactly the residual branches
in $T$.

Fix an arbitrary subset $U\subseteq[L]$. Summing the coefficients over all
subsets of $U$ gives
\begin{align}
\sum_{S\subseteq U}\Delta_S v_x
&=
\sum_{S\subseteq U}
\sum_{T\subseteq S}
(-1)^{|S|-|T|}v_x(\mathbf{1}_T) \\
&=
\sum_{T\subseteq U}
v_x(\mathbf{1}_T)
\sum_{S:\,T\subseteq S\subseteq U}
(-1)^{|S|-|T|},
\end{align}
where the second equality follows by exchanging the order of the two finite
sums.

For a fixed $T\subseteq U$, every set $S$ satisfying
$T\subseteq S\subseteq U$ can be written uniquely as
$S=T\cup R$, where $R\subseteq U\setminus T$. Hence,
\begin{align}
\sum_{S:\,T\subseteq S\subseteq U}
(-1)^{|S|-|T|}
&=
\sum_{R\subseteq U\setminus T}
(-1)^{|R|} \\
&=
\sum_{r=0}^{|U\setminus T|}
\binom{|U\setminus T|}{r}(-1)^r \\
&=
(1-1)^{|U\setminus T|} \\
&=
\begin{cases}
1, & T=U,\\
0, & T\subsetneq U.
\end{cases}
\end{align}
Therefore, every term in the outer sum vanishes except the term with $T=U$.
It follows that
\[
\sum_{S\subseteq U}\Delta_S v_x
=
v_x(\mathbf{1}_U).
\]
Since $U\subseteq[L]$ was arbitrary, the identity holds for every subset
$U\subseteq[L]$.
\end{proof}

\clearpage
\section{Proof of Proposition~\ref{prop:residual-scaling}}
\label{app:residual-scaling-proof}

We prove Proposition~\ref{prop:residual-scaling} in three steps. First, we state the finite-difference identity that connects M\"obius coefficients to mixed partial derivatives. Second, we show that mixed derivatives of the hidden states with respect to residual gates scale with the corresponding power of $\lambda$. Third, we combine these two facts to obtain the desired bound.

\paragraph{Notation.}
For a subset $A\subseteq[L]$, write
\begin{equation}
    \partial_A
    :=
    \prod_{a\in A}
    \frac{\partial}{\partial m_a},
\end{equation}
with the convention that $\partial_{\emptyset}$ is the identity operator.
Let $x_l^{(\lambda)}(m)$ denote the hidden representation at block $l$ under gate mask $m\in[0,1]^L$ and residual scale $\lambda$. We write the final scalar readout as
\begin{equation}
    v_x^{(\lambda)}(m)
    =
    R\!\left(x_{L+1}^{(\lambda)}(m)\right),
\end{equation}
where $R$ includes the remaining network operations after the residual stack and the scalar readout $\rho_x$.

We assume that, for the fixed input $x$, all hidden states induced by $m\in[0,1]^L$ and $0<\lambda\leq\lambda_0$ remain in a compact set $\mathcal{K}$ containing the shortcut trajectory. We also assume that each residual map $F_l$ and the scalar readout $R$ are $C^L$ on $\mathcal{K}$ and that all derivatives up to order $L$ are bounded on $\mathcal{K}$. These are the smoothness and boundedness assumptions stated in Proposition~\ref{prop:residual-scaling}.

\begin{lemma}[Finite-difference identity]
\label{lem:finite-difference-identity}
Let $g:[0,1]^L\to\mathbb{R}$ be $C^{|S|}$ in the coordinates indexed by $S$. Then
\begin{equation}
    \sum_{T\subseteq S}
    (-1)^{|S|-|T|}
    g(\mathbf{1}_T)
    =
    \int_{[0,1]^{|S|}}
    \partial_S g(u_S,0_{\bar S})
    \,du_S,
    \label{eq:finite-difference-identity}
\end{equation}
where $\bar S=[L]\setminus S$, and $(u_S,0_{\bar S})$ denotes the point whose coordinates in $S$ are given by $u_S$ and whose coordinates outside $S$ are set to $0$.
\end{lemma}

\begin{proof}
The result follows by applying the fundamental theorem of calculus once for each coordinate in $S$. For a single coordinate $i$, we have
\begin{equation}
    g(1_i,0_{\bar{\{i\}}}) - g(0_i,0_{\bar{\{i\}}})
    =
    \int_0^1
    \frac{\partial g}{\partial m_i}(u_i,0_{\bar{\{i\}}})
    \,du_i.
\end{equation}
Applying the same argument iteratively over all coordinates in $S$ gives Equation~\ref{eq:finite-difference-identity}. Equivalently, the left-hand side is the $|S|$-fold forward finite difference of $g$ over the coordinates in $S$, and the right-hand side is its integral representation.
\end{proof}

\begin{lemma}[Gate derivatives of hidden states scale with residual order]
\label{lem:hidden-state-scaling}
Under the smoothness and boundedness assumptions above, for every block index $l$ and every subset $A\subseteq[L]$, there exists a constant $C_{l,A,x}>0$, independent of $\lambda$, such that
\begin{equation}
    \sup_{m\in[0,1]^L}
    \left\|
    \partial_A x_l^{(\lambda)}(m)
    \right\|
    \leq
    C_{l,A,x}\lambda^{|A|}
    \qquad
    \text{for all } 0<\lambda\leq\lambda_0.
    \label{eq:hidden-state-scaling}
\end{equation}
\end{lemma}

\begin{proof}
We prove the statement by induction over the residual block index $l$.

For the input representation $x_1^{(\lambda)}(m)=x$, the claim is immediate. If $A=\emptyset$, the hidden state is bounded because it lies in the compact set $\mathcal{K}$. If $A\neq\emptyset$, then $\partial_A x_1^{(\lambda)}(m)=0$, so Equation~\ref{eq:hidden-state-scaling} holds.

Now suppose the claim holds for block $l$. We prove it for block $l+1$. The scaled residual recursion is
\begin{equation}
    x_{l+1}^{(\lambda)}(m)
    =
    x_l^{(\lambda)}(m)
    +
    \lambda m_l
    F_l\!\left(x_l^{(\lambda)}(m)\right).
    \label{eq:state-recursion-proof}
\end{equation}
Applying $\partial_A$ gives
\begin{equation}
    \partial_A x_{l+1}^{(\lambda)}(m)
    =
    \partial_A x_l^{(\lambda)}(m)
    +
    \lambda\,
    \partial_A
    \left[
        m_l F_l\!\left(x_l^{(\lambda)}(m)\right)
    \right].
    \label{eq:state-derivative-proof}
\end{equation}

We analyze the second term. There are two cases.

\paragraph{Case 1: $l\notin A$.}
Then no derivative falls directly on the factor $m_l$. Hence,
\begin{equation}
    \partial_A
    \left[
        m_l F_l\!\left(x_l^{(\lambda)}(m)\right)
    \right]
    =
    m_l
    \partial_A
    F_l\!\left(x_l^{(\lambda)}(m)\right).
\end{equation}
By the multivariate chain rule, $\partial_A F_l(x_l^{(\lambda)}(m))$ is a finite sum of terms of the form
\begin{equation}
    D^r F_l\!\left(x_l^{(\lambda)}(m)\right)
    \left[
        \partial_{P_1}x_l^{(\lambda)}(m),
        \ldots,
        \partial_{P_r}x_l^{(\lambda)}(m)
    \right],
    \label{eq:chain-rule-term}
\end{equation}
where $\{P_1,\ldots,P_r\}$ ranges over partitions of $A$ into nonempty disjoint subsets. Since the derivatives of $F_l$ are uniformly bounded on $\mathcal{K}$, and by the induction hypothesis
\begin{equation}
    \left\|
    \partial_{P_j}x_l^{(\lambda)}(m)
    \right\|
    \leq
    C_{l,P_j,x}\lambda^{|P_j|},
\end{equation}
each term in Equation~\ref{eq:chain-rule-term} is bounded by a constant times
\begin{equation}
    \lambda^{|P_1|+\cdots+|P_r|}
    =
    \lambda^{|A|}.
\end{equation}
Therefore,
\begin{equation}
    \left\|
    \partial_A
    F_l\!\left(x_l^{(\lambda)}(m)\right)
    \right\|
    \leq
    C\lambda^{|A|}.
\end{equation}
Multiplying by the explicit factor $\lambda$ in Equation~\ref{eq:state-derivative-proof} yields a term of order $\lambda^{|A|+1}$, which is bounded by a constant times $\lambda^{|A|}$ for $0<\lambda\leq\lambda_0$. Together with the induction bound on $\partial_A x_l^{(\lambda)}(m)$, this gives
\begin{equation}
    \left\|
    \partial_A x_{l+1}^{(\lambda)}(m)
    \right\|
    \leq
    C_{l+1,A,x}\lambda^{|A|}.
\end{equation}

\paragraph{Case 2: $l\in A$.}
Let $A'=A\setminus\{l\}$. Since the state $x_l^{(\lambda)}(m)$ depends only on gates from earlier residual blocks, it does not depend on $m_l$. Therefore,
\begin{equation}
    \partial_A x_l^{(\lambda)}(m)=0.
\end{equation}
Moreover,
\begin{equation}
    \partial_A
    \left[
        m_l F_l\!\left(x_l^{(\lambda)}(m)\right)
    \right]
    =
    \partial_{A'}
    F_l\!\left(x_l^{(\lambda)}(m)\right).
\end{equation}
Using the same chain-rule argument as in Case 1, but now for $A'$, gives
\begin{equation}
    \left\|
    \partial_{A'}
    F_l\!\left(x_l^{(\lambda)}(m)\right)
    \right\|
    \leq
    C\lambda^{|A'|}
    =
    C\lambda^{|A|-1}.
\end{equation}
Multiplying by the explicit factor $\lambda$ in Equation~\ref{eq:state-derivative-proof}, we obtain
\begin{equation}
    \left\|
    \partial_A x_{l+1}^{(\lambda)}(m)
    \right\|
    \leq
    C_{l+1,A,x}\lambda^{|A|}.
\end{equation}

Both cases establish the induction step. Hence the result holds for all blocks $l$ and all subsets $A\subseteq[L]$.
\end{proof}

\begin{lemma}[Gate derivatives of the scalar readout scale with residual order]
\label{lem:readout-scaling}
Under the same assumptions, for every subset $S\subseteq[L]$, there exists a constant $B_{S,x}>0$, independent of $\lambda$, such that
\begin{equation}
    \sup_{m\in[0,1]^L}
    \left|
    \partial_S v_x^{(\lambda)}(m)
    \right|
    \leq
    B_{S,x}\lambda^{|S|}
    \qquad
    \text{for all } 0<\lambda\leq\lambda_0.
    \label{eq:readout-scaling}
\end{equation}
\end{lemma}

\begin{proof}
Recall that
\begin{equation}
    v_x^{(\lambda)}(m)
    =
    R\!\left(x_{L+1}^{(\lambda)}(m)\right).
\end{equation}
Applying the multivariate chain rule, $\partial_S v_x^{(\lambda)}(m)$ is a finite sum of terms of the form
\begin{equation}
    D^r R\!\left(x_{L+1}^{(\lambda)}(m)\right)
    \left[
        \partial_{P_1}x_{L+1}^{(\lambda)}(m),
        \ldots,
        \partial_{P_r}x_{L+1}^{(\lambda)}(m)
    \right],
\end{equation}
where $\{P_1,\ldots,P_r\}$ ranges over partitions of $S$ into nonempty disjoint subsets. Since the derivatives of $R$ are uniformly bounded and Lemma~\ref{lem:hidden-state-scaling} gives
\begin{equation}
    \left\|
    \partial_{P_j}x_{L+1}^{(\lambda)}(m)
    \right\|
    \leq
    C_{L+1,P_j,x}\lambda^{|P_j|},
\end{equation}
each term is bounded by a constant times
\begin{equation}
    \lambda^{|P_1|+\cdots+|P_r|}
    =
    \lambda^{|S|}.
\end{equation}
Summing over the finitely many partitions of $S$ gives Equation~\ref{eq:readout-scaling}.
\end{proof}

\begin{proof}[Proof of Proposition~\ref{prop:residual-scaling}]
Fix a subset $S\subseteq[L]$ with $|S|=k$. By Definition~\ref{def:mobius-residual-interaction}, the M\"obius residual interaction coefficient is
\begin{equation}
    \Delta_S v_x^{(\lambda)}
    =
    \sum_{T\subseteq S}
    (-1)^{|S|-|T|}
    v_x^{(\lambda)}(\mathbf{1}_T).
\end{equation}
Applying Lemma~\ref{lem:finite-difference-identity} with $g=v_x^{(\lambda)}$ gives
\begin{equation}
    \Delta_S v_x^{(\lambda)}
    =
    \int_{[0,1]^k}
    \partial_S v_x^{(\lambda)}(u_S,0_{\bar S})
    \,du_S.
\end{equation}
Taking absolute values and using Lemma~\ref{lem:readout-scaling}, we obtain
\begin{align}
    \left|
    \Delta_S v_x^{(\lambda)}
    \right|
    &\leq
    \int_{[0,1]^k}
    \left|
    \partial_S v_x^{(\lambda)}(u_S,0_{\bar S})
    \right|
    \,du_S \\
    &\leq
    \int_{[0,1]^k}
    B_{S,x}\lambda^k
    \,du_S \\
    &=
    B_{S,x}\lambda^k.
\end{align}
This proves Proposition~\ref{prop:residual-scaling}.
\end{proof}

\clearpage
\section{Naive Additive Reconstruction of Masked Responses}
\label{app:naive-path-sum}

This appendix provides the complete experimental details for testing whether
masked-subnetwork responses can be interpreted as independent additive
contributions. We compare three naive aggregations of the same residual-mask
evaluations with exact M\"obius reconstruction on ImageNet-pretrained
ResNet-18 and ResNet-34. Across both architectures, summation, averaging, and
shortcut-baseline centering fail to recover the full-network logits, whereas
M\"obius inversion reconstructs them to numerical precision. The experiment
therefore distinguishes masked responses, which are overlapping nonlinear
interventions, from the additive interaction coefficients derived from them.

\subsection{Motivation}

The path-ensemble view characterizes a ResNet as a collection of structural
paths, but it does not assign independent additive outputs to those paths. In
particular, the response $h_x(\mathbf{1}_S)$ obtained by activating a subset
$S$ of residual branches is the output of an entire surviving nonlinear
subnetwork. It retains shared stem and shortcut computation, normalization,
nonlinear activations, projection shortcuts, and the final classifier, and a
single mask may itself contain multiple structural paths. Consequently,
directly aggregating masked responses can repeatedly count shared and
lower-order effects.

We use naive aggregation as a diagnostic rather than as a claim about the
classical path-ensemble literature. The question is whether masked responses
nevertheless behave approximately like additive contribution units:
\emph{Can the logits of the full ResNet be reconstructed by directly summing,
averaging, or baseline-centering its masked-subnetwork responses?} A negative
answer motivates the use of M\"obius coefficients, which remove repeated
lower-order effects through inclusion--exclusion.

\subsection{Experimental Setup and Evaluation Metrics}
\label{app:naive-path-setup}

\paragraph{Models and Data.}
We use the torchvision ImageNet-pretrained ResNet-18 and ResNet-34 checkpoints
with \texttt{IMAGENET1K\_V1} weights in evaluation mode. ResNet-18 contains
$L=8$ BasicBlocks and is evaluated on $N=10{,}000$ images sampled uniformly
from the unlabeled ImageNet-1K test split with seed $0$. ResNet-34 contains
$L=16$ BasicBlocks and is evaluated on a fixed $N=1{,}000$-image subset of the
same test pool. We apply the standard preprocessing associated with the
pretrained checkpoints. Because the test images are unlabeled, no metric in
this experiment uses ground-truth classes.

\paragraph{Residual-branch Interventions.}
A binary gate $m_l\in\{0,1\}$ is placed on the residual branch of every
BasicBlock:
\begin{equation}
x_{l+1}(m)
=
\phi_l\!\left(
s_l(x_l(m))+m_lF_l(x_l(m))
\right),
\end{equation}
where $s_l$ is the identity or projection shortcut and $\phi_l$ is the
post-addition activation. Setting $m_l=0$ suppresses only the residual update;
the shortcut, including any projection used for downsampling or channel
matching, remains active.

For a subset $S\subseteq[L]$, let $\mathbf{1}_S\in\{0,1\}^L$ activate exactly
the residual branches in $S$, and let
$h_x(\mathbf{1}_S)\in\mathbb{R}^{1000}$ denote the corresponding pre-softmax
logits. We exhaustively evaluate all $2^8=256$ masks for each ResNet-18 input
and all $2^{16}=65{,}536$ masks for each ResNet-34 input. The full-network and
shortcut-only logits are
\begin{equation}
h_x^{\mathrm{full}}=h_x(\mathbf{1}_{[L]}),
\qquad
h_x^{\emptyset}=h_x(\mathbf{0}).
\end{equation}

\paragraph{Candidate Reconstructions.}
We compare three direct aggregations of the masked responses:
\begin{align}
A_x^{\mathrm{raw}}
&=
\sum_{S\subseteq[L]}h_x(\mathbf{1}_S),
\\
A_x^{\mathrm{mean}}
&=
2^{-L}\sum_{S\subseteq[L]}h_x(\mathbf{1}_S),
\\
A_x^{\mathrm{ctr}}
&=
h_x^{\emptyset}
+
\sum_{\emptyset\neq S\subseteq[L]}
\left(
h_x(\mathbf{1}_S)-h_x^{\emptyset}
\right).
\end{align}
The raw sum tests direct additivity, the mean tests an ensemble-style average,
and the centered sum removes repeated inclusion of the shortcut-only baseline.

As an exact control, we compute the vector-valued M\"obius coefficients
component-wise:
\begin{equation}
\Delta_S h_x
=
\sum_{T\subseteq S}
(-1)^{|S|-|T|}
h_x(\mathbf{1}_T),
\end{equation}
and reconstruct the full logits as
\begin{equation}
A_x^{\mathrm{Mob}}
=
\sum_{S\subseteq[L]}\Delta_S h_x
=
h_x^{\mathrm{full}}.
\end{equation}
The ResNet-18 transform is accumulated in FP32, while the larger ResNet-34
transform is accumulated in FP64 to control numerical cancellation.

\paragraph{Evaluation Metrics.}
For any candidate reconstruction $A_x$, we compute the norm ratio and relative
reconstruction error
\begin{equation}
Q(A_x)
=
\frac{\|A_x\|_2}{\|h_x^{\mathrm{full}}\|_2+\epsilon},
\qquad
R(A_x)
=
\frac{\|h_x^{\mathrm{full}}-A_x\|_2}
{\|h_x^{\mathrm{full}}\|_2+\epsilon},
\end{equation}
where the target values are $1$ and $0$, respectively. To determine whether a
failure is caused only by global scale, we fit the per-image least-squares
scalar
\begin{equation}
\alpha_x^\star
=
\frac{
\langle h_x^{\mathrm{full}},A_x\rangle
}{
\|A_x\|_2^2+\epsilon
}
\end{equation}
and report the optimally scaled error
\begin{equation}
R_{\mathrm{scale}}(A_x)
=
\frac{
\|h_x^{\mathrm{full}}-\alpha_x^\star A_x\|_2
}{
\|h_x^{\mathrm{full}}\|_2+\epsilon
}.
\end{equation}
We additionally report the cosine similarity
\begin{equation}
\operatorname{Cos}(A_x,h_x^{\mathrm{full}})
=
\frac{
\langle A_x,h_x^{\mathrm{full}}\rangle
}{
\|A_x\|_2\|h_x^{\mathrm{full}}\|_2+\epsilon
}
\end{equation}
and agreement with the full model's top-1 prediction,
\begin{equation}
\operatorname{Agree}(A_x)
=
\mathbf{1}\!\left[
\arg\max_c A_{x,c}
=
\arg\max_c h_{x,c}^{\mathrm{full}}
\right].
\end{equation}
Top-1 agreement is prediction agreement with the full model, not classification
accuracy. We use $\epsilon=10^{-12}$ and report per-image means and standard
errors.

\subsection{Results and Analysis}
\label{app:naive-path-results}

Table~\ref{tab:naive-vs-mobius} reports the complete numerical results, and
Figure~\ref{fig_appendix:naive-path-sum-fails} provides the corresponding
visualization.

\begin{table}[t]
\centering
\caption{Naive aggregation versus exact M\"obius reconstruction.
Entries are per-image means over fixed ImageNet-1K test subsets. ResNet-18 uses
$N=10{,}000$ images and $256$ masks per image; ResNet-34 uses $N=1{,}000$ images
and $65{,}536$ masks per image. Top-1 agreement is measured against the
corresponding full-model prediction.}
\label{tab:naive-vs-mobius}
\resizebox{\textwidth}{!}{
\begin{tabular}{llccccc}
\toprule[1.5pt]
Model & Reconstruction & Norm ratio $(\rightarrow 1)$ & Rel.\ error $(\downarrow)$ & Scaled error $(\downarrow)$ & Cosine $(\uparrow)$ & Top-1 agree $(\uparrow)$ \\
\midrule
ResNet-18 & Raw sum      & $79.366$ & $79.150$ & $0.9606$ & $0.2217$ & $0.0064$ \\
          & Mean output  & $0.3100$ & $0.9771$ & $0.9606$ & $0.2217$ & $0.0064$ \\
          & Centered sum & $77.412$ & $77.190$ & $0.9595$ & $0.2288$ & $0.0062$ \\
          & M\"obius     & $1.0000$ & $8.506{\times}10^{-7}$ & $8.550{\times}10^{-7}$ & $1.0000$ & $1.0000$ \\
\midrule
ResNet-34 & Raw sum      & $18{,}852.3$ & $18{,}851.9$ & $0.8890$ & $0.4400$ & $0.0250$ \\
          & Mean output  & $0.2877$ & $0.9104$ & $0.8890$ & $0.4400$ & $0.0250$ \\
          & Centered sum & $17{,}750.2$ & $17{,}749.8$ & $0.8787$ & $0.4587$ & $0.0240$ \\
          & M\"obius     & $1.0000$ & $0$ & $1.76{\times}10^{-16}$ & $1.0000$ & $1.0000$ \\
\bottomrule[1.5pt]
\end{tabular}
}
\end{table}

\begin{figure}[t]
\centering
\includegraphics[width=\textwidth]{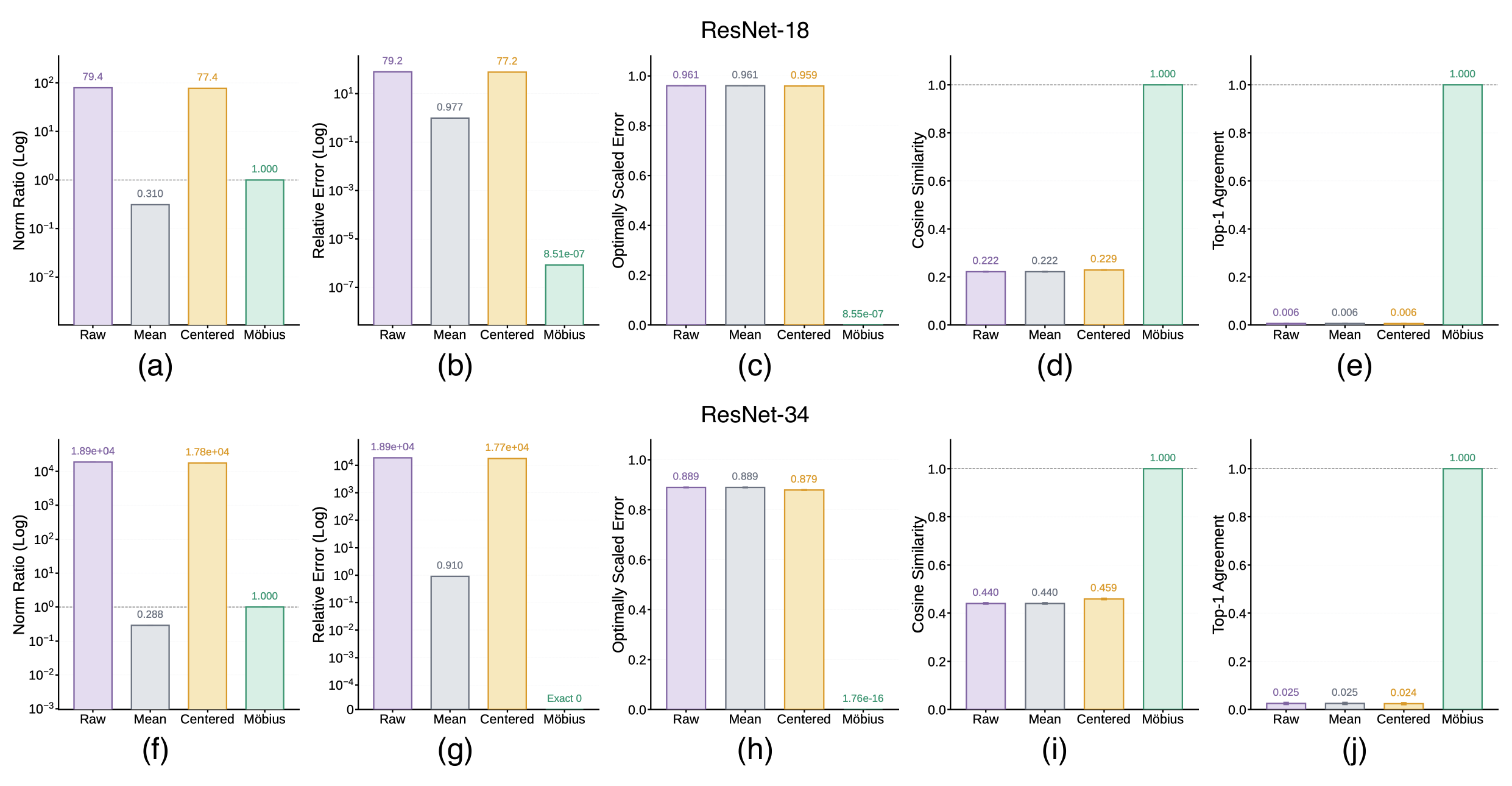}
\caption{Naive masked-response aggregation fails, whereas M\"obius
inversion reconstructs the full logits. The top and bottom rows show
ResNet-18 and ResNet-34, respectively. Columns report norm ratio, relative
error, optimally scaled error, cosine similarity, and top-1 agreement with the
full model. Bars and error bars denote per-image means and standard errors over
the fixed ImageNet-1K test subsets.}
\label{fig_appendix:naive-path-sum-fails}
\end{figure}

All three naive reconstructions fail on both architectures. The raw and
centered sums severely inflate the logit norm. Their mean norm ratios are
$79.4$ and $77.4$ for ResNet-18, and approximately
$1.89\times10^4$ and $1.78\times10^4$ for ResNet-34. The centered result shows
that this failure cannot be explained solely by repeatedly including the
shortcut-only baseline.

Averaging prevents the magnitude explosion, but does not recover the
full-model logits. The mean-output norm ratios are $0.310$ and $0.288$, while
the corresponding relative errors remain $0.977$ and $0.910$ for ResNet-18
and ResNet-34. Thus, an unweighted ensemble average of masked subnetworks is
also not a faithful approximation.

Optimal scalar rescaling does not resolve the discrepancy. The scaled errors
remain between $0.959$ and $0.961$ for ResNet-18 and between $0.879$ and
$0.889$ for ResNet-34. Cosine similarities remain low, ranging from
$0.222$ to $0.229$ and from $0.440$ to $0.459$, respectively. Correspondingly,
the naive reconstructions preserve the full model's top-1 prediction on only
$0.62$--$0.64\%$ of ResNet-18 inputs and $2.4$--$2.5\%$ of ResNet-34 inputs.
The failure is therefore directional rather than merely a global scale
mismatch.

In contrast, M\"obius inversion reconstructs the full logits up to numerical
precision using exactly the same masked evaluations. For ResNet-18, the mean
relative error is $8.51\times10^{-7}$; for the FP64 ResNet-34 transform, it is
zero at the reported precision. Both reconstructions have unit norm ratio and
cosine similarity and preserve the full-model prediction on every evaluated
image.

These results do not challenge the structural path-ensemble interpretation.
Rather, they establish that a masked-subnetwork response is an overlapping
nonlinear intervention, not an independent additive path contribution.
Summation, averaging, baseline centering, and optimal rescaling cannot remove
this overlap. The exact additive units of the residual-mask response are the
M\"obius coefficients $\Delta_S h_x$, whose inclusion--exclusion construction
removes effects already represented by lower-order subsets. This distinction
motivates the functional residual-interaction analysis used throughout the
paper.


\clearpage
\section{Interaction-Order Spectrum}
\label{app:interaction-order-spectrum}

This section provides the full protocol and numerical results for the
interaction-order analysis summarized in
Figure~\ref{fig:interaction-order-spectrum}. We exhaustively evaluate all
residual-branch masks of pretrained ResNet-18 and ResNet-34 models and examine
how their M\"obius squared coefficient mass is distributed across interaction
orders. At the natural operating point, $\lambda=1$, both models are dominated
by mid-to-high-order interactions: their spectra peak at orders $5$ and $10$,
respectively, above the corresponding midpoints $L/2=4$ and $L/2=8$.

\subsection{Motivation}

The exact reconstruction experiment in
Appendix~\ref{app:naive-path-sum} establishes that M\"obius coefficients form
an additive decomposition of the residual-mask response. We next ask how the
magnitude of these coefficients is distributed across interaction orders. For
a subset $S\subseteq[L]$, the order $|S|$ counts the number of residual
branches involved in the corresponding interaction. It should not be confused
with structural path length: branches in a low-order interaction may be far
apart in depth, while a short structural path need not induce a low-order
functional interaction.

Our diagnostic question is therefore whether a pretrained ResNet, at its
natural operating point, is dominated by individual and low-order residual
corrections or by interactions involving larger groups of branches. This
experiment also establishes the $\lambda=1$ reference point for the residual
scaling intervention. Proposition~\ref{prop:residual-scaling} characterizes
the attenuation of each fixed $k$-way coefficient as
$\lambda\rightarrow0$; it does not require the unscaled network itself to be
low-order dominated.

\subsection{Experimental Setup and Evaluation Metrics}

\paragraph{Models and Data.}
We evaluate torchvision ResNet-18 and ResNet-34 checkpoints pretrained with
the \texttt{IMAGENET1K\_V1} weights. Their residual stages contain
$[2,2,2,2]$ and $[3,4,6,3]$ BasicBlocks, giving $L=8$ and $L=16$ gated
residual branches, respectively. Each block computes
$x_{l+1}(m)=\phi_l\!\left(s_l(x_l(m))+m_lF_l(x_l(m))\right)$, so setting
$m_l=0$ suppresses only the residual branch while preserving the identity or
projection shortcut.

For ResNet-18, we use a fixed set of $N=10{,}000$ images from the unlabeled
ImageNet-1K test split and enumerate all $2^8=256$ masks per image. For
ResNet-34, we use a deterministic $N=1{,}000$-image subset of the same test
pool and enumerate all $2^{16}=65{,}536$ masks. The smaller ResNet-34 sample
reflects the substantially higher cost of exhaustive mask enumeration. Both
models are evaluated in inference mode using the preprocessing associated
with their pretrained weights.

\paragraph{Readouts and Interaction Coefficients.}
For each input $x$, let $h_x(m)\in\mathbb{R}^{1000}$ denote the logit vector
under mask $m$. The scalar readout is the logit of the class predicted by the
unmasked model,
\[
v_x(m)=h_{x,\hat y(x)}(m),
\qquad
\hat y(x)=\arg\max_c h_{x,c}(\mathbf{1}_{[L]}),
\]
where $\hat y(x)$ is fixed across all masks. We apply M\"obius inversion
component-wise to obtain the scalar coefficients $\Delta_Sv_x$ and vector
coefficients $\Delta_Sh_x$. Their magnitudes are
$a_S^v(x)=|\Delta_Sv_x|$ and $a_S^h(x)=\|\Delta_Sh_x\|_2$.

\paragraph{Order Statistics.}
For either readout $r\in\{v,h\}$, the squared coefficient mass at order $k$
and the total non-constant mass are
\[
E_k^r(x)=
\sum_{\substack{S\subseteq[L]\\|S|=k}}a_S^r(x)^2,
\qquad
Z^r(x)=\sum_{j=1}^{L}E_j^r(x).
\]
We report the normalized spectrum, cumulative low-order mass, effective
interaction order, and mean per-interaction mass:
\[
\widetilde E_k^r(x)=\frac{E_k^r(x)}{Z^r(x)},
\qquad
C_{\le q}^r(x)=\frac{\sum_{k=1}^{q}E_k^r(x)}{Z^r(x)},
\]
\[
\kappa^r(x)=
\frac{\sum_{k=1}^{L}kE_k^r(x)}{Z^r(x)},
\qquad
M_k^r(x)=\frac{E_k^r(x)}{\binom{L}{k}}.
\]
The normalized spectrum, cumulative mass, and effective order exclude the
zeroth-order shortcut-only coefficient. We include $M_0^r=E_0^r$ in the
figure only as a baseline reference. The quantity $M_k^r$ controls for the
different numbers $\binom{L}{k}$ of subsets available at each order.

All quantities are first computed and normalized separately for each image
and then averaged across images. We use ``interaction energy'' as shorthand
for squared M\"obius coefficient mass. Because the coefficients are not
generally orthogonal, $\widetilde E_k^r$ should not be interpreted as the
fraction of the full-logit norm explained by order $k$.

\subsection{Results and Analysis}
\label{app:interaction-order-spectrum-results}

\begin{figure*}[t]
    \centering
    \includegraphics[width=\textwidth]{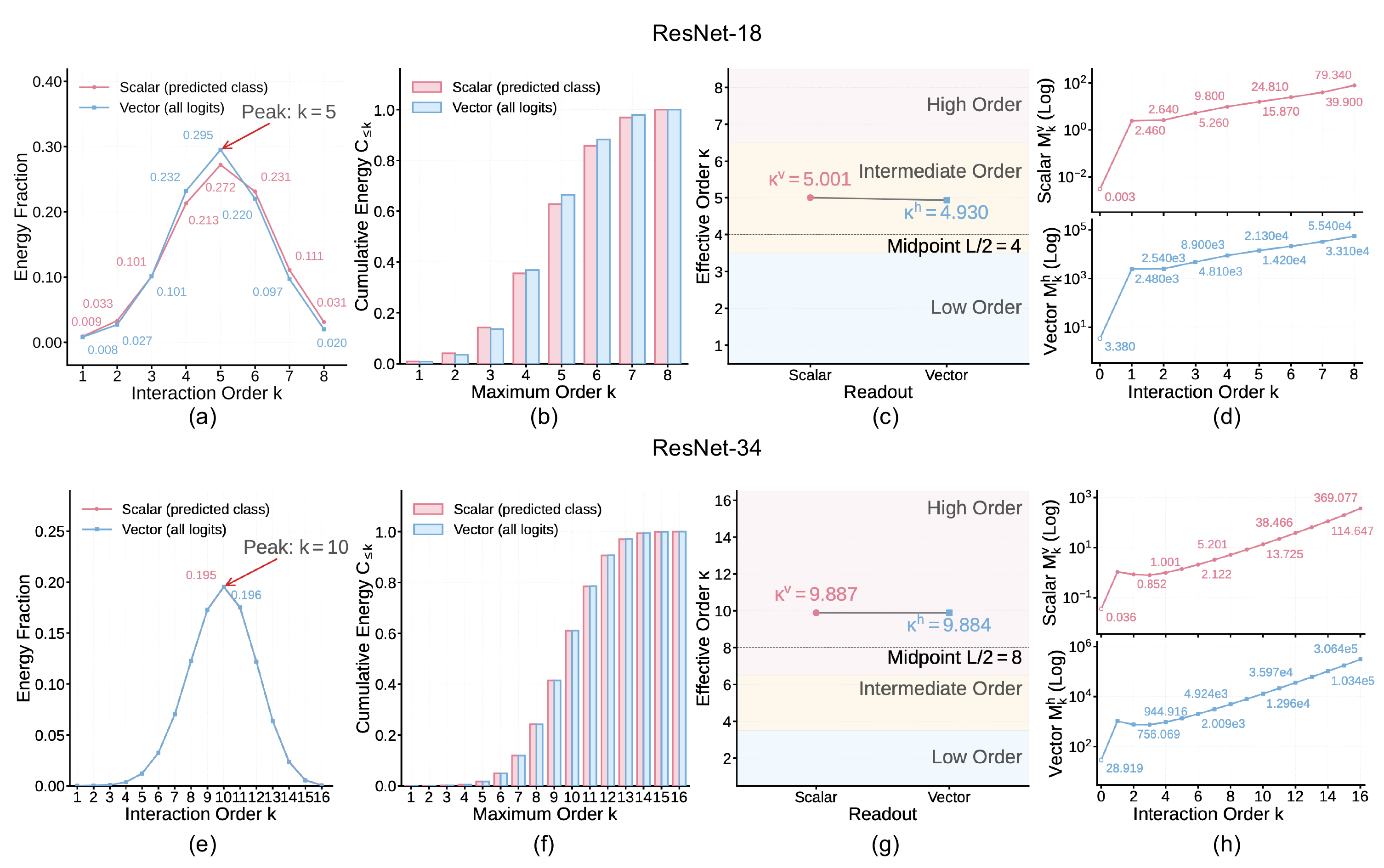}
    \caption{Residual interaction-order spectra at the natural
    operating point. Results are shown for ResNet-18
    (top; $N=10{,}000$, $L=8$) and ResNet-34
    (bottom; $N=1{,}000$, $L=16$). From left to right, the columns show
    the normalized squared coefficient mass $\widetilde E_k$, cumulative
    low-order mass $C_{\le k}$, effective interaction order $\kappa$, and
    mean per-interaction mass $M_k$. Scalar readouts use the class predicted
    by the full model, while vector readouts use all logits. Both models peak
    above $L/2$, and the large high-order $M_k$ values show that this pattern
    is not explained solely by the number of subsets at each order.}
    \label{fig_appendix:interaction-order-spectrum}
\end{figure*}

\begin{table*}[t]
\centering
\small
\setlength{\tabcolsep}{5pt}
\caption{Summary of the interaction-order spectra.
$C_{\le L/2}$ includes the midpoint order, while $M_L/M_1$ compares the
dataset-mean per-interaction masses at the highest and first orders. Values
are rounded to approximately four significant digits.}
\label{tab:exp2-order-summary}
\begin{tabular}{llccccc}
\toprule[1.5pt]
\textbf{Model} & \textbf{Readout}
& \textbf{Peak $k$ ($\widetilde E_k$)}
& \boldmath{$C_{\le3}$}
& \boldmath{$C_{\le L/2}$}
& \boldmath{$\kappa$}
& \boldmath{$M_L/M_1$} \\
\midrule
ResNet-18
& Scalar
& $5\;(0.2721)$
& $0.1417$
& $0.3545$
& $5.001$
& $32.21$ \\
& Vector
& $5\;(0.2951)$
& $0.1356$
& $0.3675$
& $4.929$
& $22.31$ \\
\midrule
ResNet-34
& Scalar
& $10\;(0.1951)$
& $0.001153$
& $0.2421$
& $9.887$
& $341.4$ \\
& Vector
& $10\;(0.1956)$
& $0.001116$
& $0.2419$
& $9.884$
& $296.4$ \\
\bottomrule[1.5pt]
\end{tabular}
\end{table*}

\begin{table*}[t]
\centering
\small
\setlength{\tabcolsep}{7pt}
\caption{Complete normalized interaction-order spectra.
Entries report the dataset-mean normalized squared coefficient mass
$\widetilde E_k$. Six decimal places are retained so that the small
low-order values of ResNet-34 are not rounded to zero.}
\label{tab:exp2-full-spectrum}
\begin{tabular}{c cc cc}
\toprule[1.5pt]
& \multicolumn{2}{c}{\textbf{ResNet-18}}
& \multicolumn{2}{c}{\textbf{ResNet-34}} \\
\cmidrule(lr){2-3}\cmidrule(lr){4-5}
\textbf{$k$}
& \textbf{Scalar} & \textbf{Vector}
& \textbf{Scalar} & \textbf{Vector} \\
\midrule
1  & $0.008619$ & $0.007611$ & $0.000038$ & $0.000036$ \\
2  & $0.032546$ & $0.027096$ & $0.000212$ & $0.000196$ \\
3  & $0.100549$ & $0.100901$ & $0.000903$ & $0.000884$ \\
4  & $0.212804$ & $0.231893$ & $0.003625$ & $0.003595$ \\
5  & $\mathbf{0.272092}$ & $\mathbf{0.295098}$
   & $0.012134$ & $0.012054$ \\
6  & $0.231087$ & $0.220151$ & $0.032453$ & $0.032361$ \\
7  & $0.111121$ & $0.096990$ & $0.070109$ & $0.070110$ \\
8  & $0.031182$ & $0.020260$ & $0.122593$ & $0.122682$ \\
9  & -- & -- & $0.172787$ & $0.172977$ \\
10 & -- & -- & $\mathbf{0.195111}$ & $\mathbf{0.195567}$ \\
11 & -- & -- & $0.174654$ & $0.175310$ \\
12 & -- & -- & $0.121618$ & $0.121884$ \\
13 & -- & -- & $0.063775$ & $0.063326$ \\
14 & -- & -- & $0.023746$ & $0.023134$ \\
15 & -- & -- & $0.005600$ & $0.005306$ \\
16 & -- & -- & $0.000641$ & $0.000577$ \\
\bottomrule[1.5pt]
\end{tabular}
\end{table*}

\paragraph{Mid-to-high-order interactions dominate.}
Figure~\ref{fig_appendix:interaction-order-spectrum}(a,e) and
Table~\ref{tab:exp2-full-spectrum} show that both readouts peak at $k=5$ for
ResNet-18 and at $k=10$ for ResNet-34. For ResNet-18, orders $4$--$6$
contain $71.6\%$ of the scalar mass and $74.7\%$ of the vector mass, whereas
orders $1$--$3$ contain only $14.17\%$ and $13.56\%$. For ResNet-34,
orders $8$--$12$ account for $78.68\%$ and $78.84\%$, while orders
$1$--$3$ contribute only approximately $0.11\%$. The cumulative curves in
Figure~\ref{fig_appendix:interaction-order-spectrum}(b,f) further show that
$63.25\%$--$64.55\%$ of the ResNet-18 mass and approximately $75.8\%$ of
the ResNet-34 mass lie above their respective midpoints. Consistently, the
effective orders in Figure~\ref{fig_appendix:interaction-order-spectrum}(c,g) and
Table~\ref{tab:exp2-order-summary} are above $L/2$ for both models.

\paragraph{The Pattern is not purely combinatorial.}
The total mass $E_k$ may be elevated near the middle orders because those
orders contain more subsets. Figure~\ref{fig_appendix:interaction-order-spectrum}(d,h)
therefore reports the mean per-interaction mass
$M_k=E_k/\binom{L}{k}$. For ResNet-18, $M_k$ increases with order and
reaches its maximum at $k=L$, with $M_L/M_1=32.21$ for the scalar readout
and $22.31$ for the vector readout. For ResNet-34, $M_k$ decreases slightly
over orders $1$--$3$ but then rises sharply, reaching
$M_L/M_1=341.4$ and $296.4$, respectively. These ratios are summarized in
Table~\ref{tab:exp2-order-summary}. Thus, the mid-to-high-order concentration
cannot be explained solely by the combinatorial prevalence of middle-order
subsets.

\paragraph{Readout agreement and interpretation.}
Across both architectures, the scalar and vector curves in
Figure~\ref{fig_appendix:interaction-order-spectrum} have closely aligned peak
locations, cumulative profiles, and effective orders. The numerical
comparison in Table~\ref{tab:exp2-order-summary} confirms this agreement,
indicating that the conclusion is not specific to the predicted-class
readout. At $\lambda=1$, both studied checkpoints therefore operate in a
mid-to-high-order functional interaction regime. This observation does not
contradict Proposition~\ref{prop:residual-scaling}, which characterizes the
attenuation of fixed-order coefficients as $\lambda\rightarrow0$. The
subsequent residual-scaling experiment tests whether reducing $\lambda$
causally shifts the observed spectra toward lower orders.

\clearpage
\section{Residual-Scaling Sweep}
\label{app:residual-scaling-sweep}

This section examines how uniform inference-time residual scaling changes the
functional interaction spectrum. For both ResNet-18 and ResNet-34, decreasing
the residual scale $\lambda$ monotonically reduces the effective interaction
order and transfers squared coefficient mass toward orders $1$--$3$.
However, substantial low-order concentration emerges only under strong
scaling, where the model's predictions differ almost completely from those at
$\lambda=1$. Residual scaling therefore provides causal control over the
interaction spectrum, but the observed shift is not prediction preserving.

\subsection{Motivation}

Appendix~\ref{app:interaction-order-spectrum} shows that the unscaled
ResNet-18 and ResNet-34 checkpoints operate in mid-to-high-order interaction
regimes. Their normalized spectra peak at $k=5$ and $k=10$, with effective
orders near $5$ and $10$, respectively. Proposition~\ref{prop:residual-scaling}
predicts that, for a fixed readout, a $k$-way coefficient is attenuated as
$\mathcal{O}(\lambda^k)$ when residual updates are uniformly scaled by
$\lambda$. Higher-order coefficients should therefore decay faster than
lower-order coefficients as $\lambda$ decreases.

We test this prediction through an inference-time intervention, without
retraining. The experiment asks two related questions: whether decreasing
$\lambda$ causally shifts the interaction spectrum toward lower orders, and
whether such a shift can occur while preserving the predictions of the
unscaled model.

\subsection{Experimental Setup and Evaluation Metrics}

\paragraph{Models, Data, and Intervention.}
We use the same pretrained models, test images, residual gates, and exhaustive
mask evaluations as in Appendix~\ref{app:interaction-order-spectrum}:
ResNet-18 with $L=8$ and $N=10{,}000$, and ResNet-34 with $L=16$ and
$N=1{,}000$. For every binary mask $m\in\{0,1\}^L$, the scaled block
computes
\[
x_{l+1}^{(\lambda)}(m)
=
\phi_l\!\left(
s_l(x_l^{(\lambda)}(m))
+
\lambda m_lF_l(x_l^{(\lambda)}(m))
\right).
\]
Thus, an active residual branch is multiplied by $\lambda$, an inactive
branch remains suppressed, and the shortcut is always preserved. We sweep
\[
\lambda\in
\{1.00,0.95,0.90,0.85,0.80,0.75,0.50,0.25,0.10\}
\]
without modifying the pretrained parameters.

\paragraph{Readouts.}
Let $h_x^{(\lambda)}(m)$ denote the logits under residual scale $\lambda$.
For the scalar spectrum, we select the class predicted by the scaled
full-mask model,
\[
\hat y_\lambda(x)
=
\arg\max_c h_{x,c}^{(\lambda)}(\mathbf{1}_{[L]}),
\qquad
v_x^{(\lambda)}(m)
=
h_{x,\hat y_\lambda(x)}^{(\lambda)}(m).
\]
The class $\hat y_\lambda(x)$ is fixed across all masks at a given
$\lambda$, but it may change across values of $\lambda$. The vector readout
uses the complete logit vector and is unaffected by this coordinate choice.

\paragraph{Spectrum Metrics.}
For each $\lambda$, we recompute the scalar and vector M\"obius coefficients
and the order-wise squared coefficient masses
$E_k^{r,(\lambda)}(x)$ for $r\in\{v,h\}$. We summarize the spectrum using
\[
\kappa^{r,(\lambda)}(x)
=
\frac{
\sum_{k=1}^{L}kE_k^{r,(\lambda)}(x)
}{
\sum_{k=1}^{L}E_k^{r,(\lambda)}(x)
},
\qquad
C_{\le3}^{r,(\lambda)}(x)
=
\frac{
\sum_{k=1}^{3}E_k^{r,(\lambda)}(x)
}{
\sum_{k=1}^{L}E_k^{r,(\lambda)}(x)
}.
\]
A decrease in $\kappa$ and an increase in $C_{\le3}$ indicate a shift toward
lower-order interactions. The complementary high-order mass is
$T_{>3}=1-C_{\le3}$ and is therefore omitted from the tables.

\paragraph{Prediction Agreement.}
To quantify how much the intervention changes the model's behavior, we
compare the full-mask prediction at each $\lambda$ with the prediction at
$\lambda=1$:
\[
A(\lambda)
=
\frac{1}{N}
\sum_{x}
\mathbf{1}\!\left[
\arg\max_c h_{x,c}^{(\lambda)}(\mathbf{1}_{[L]})
=
\arg\max_c h_{x,c}^{(1)}(\mathbf{1}_{[L]})
\right].
\]
This quantity is shared by the scalar and vector analyses. Because the
ImageNet test images are unlabeled, $A(\lambda)$ measures prediction
agreement rather than classification accuracy. All reported quantities are
computed per image and then averaged across images.

\subsection{Results and Analysis}
\label{app:residual-scaling-results}

\begin{figure}[t]
    \centering
    \includegraphics[width=\textwidth]
    {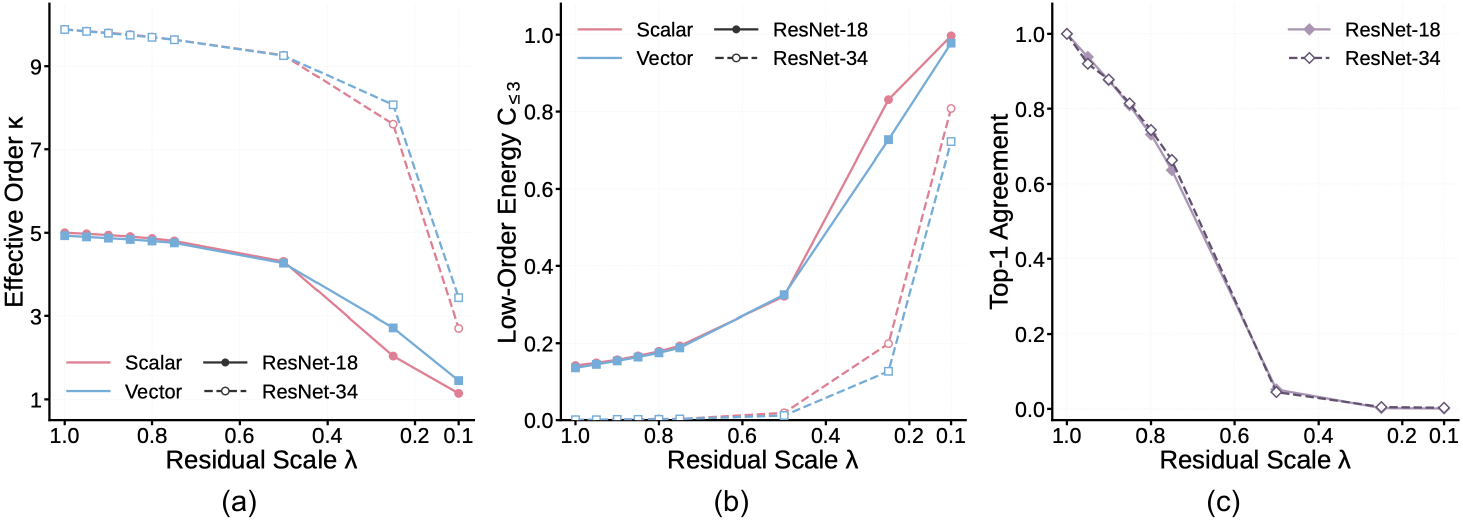}
    \caption{Inference-time residual-scaling sweep.
    Results are shown for pretrained ResNet-18
    ($N=10{,}000$) and ResNet-34 ($N=1{,}000$), without retraining.
    Solid and dashed lines denote ResNet-18 and ResNet-34, respectively.
    Pink and blue denote scalar and vector readouts.
    (a) Effective interaction order $\kappa$.
    (b) Cumulative low-order squared coefficient mass $C_{\le3}$.
    (c) Top-1 agreement between the scaled full-mask prediction and the
    corresponding prediction at $\lambda=1$. Decreasing $\lambda$ shifts
    both models toward lower-order interactions, but large shifts coincide
    with substantial changes in prediction.}
    \label{fig_appendix:residual-scaling-sweep}
\end{figure}

\begin{table*}[t]
\centering
\small
\setlength{\tabcolsep}{7pt}
\caption{Residual-scaling sweep for ResNet-18.
Entries are dataset means over $N=10{,}000$ images and are rounded to four
decimal places. Agreement compares the full-mask prediction at each
$\lambda$ with that at $\lambda=1$.}
\label{tab:exp3-r18-scaling}
\begin{tabular}{cccccc}
\toprule[1.5pt]
\boldmath{$\lambda$}
& \boldmath{$\kappa^v$}
& \boldmath{$\kappa^h$}
& \boldmath{$C_{\le3}^v$}
& \boldmath{$C_{\le3}^h$}
& \textbf{Top-1 agreement} \\
\midrule
1.00 & 5.0009 & 4.9295 & 0.1417 & 0.1356 & 1.0000 \\
0.95 & 4.9756 & 4.9002 & 0.1491 & 0.1447 & 0.9375 \\
0.90 & 4.9448 & 4.8706 & 0.1573 & 0.1541 & 0.8759 \\
0.85 & 4.9086 & 4.8389 & 0.1671 & 0.1640 & 0.8093 \\
0.80 & 4.8635 & 4.8022 & 0.1786 & 0.1751 & 0.7314 \\
0.75 & 4.8042 & 4.7572 & 0.1933 & 0.1882 & 0.6365 \\
0.50 & 4.3103 & 4.2689 & 0.3211 & 0.3254 & 0.0518 \\
0.25 & 2.0431 & 2.7209 & 0.8314 & 0.7268 & 0.0019 \\
0.10 & 1.1444 & 1.4517 & 0.9968 & 0.9777 & 0.0014 \\
\bottomrule[1.5pt]
\end{tabular}
\end{table*}

\begin{table*}[t]
\centering
\small
\setlength{\tabcolsep}{7pt}
\caption{Residual-scaling sweep for ResNet-34.
Entries are dataset means over $N=1{,}000$ images and are rounded to four
decimal places. Agreement compares the full-mask prediction at each
$\lambda$ with that at $\lambda=1$.}
\label{tab:exp3-r34-scaling}
\begin{tabular}{cccccc}
\toprule[1.5pt]
\boldmath{$\lambda$}
& \boldmath{$\kappa^v$}
& \boldmath{$\kappa^h$}
& \boldmath{$C_{\le3}^v$}
& \boldmath{$C_{\le3}^h$}
& \textbf{Top-1 agreement} \\
\midrule
1.00 & 9.8872 & 9.8835 & 0.0012 & 0.0011 & 1.0000 \\
0.95 & 9.8462 & 9.8390 & 0.0014 & 0.0013 & 0.9200 \\
0.90 & 9.8027 & 9.7938 & 0.0017 & 0.0016 & 0.8780 \\
0.85 & 9.7616 & 9.7474 & 0.0020 & 0.0020 & 0.8140 \\
0.80 & 9.7013 & 9.6957 & 0.0025 & 0.0024 & 0.7430 \\
0.75 & 9.6391 & 9.6373 & 0.0032 & 0.0030 & 0.6630 \\
0.50 & 9.2639 & 9.2544 & 0.0188 & 0.0118 & 0.0450 \\
0.25 & 7.6038 & 8.0708 & 0.1986 & 0.1267 & 0.0050 \\
0.10 & 2.7084 & 3.4369 & 0.8082 & 0.7219 & 0.0030 \\
\bottomrule[1.5pt]
\end{tabular}
\end{table*}

\paragraph{Residual scaling shifts both spectra toward lower orders.}
Figure~\ref{fig_appendix:residual-scaling-sweep}(a,b) and
Tables~\ref{tab:exp3-r18-scaling}--\ref{tab:exp3-r34-scaling} show that
$\kappa$ decreases and $C_{\le3}$ increases monotonically as $\lambda$
decreases. In ResNet-18, the scalar and vector effective orders fall from
$5.0009$ and $4.9295$ at $\lambda=1$ to $1.1444$ and $1.4517$ at
$\lambda=0.1$. Over the same range, $C_{\le3}$ rises from $0.1417$ and
$0.1356$ to $0.9968$ and $0.9777$. ResNet-34 exhibits the same directional
response: its effective orders decrease from $9.8872$ and $9.8835$ to
$2.7084$ and $3.4369$, while $C_{\le3}$ increases from approximately
$0.001$ to $0.8082$ and $0.7219$.

\paragraph{The deeper model shifts more slowly.}
At a fixed residual scale, ResNet-34 retains substantially more high-order
mass than ResNet-18. At $\lambda=0.25$, for example, the ResNet-18
low-order masses are $0.8314$ and $0.7268$, whereas the corresponding
ResNet-34 values are only $0.1986$ and $0.1267$. ResNet-34 requires scaling
to $\lambda=0.1$ before most of its squared coefficient mass moves into
orders $1$--$3$. Thus, the direction of the intervention is consistent
across architectures, but the scale required to reach a low-order regime
depends strongly on network depth.

\paragraph{Large spectral shifts do not preserve predictions.}
Figure~\ref{fig_appendix:residual-scaling-sweep}(c) shows that prediction agreement
decreases much faster than the effective interaction order. At
$\lambda=0.95$, agreement is still $0.9375$ for ResNet-18 and $0.9200$ for
ResNet-34, but the changes in $\kappa$ and $C_{\le3}$ remain small. At
$\lambda=0.8$, approximately one quarter of the predictions have already
changed, while both networks remain close to their original interaction
orders. By $\lambda=0.5$, agreement has fallen to $0.0518$ and $0.0450$,
yet ResNet-18 still has $\kappa\approx4.3$ and ResNet-34 has
$\kappa\approx9.3$. The strongest low-order concentration occurs only at
$\lambda\in\{0.25,0.1\}$, where agreement with the unscaled predictions is
nearly zero.

\paragraph{Interpretation.}
The monotonic decrease in $\kappa$ and increase in $C_{\le3}$ are consistent
with the faster attenuation of high-order coefficients predicted by
Proposition~\ref{prop:residual-scaling}. The experiment should nevertheless
be interpreted as directional evidence rather than a direct verification of
the asymptotic bound: the proposition is coefficient-wise, whereas
$\kappa$ and $C_{\le3}$ are normalized aggregate statistics. Moreover, the
scalar coordinate may change across $\lambda$ because it follows the scaled
model's predicted class. The vector readout is unaffected by this choice and
shows the same qualitative trend. Overall, residual scaling causally moves
the functional decomposition toward lower orders, but this shift does not
occur without substantially changing the model's predictions.

\clearpage
\section{Top-$K$ Residual-Interaction Reconstruction}
\label{app:topk-sparse-reconstruction}

This section tests whether the exact residual-interaction dictionary is
compressible at the level of individual terms. We rank the non-empty
M\"obius coefficients separately for each input and reconstruct the scalar
predicted-class readout or the full logit vector from the largest $K$ terms.
The results reveal a consistent distinction: squared coefficient mass is
concentrated, but faithful output reconstruction can require a substantially
larger support. This distinction is modest for ResNet-18 and pronounced for
ResNet-34. The selected sets are post-hoc analytical supports and do not imply
sparse execution.

\subsection{Motivation}

The interaction-order spectrum in
Appendix~\ref{app:interaction-order-spectrum} describes how coefficient mass
is distributed across orders $|S|$, but it does not determine how many
individual terms carry that mass. A mid- or high-order spectrum may still be
concentrated in relatively few subsets. We therefore ask: \emph{How many
input-specific interaction terms are required to capture the coefficient mass,
approximate the logits, and preserve the full model's prediction?} These are
different criteria because signed scalar coefficients can cancel and
vector-valued coefficients are not generally orthogonal.

\subsection{Experimental Setup and Evaluation Metrics}

\paragraph{Models and Residual-mask Evaluations.}
We analyze torchvision ResNet-18 and ResNet-34 checkpoints pretrained on
ImageNet-1K. ResNet-18 contains $L=8$ gateable residual branches and is
evaluated on a fixed set of $N=10{,}000$ unlabeled ImageNet-1K test images;
ResNet-34 contains $L=16$ branches and is evaluated on a fixed set of
$N=1{,}000$ test images. For every image, we exhaustively evaluate all $2^L$
binary residual-branch masks while preserving the identity or projection
shortcut. M\"obius inversion then yields $M=2^L-1$ non-empty interaction
coefficients: $M=255$ for ResNet-18 and $M=65{,}535$ for ResNet-34. The scalar
readout $v_x$ is the logit of the class predicted by the all-open model, with
that class fixed across masks; $h_x$ denotes the complete logit vector.

\paragraph{Per-input Magnitude Selection and Reconstruction.}
For each input $x$, scalar and vector interactions are ranked independently by
$a_S^v(x)=|\Delta_Sv_x|$ and
$a_S^h(x)=\|\Delta_Sh_x\|_2$. Let
$\mathcal A_K^v(x)$ and $\mathcal A_K^h(x)$ denote the corresponding $K$
largest non-empty subsets. The baseline coefficient is always retained, and
the signed reconstructions are
\begin{equation}
\begin{aligned}
\widehat v_{x,K}
&=\Delta_{\emptyset}v_x
  +\sum_{S\in\mathcal A_K^v(x)}\Delta_Sv_x,\\
\widehat h_{x,K}
&=\Delta_{\emptyset}h_x
  +\sum_{S\in\mathcal A_K^h(x)}\Delta_Sh_x.
\end{aligned}
\label{eq:app-topk-reconstruction}
\end{equation}
Magnitude selection maximizes retained squared coefficient mass for a fixed
$K$; it does not directly minimize reconstruction error or maximize prediction
agreement. We therefore refer to it as \emph{per-input magnitude Top-$K$}, not
as an oracle reconstruction.

\paragraph{Metrics.}
For either readout $r\in\{v,h\}$, the retained squared coefficient mass is
\begin{equation}
\Gamma_K^r(x)
=
\frac{\sum_{S\in\mathcal A_K^r(x)}a_S^r(x)^2}
     {\sum_{S\ne\emptyset}a_S^r(x)^2+\epsilon}.
\label{eq:app-topk-mass}
\end{equation}
We measure scalar and vector relative errors by
$\mathrm{Err}_K^v=|v_x(\mathbf{1}_{[L]})-\widehat v_{x,K}|/(|v_x(\mathbf{1}_{[L]})|+\epsilon)$
and
$\mathrm{Err}_K^h=\|h_x(\mathbf{1}_{[L]})-\widehat h_{x,K}\|_2/
(\|h_x(\mathbf{1}_{[L]})\|_2+\epsilon)$. For vector reconstructions, we additionally
report cosine similarity $\mathrm{Cos}_K$ and Top-1 agreement with the all-open model. The
latter is prediction agreement, not ground-truth accuracy. We also plot the
ratio between the reconstructed and full-model predicted-class margins, but do
not use its mean as a primary statistic because the ratio is unstable when the
full-model margin is close to zero.

For thresholds $\tau\in\{0.10,0.05\}$ and
$\eta\in\{0.90,0.95,0.99\}$, we summarize the vector ranking by
\begin{equation}
K_\tau^{\mathrm{err}}(x)
=\min\{K:\mathrm{Err}_K^h(x)\le\tau\},
\qquad
K_\eta^{\mathrm{mass}}(x)
=\min\{K:\Gamma_K^h(x)\ge\eta\}.
\label{eq:app-topk-effective-k}
\end{equation}
Because reconstruction error need not decrease monotonically with $K$,
$K_\tau^{\mathrm{err}}$ is a first-hit statistic. In contrast,
$\Gamma_K^h$ is monotone by construction.

\paragraph{Baselines and Reporting Grids.}
We compare magnitude Top-$K$ with three vector-readout controls: uniform
random-$K$ selection, order-matched random selection that preserves the number
of selected subsets at each order, and low-order truncation retaining every
subset with $|S|\le r$. All reconstructions include
$\Delta_{\emptyset}h_x$ and sum the selected signed coefficients. Random-$K$
results are averaged over five seeds. The order-matched result uses the stored
sampling protocol: one realization for ResNet-18 and five seeds for ResNet-34.
ResNet-18 is reported at
$K\in\{1,2,4,8,16,32,64,128,255\}$ and at retained fractions
$\{1,5,10,25,50,100\}\%$. ResNet-34 uses the analogous absolute grid together
with retained fractions
$\{0.1,0.5,1,5,10,25,50,100\}\%$, with $K=\lceil pM\rceil$.

\begin{figure*}[t]
\centering
\includegraphics[width=\textwidth,trim=7.9bp 23bp 7.9bp 15.5bp,clip]
{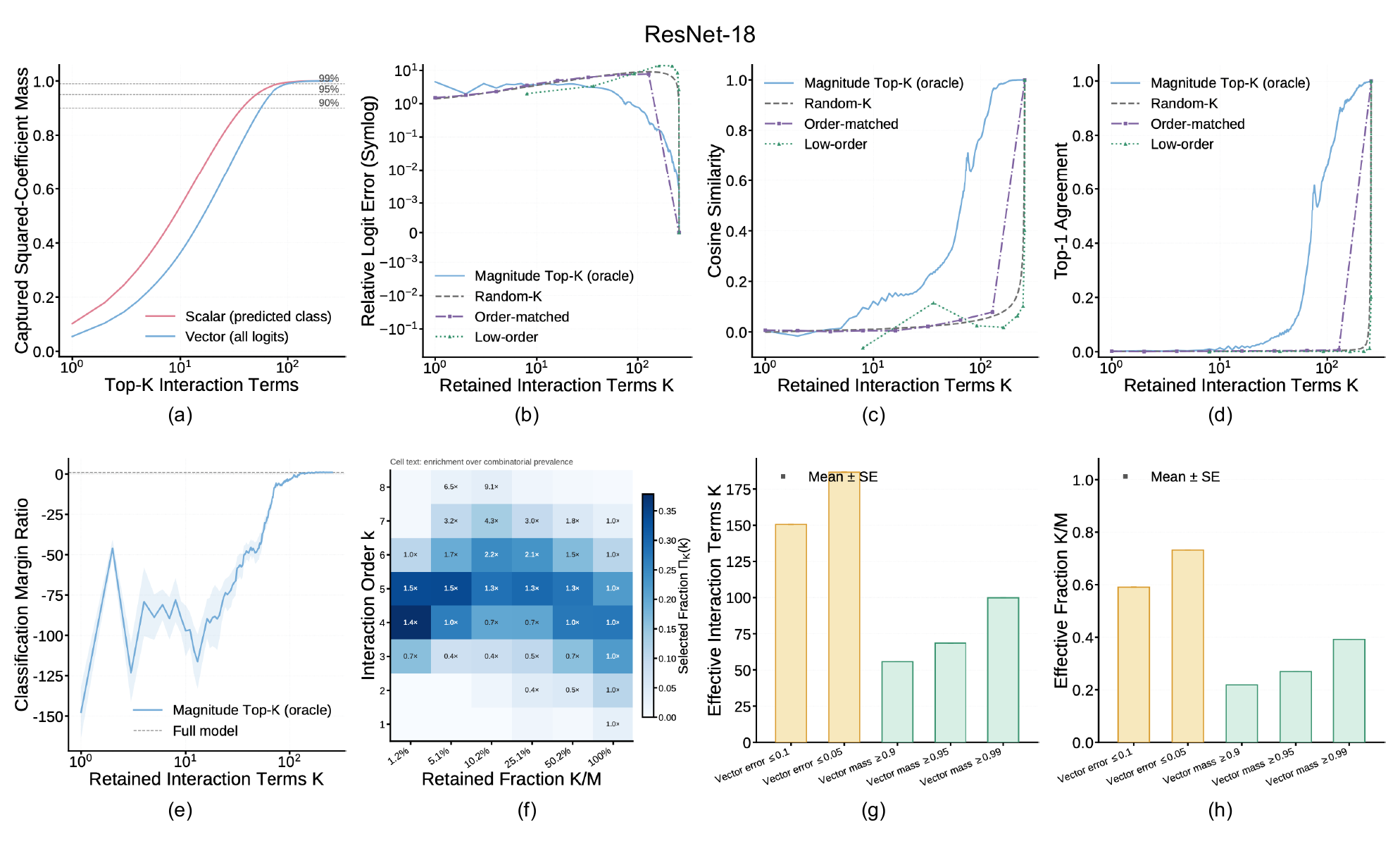}
\caption{Per-input magnitude Top-$K$ reconstruction for ResNet-18.
(a) Retained squared coefficient mass for scalar and vector rankings.
(b--d) Vector relative error, cosine similarity, and agreement with the
all-open model under magnitude, random, order-matched, and low-order selection.
(e) Predicted-class margin ratio. (f) Order composition of the magnitude
Top-$K$ sets; color denotes the selected fraction within each order and text
denotes enrichment over combinatorial prevalence. (g,h) Mean per-image
first-hit budgets and retained fractions for the stated error and mass
thresholds.}
\label{fig:app-exp3-r18}
\end{figure*}

\begin{figure*}[t]
\centering
\includegraphics[width=\textwidth,trim=7.9bp 23bp 7.9bp 15.5bp,clip]
{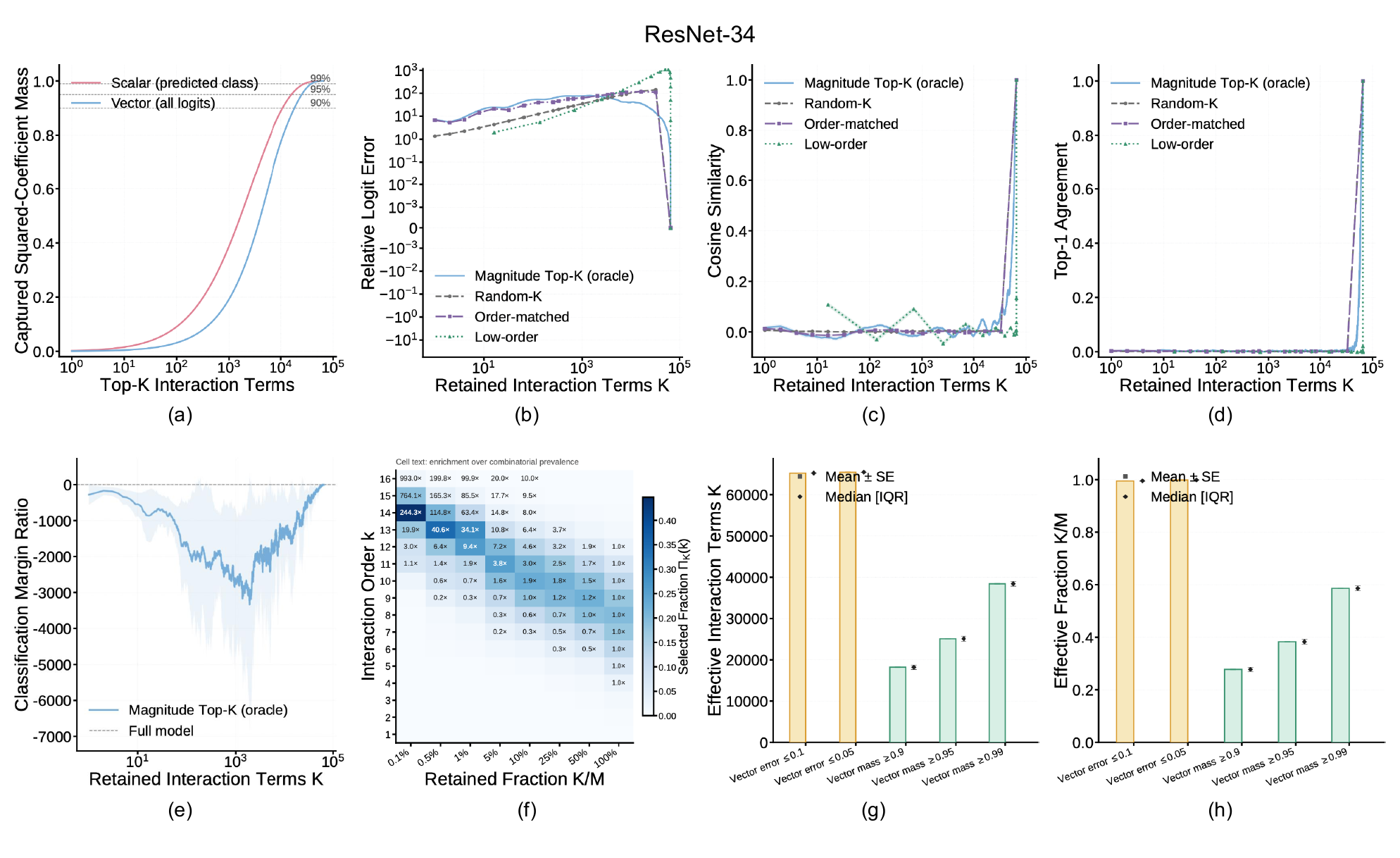}
\caption{Per-input magnitude Top-$K$ interaction reconstruction for
ResNet-34. Results are averaged over $N=1{,}000$ unlabeled ImageNet-1K test
images. The interaction dictionary contains $M=2^{16}-1=65{,}535$ non-empty
terms; the shortcut-only baseline is always included and is not counted in
$K$. (a) Squared coefficient mass captured by independently ranked scalar
predicted-class and vector-logit Top-$K$ sets. (b--d) Relative vector-logit
error, cosine similarity, and agreement with the all-open model prediction
under magnitude Top-$K$, random-$K$, order-matched random, and low-order
selection. Agreement is measured against the full-model prediction rather
than the ground-truth label. (e) Ratio of the reconstructed predicted-class
margin to the full-model margin. (f) Order composition of the vector
magnitude Top-$K$ sets across retained fractions; color denotes the selected
fraction within each interaction order, and cell text gives its enrichment
over combinatorial prevalence. (g,h) Effective interaction counts and their
fractions of the full dictionary for vector-error and coefficient-mass
thresholds; markers summarize the mean and median distributions across
inputs. Coefficient mass is concentrated, but faithful reconstruction is much
less sparse: retaining $50\%$ of the terms captures $97.90\%$ of the vector
mass, yet yields relative error $13.87$ and only $0.60\%$ prediction
agreement.}
\label{fig:app-exp3-r34}
\end{figure*}

\begin{table*}[t]
\centering
\caption{Magnitude Top-$K$ concentration and reconstruction.
Entries are dataset means; uncertainty terms are omitted for readability.
Scalar and vector supports are ranked separately. Top-1 agreement is measured
against the corresponding all-open model prediction.}
\label{tab:app-exp3-topk}
\scriptsize
\setlength{\tabcolsep}{4.0pt}
\begin{tabular}{@{}lrrrrrrr@{}}
\toprule[1.5pt]
Retained & $K$ & $\Gamma_K^v$ & $\mathrm{Err}_K^v$
& $\Gamma_K^h$ & $\mathrm{Err}_K^h$ & $\mathrm{Cos}_K$ & Agree \\
\midrule
\multicolumn{8}{l}{\textit{ResNet-18 ($N=10{,}000$, $M=255$)}}\\
1\%   & 3   & 0.2445 & 0.9934 & 0.1450 & 4.0016 &  0.0019 & 0.0019 \\
5\%   & 13  & 0.6140 & 0.8625 & 0.4345 & 4.0782 &  0.1214 & 0.0063 \\
10\%  & 26  & 0.8161 & 0.6754 & 0.6548 & 3.6881 &  0.1671 & 0.0339 \\
25\%  & 64  & 0.9751 & 0.2644 & 0.9364 & 1.7686 &  0.4896 & 0.3048 \\
50\%  & 128 & 0.9988 & 0.0575 & 0.9971 & 0.2850 &  0.9554 & 0.8766 \\
100\% & 255 & 1.0000 & $1.61\!\times\!10^{-7}$ & 1.0000
      & $1.04\!\times\!10^{-6}$ & 1.0000 & 1.0000 \\
\midrule
\multicolumn{8}{l}{\textit{ResNet-34 ($N=1{,}000$, $M=65{,}535$)}}\\
0.1\% & 66     & 0.0672 & 4.8806 & 0.0223 & 44.1222 &  0.0099 & 0.0010 \\
0.5\% & 328    & 0.2023 & 7.3811 & 0.0817 & 69.5117 &  0.0006 & 0.0000 \\
1\%   & 656    & 0.3077 & 8.0174 & 0.1400 & 76.1582 & -0.0164 & 0.0000 \\
5\%   & 3,277  & 0.6597 & 8.1968 & 0.4409 & 65.1384 & -0.0097 & 0.0020 \\
10\%  & 6,554  & 0.8125 & 6.6638 & 0.6545 & 42.4853 &  0.0115 & 0.0000 \\
25\%  & 16,384 & 0.9528 & 4.4102 & 0.8810 & 35.5789 &  0.0288 & 0.0020 \\
50\%  & 32,768 & 0.9947 & 1.7966 & 0.9790 & 13.8741 &  0.0788 & 0.0060 \\
100\% & 65,535 & 1.0000 & $<10^{-12}$ & 1.0000 & $<10^{-12}$
       & 1.0000 & 1.0000 \\
\bottomrule[1.5pt]
\end{tabular}
\end{table*}

\begin{table*}[t]
\centering
\caption{Representative vector-readout baseline comparisons.
Magnitude selection substantially improves numerical reconstruction over
random and order-matched selection, but this does not necessarily translate
into prediction agreement. Values are dataset means.}
\label{tab:app-exp3-baselines}
\scriptsize
\setlength{\tabcolsep}{4.2pt}
\begin{tabular}{@{}llrrrrrrr@{}}
\toprule[1.5pt]
Model & Retained & $K$ & \multicolumn{3}{c}{Relative error}
& \multicolumn{3}{c}{Top-1 agreement}\\
\cmidrule(lr){4-6}\cmidrule(l){7-9}
& & & Magnitude & Random & Order-matched
& Magnitude & Random & Order-matched\\
\midrule
ResNet-18 & 25\% & 64     & 1.7686  & 7.9048   & 7.3574   & 0.3048 & 0.0028 & 0.0044\\
          & 50\% & 128    & 0.2850  & 9.0794   & 7.7646   & 0.8766 & 0.0042 & 0.0068\\
ResNet-34 & 25\% & 16,384 & 35.5789 & 123.0253 & 117.9390 & 0.0020 & 0.0018 & 0.0006\\
          & 50\% & 32,768 & 13.8741 & 141.9837 & 111.8046 & 0.0060 & 0.0016 & 0.0024\\
\bottomrule[1.5pt]
\end{tabular}
\end{table*}

\begin{table}[t]
\centering
\caption{Vector effective budgets. Mean per-image first-hit $K$ and
its fraction of the $M$ non-empty interactions. Error thresholds use the first
crossing of a potentially non-monotone curve; mass thresholds are monotone.}
\label{tab:app-exp3-effective-k}
\small
\setlength{\tabcolsep}{4.0pt}
\begin{tabular}{@{}llrr@{}}
\toprule[1.5pt]
Model & Criterion & Mean $K$ & $K/M$ \\
\midrule
ResNet-18 & $\mathrm{Err}_K^h\le0.10$ & 150.60 & 59.059\%\\
          & $\mathrm{Err}_K^h\le0.05$ & 186.64 & 73.191\%\\
          & $\Gamma_K^h\ge0.90$       &  55.77 & 21.870\%\\
          & $\Gamma_K^h\ge0.95$       &  68.59 & 26.898\%\\
          & $\Gamma_K^h\ge0.99$       &  99.85 & 39.155\%\\
\midrule
ResNet-34 & $\mathrm{Err}_K^h\le0.10$ & 65,203.74 & 99.495\%\\
          & $\mathrm{Err}_K^h\le0.05$ & 65,425.85 & 99.833\%\\
          & $\Gamma_K^h\ge0.90$       & 18,188.62 & 27.754\%\\
          & $\Gamma_K^h\ge0.95$       & 25,097.15 & 38.296\%\\
          & $\Gamma_K^h\ge0.99$       & 38,389.40 & 58.578\%\\
\bottomrule[1.5pt]
\end{tabular}
\end{table}

\subsection{Results and Analysis}
\label{app:topk-sparse-reconstruction-results}

\paragraph{Squared coefficient mass is concentrated.}
Figures~\ref{fig:app-exp3-r18}a and~\ref{fig:app-exp3-r34}a, together with
Table~\ref{tab:app-exp3-topk}, show substantial mass concentration in both
models. For ResNet-18, retaining $25\%$ of the dictionary captures $97.51\%$
of scalar mass and $93.64\%$ of vector mass; at $50\%$, these values rise to
$99.88\%$ and $99.71\%$. ResNet-34 is less concentrated but remains
compressible in coefficient mass: $25\%$ of its terms capture $95.28\%$ of
scalar mass and $88.10\%$ of vector mass, while $50\%$ captures $99.47\%$ and
$97.90\%$. The effective budgets in
Table~\ref{tab:app-exp3-effective-k} give the same picture: $90\%$ of vector
mass requires, on average, $21.87\%$ of the ResNet-18 dictionary and $27.75\%$
of the ResNet-34 dictionary.

\paragraph{Mass concentration does not imply faithful reconstruction.}
The reconstruction metrics separate sharply from the mass curves. In
ResNet-18, retaining $K=64$ terms ($25\%$) captures $93.64\%$ of vector mass,
but yields relative error $1.7686$, cosine similarity $0.4896$, and only
$30.48\%$ agreement with the full-model prediction. At $K=128$ ($50\%$), the
error falls to $0.2850$ and agreement rises to $87.66\%$. The mean per-image
first-hit budgets for $10\%$ and $5\%$ vector error are $150.60$ and $186.64$
terms, or $59.06\%$ and $73.19\%$ of the dictionary
(Table~\ref{tab:app-exp3-effective-k}). Thus, ResNet-18 has a compressible
coefficient distribution, but accurate logit and decision reconstruction
retains a substantial tail.

This separation is much stronger in ResNet-34. Although $K=16{,}384$ ($25\%$)
captures $88.10\%$ of vector mass, its relative error is $35.5789$ and its
Top-1 agreement is $0.20\%$. Even at $K=32{,}768$ ($50\%$), vector mass is
$97.90\%$, but error remains $13.8741$ and agreement is only $0.60\%$
(Table~\ref{tab:app-exp3-topk}). On average, $99.495\%$ of all non-empty terms
are needed to reach $10\%$ vector error. For this checkpoint, squared
coefficient mass is therefore compressible, whereas magnitude-truncated logit
and prediction reconstruction is effectively dense.

\paragraph{Term identity matters, but order alone is insufficient.}
Table~\ref{tab:app-exp3-baselines} and panels (b)--(d) of
Figures~\ref{fig:app-exp3-r18} and~\ref{fig:app-exp3-r34} show that magnitude
selection usually yields much lower numerical error than random or
order-matched selection. For ResNet-18 at $K=64$, it also raises agreement from
below $0.5\%$ for both controls to $30.48\%$. In ResNet-34, magnitude selection
reduces error at the $25\%$ and $50\%$ budgets, but all three methods still
have near-zero agreement. Hence, selecting the identities of large terms is
more informative than selecting their orders alone, but it is not sufficient
for sparse decision reconstruction in the deeper model.

Low-order truncation fails for a complementary reason. In ResNet-18, retaining
all orders through $r=7$ keeps $254$ of $255$ terms yet gives error $2.6881$
and agreement $20.12\%$; including the eighth-order term restores exact
reconstruction. In ResNet-34, retaining all orders through $r=15$ keeps
$65{,}534$ of $65{,}535$ terms but still gives error $6.8803$ and agreement
$2.10\%$. The order-composition heatmaps in panels (f) further show that the
magnitude-selected sets are enriched in mid- and high-order terms, consistent
with the spectra in Appendix~\ref{app:interaction-order-spectrum}. These
results reflect cancellation within the anchored M\"obius expansion and
should not be interpreted as an executable high-order path.

\paragraph{Conclusion.}
The robust sparsity result is concentration of squared M\"obius coefficient
mass. ResNet-18 additionally permits moderately compressed prediction
reconstruction, whereas ResNet-34 requires nearly the full interaction
dictionary for small logit error under magnitude truncation. The experiment
therefore supports input-specific contribution concentration, not sparse
execution or uniformly sparse decision reconstruction. The next experiment
examines whether the identities of the high-magnitude interactions vary
systematically across inputs.


\clearpage
\section{Input-Dependent Residual Expert Sets}
\label{app:input-dependent-expert-sets}

This section tests whether the high-magnitude residual interactions identified
for one input are also selected for other inputs. Across both ResNet-18 and
ResNet-34, vector-ranked interaction sets contain a large shared global core,
but same-class images consistently exhibit greater overlap than
different-class images. Decision-conditioned scalar sets show substantially
stronger class alignment, whereas confidence has little effect on set
similarity. The resulting picture is therefore not one of disjoint routes,
but of class-structured, input-dependent variation around a broadly reused
interaction set.

\subsection{Motivation}

Coefficient concentration alone does not establish input-dependent routing. A
fixed set of large interactions could explain every input and still yield a
compressible representation. A routing-like interpretation additionally
requires the identities of the selected interactions to vary across inputs in
a structured manner. We therefore ask whether input-specific magnitude
Top-$K$ sets differ across images, whether same-class images share more terms
than different-class images, and whether any class gap survives random-set
and label-shuffle controls. Because the underlying ResNets remain densely
executed, this is a test of input-dependent functional contribution rather
than conditional execution.

\subsection{Experimental Setup and Evaluation Metrics}

\paragraph{Models, Data, and Interaction Rankings.}
We use torchvision ResNet-18 and ResNet-34 checkpoints with
IMAGENET1K\_V1 weights. Both models are evaluated on the same class-balanced
ImageNet-1K validation subset containing ten images from each of the
$1{,}000$ classes ($N=10{,}000$). The all-open Top-1 accuracies are $69.94\%$
for ResNet-18 and $73.55\%$ for ResNet-34, providing a check that images,
labels, and pretrained weights are aligned. As in the preceding experiments,
we enumerate all binary residual masks and compute the exact M\"obius
coefficients. ResNet-18 has $L=8$ residual branches and $M=255$ non-empty
interactions; ResNet-34 has $L=16$ and $M=65{,}535$.

For each image $x$, vector interactions are ranked by
$a_S^h(x)=\|\Delta_Sh_x\|_2$, while scalar interactions are ranked by
$a_S^v(x)=|\Delta_Sv_x|$. Here $v_x$ is the logit of the class predicted by
the all-open model, with that class fixed across masks. The corresponding
per-input sets are
\begin{equation}
\mathcal A_K^r(x)
=
\operatorname*{TopK}_{S\ne\emptyset} a_S^r(x),
\qquad r\in\{h,v\}.
\label{eq:app-exp4-topk-set}
\end{equation}
The vector ranking is label-free and uses a common logit-space norm. The
scalar ranking is decision-conditioned because its output coordinate depends
on the image's full-model prediction; we therefore treat scalar results as a
complementary diagnostic rather than an independent label-free routing test.

ResNet-18 uses
$K\in\{3,8,13,26,32,64\}$, corresponding approximately to retained
fractions $\{1,3,5,10,13,25\}\%$. ResNet-34 uses
$K\in\{32,656,1967,3277,6554,8520,16384\}$, corresponding to
$\{0.05,1,3,5,10,13,25\}\%$. The fixed $K=32$ setting is used for the
nearest-neighbor diagnostic in both models. This is a comparison at the same
absolute budget, not at the same retained fraction: $K=32$ corresponds to
$12.55\%$ of the ResNet-18 dictionary but only $0.0488\%$ of the ResNet-34
dictionary. For the cross-model global-set histograms, we instead match the
retained fraction at approximately $13\%$, using $K=32$ for ResNet-18 and
$K=8{,}520$ for ResNet-34.

\paragraph{Binary and Weighted Signatures.}
The binary signature records the identities of the selected terms,
\begin{equation}
z_{K,S}^r(x)
=
\mathbf 1\!\left[S\in\mathcal A_K^r(x)\right].
\label{eq:app-exp4-binary-signature}
\end{equation}
For the vector ranking, we additionally define the normalized squared
coefficient mass
$p_S^h(x)=a_S^h(x)^2/\sum_{T\ne\emptyset}a_T^h(x)^2$ and the truncated
weighted signature
$w_{K,S}^h(x)=p_S^h(x)\mathbf 1[S\in\mathcal A_K^h(x)]$. Thus, the weighted
signature preserves the relative magnitude of selected vector interactions
without renormalizing away the omitted mass.

\paragraph{Pairwise Overlap and Class Structure.}
For binary sets, pair similarity is the Jaccard index
\begin{equation}
J_K^r(x_i,x_j)
=
\frac{|\mathcal A_K^r(x_i)\cap\mathcal A_K^r(x_j)|}
{|\mathcal A_K^r(x_i)\cup\mathcal A_K^r(x_j)|}.
\label{eq:app-exp4-jaccard}
\end{equation}
For vector-weighted signatures, we use
$J_K^w(x_i,x_j)=\sum_S\min\{w_{i,S},w_{j,S}\}/
(\sum_S\max\{w_{i,S},w_{j,S}\}+\epsilon)$. The primary class statistics
are
\begin{equation}
\Delta_{\mathrm{class}}^r(K)
=
\mathbb E[J_K^r\mid y_i=y_j]
-
\mathbb E[J_K^r\mid y_i\ne y_j],
\qquad
\mathrm{Lift}_{\mathrm{class}}^r(K)
=
\frac{\mathbb E[J_K^r\mid y_i=y_j]}
{\mathbb E[J_K^r\mid y_i\ne y_j]+\epsilon}.
\label{eq:app-exp4-class-gap}
\end{equation}
All $45{,}000$ within-class pairs are used, and an equal number of
different-class pairs are sampled. The same pair indices are reused across
$K$ whenever applicable.

\paragraph{Controls and Auxiliary Diagnostics.}
A uniform random-$K$ baseline, estimated from $2{,}000$ simulated pairs at
each $K$, measures overlap expected solely from set size.
To test whether vector sets reduce to a fixed global subset, we define
\begin{equation}
\mathcal G_K
=
\operatorname*{TopK}_{S\ne\emptyset}
\mathbb E_x[p_S^h(x)].
\label{eq:app-exp4-global-set}
\end{equation}
and report $J(\mathcal A_K^h(x),\mathcal G_K)$ for each image. The global set
is estimated and evaluated on the same $10{,}000$ images, so this overlap is
an in-sample descriptive statistic. A seeded label-shuffle control recomputes
the same-class and different-class gap after permuting the ground-truth
labels. It provides one fixed null realization at each $K$, not a permutation
test distribution. The source analysis also uses $200$ ordinary pair-level
bootstrap resamples. Because pairs share images, we treat these uncertainties
as descriptive and omit them from the tables below.

We also perform leave-one-out 1-nearest-neighbor classification using binary
Jaccard similarity at $K=32$ and compare its accuracy with the empirical
class-prior reference of $0.10\%$; the exact self-excluded random-neighbor
probability on this balanced sample is $9/9{,}999\approx0.09\%$. Binary
signatures can produce tied nearest neighbors, and the stored implementation
resolves ties by the first maximum. We therefore use 1-NN accuracy as a
diagnostic rather than as a standalone statistical test. Finally, images are
divided into upper and lower
terciles of all-open softmax confidence. We compare the Jaccard similarity of
random within-tercile pairs to test whether low-confidence images use more
idiosyncratic interaction sets.

\begin{figure*}[t]
\centering
\includegraphics[width=\textwidth,trim=7.99bp 16.48bp 7.85bp 11.77bp,clip]
{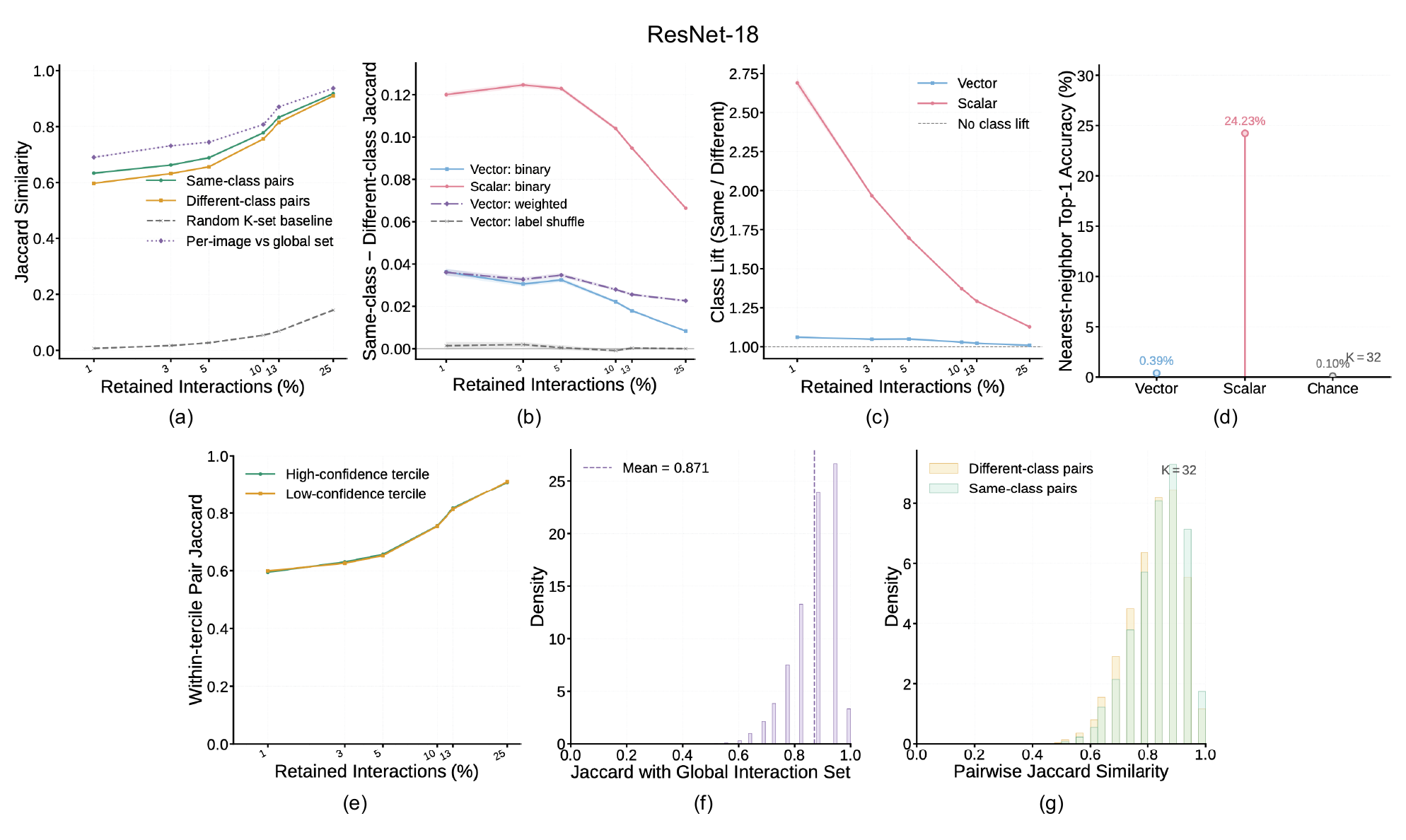}
\caption{Input dependence and class structure of ResNet-18
interaction sets. Results use $N=10{,}000$ ImageNet-1K validation images and
$M=255$ non-empty interactions. (a) Vector-set Jaccard similarity for
same-class and different-class pairs, a uniform random-$K$ baseline, and
per-image overlap with the global vector set. (b) Same-class minus
different-class gaps for vector binary, vector weighted, and
decision-conditioned scalar signatures, together with the vector
label-shuffle control. (c) Same-class/different-class overlap lift for vector
and scalar signatures. (d) Leave-one-out 1-nearest-neighbor class accuracy at
$K=32$, compared with the empirical $0.10\%$ chance level. (e) Similarity of
random pairs drawn within the high- and low-confidence terciles. (f)
Distribution of per-image overlap with the global vector set at $K=32$
($12.55\%$ retention). (g) Same-class and different-class pairwise Jaccard
distributions at $K=32$. The vector sets share a strong global core but retain
a small class-dependent component, while the decision-conditioned scalar
sets show substantially stronger class alignment.}
\label{fig:app-exp4-r18}
\end{figure*}

\begin{figure*}[t]
\centering
\includegraphics[width=\textwidth]{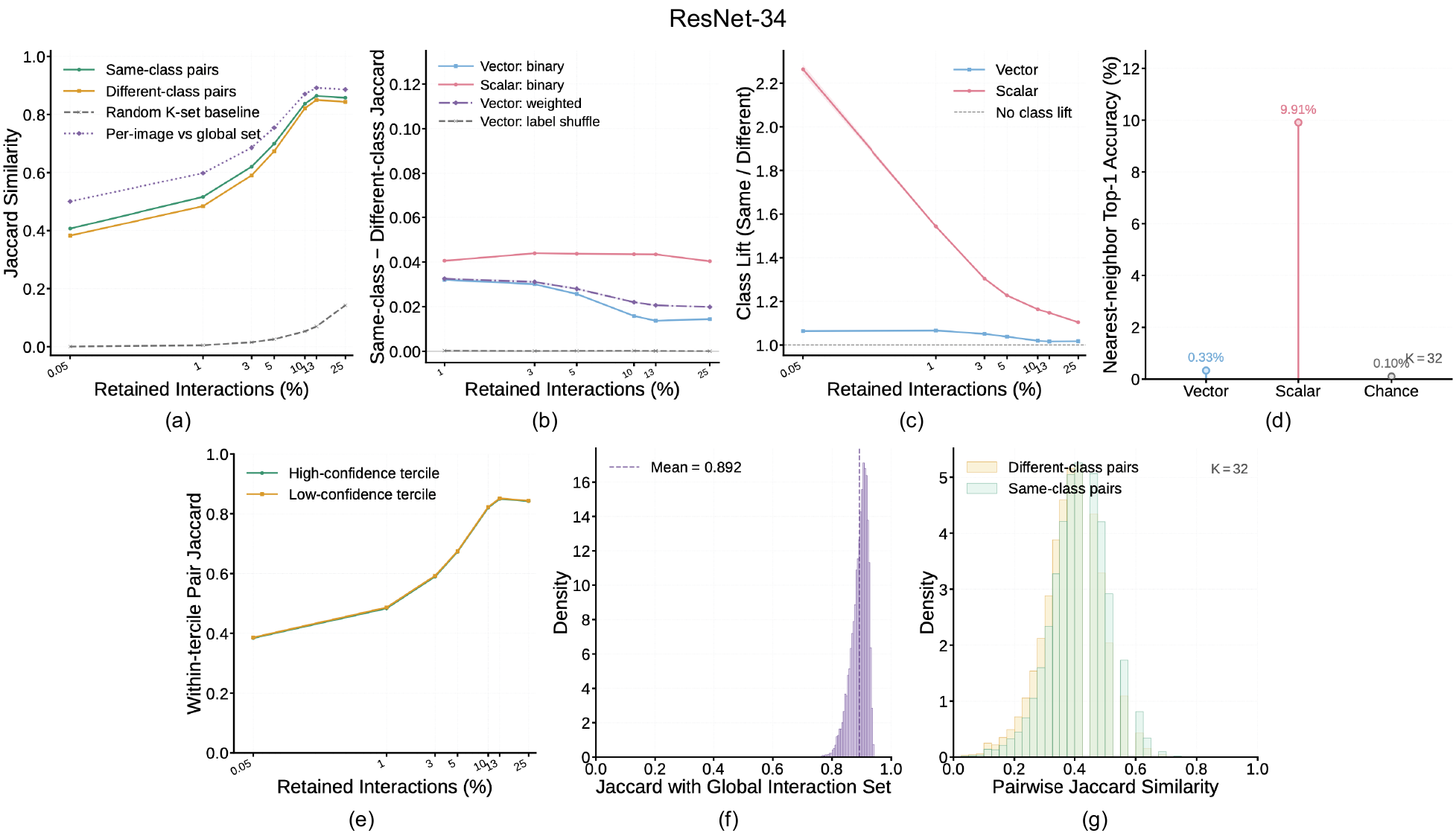}
\caption{Input dependence and class structure of ResNet-34
interaction sets. Results use the same $N=10{,}000$ class-balanced
ImageNet-1K validation images and $M=65{,}535$ non-empty interactions.
Panels (a), (c), and (e) include retained fractions from $0.05\%$ to $25\%$;
the fraction-matched comparison in panel (b) begins at $1\%$. Panel (f)
reports per-image overlap with the global
vector set at $K=8{,}520$ ($13.0\%$ retention), matching the retained fraction
used for ResNet-18 rather than its absolute $K$. Panel (g) shows the
same-class and different-class Jaccard distributions at the fixed $K=32$
used by the nearest-neighbor diagnostic. The positive class gaps replicate
the ResNet-18 trend, but high global overlap again indicates class-structured
variation around a shared interaction core rather than disjoint routes.}
\label{fig:app-exp4-r34}
\end{figure*}

\begin{table*}[t]
\centering
\caption{Vector interaction-set overlap. Entries are dataset means.
$J_{\mathrm{same}}$ and $J_{\mathrm{diff}}$ use ground-truth ImageNet classes;
$J_{\mathrm{rand}}$ is the uniform random-$K$ baseline; and
$J_{\mathrm{global}}$ is the mean per-image overlap with
$\mathcal G_K$. Uncertainty columns are omitted for readability.}
\label{tab:app-exp4-vector-overlap}
\scriptsize
\setlength{\tabcolsep}{3.7pt}
\begin{tabular}{@{}lrrrrrrr@{}}
\toprule[1.5pt]
Retained & $K$ & $J_{\mathrm{same}}^h$ & $J_{\mathrm{diff}}^h$
& $\Delta_{\mathrm{class}}^h$ & Lift$_h$ & $J_{\mathrm{rand}}$
& $J_{\mathrm{global}}$\\
\midrule
\multicolumn{8}{l}{\textit{ResNet-18 ($M=255$)}}\\
1.18\%  & 3  & 0.6334 & 0.5972 & 0.0362 & 1.061 & 0.0069 & 0.6900\\
3.14\%  & 8  & 0.6623 & 0.6318 & 0.0305 & 1.048 & 0.0167 & 0.7312\\
5.10\%  & 13 & 0.6885 & 0.6560 & 0.0325 & 1.050 & 0.0268 & 0.7443\\
10.20\% & 26 & 0.7778 & 0.7557 & 0.0222 & 1.029 & 0.0537 & 0.8072\\
12.55\% & 32 & 0.8329 & 0.8150 & 0.0179 & 1.022 & 0.0683 & 0.8705\\
25.10\% & 64 & 0.9179 & 0.9096 & 0.0083 & 1.009 & 0.1438 & 0.9371\\
\midrule
\multicolumn{8}{l}{\textit{ResNet-34 ($M=65{,}535$)}}\\
0.05\% & 32     & 0.4068 & 0.3825 & 0.0243 & 1.064 & 0.0003 & 0.5004\\
1.00\% & 656    & 0.5162 & 0.4842 & 0.0321 & 1.066 & 0.0050 & 0.5980\\
3.00\% & 1,967  & 0.6205 & 0.5904 & 0.0301 & 1.051 & 0.0152 & 0.6857\\
5.00\% & 3,277  & 0.6991 & 0.6734 & 0.0257 & 1.038 & 0.0257 & 0.7546\\
10.00\% & 6,554  & 0.8371 & 0.8213 & 0.0158 & 1.019 & 0.0527 & 0.8706\\
13.00\% & 8,520  & 0.8639 & 0.8502 & 0.0137 & 1.016 & 0.0695 & 0.8921\\
25.00\% & 16,384 & 0.8576 & 0.8431 & 0.0145 & 1.017 & 0.1429 & 0.8862\\
\bottomrule[1.5pt]
\end{tabular}
\end{table*}

\begin{table*}[t]
\centering
\caption{Class-alignment gaps across signatures and controls.
$\Delta_h$ is binary vector Jaccard, $\Delta_w$ is weighted vector Jaccard,
$\Delta_v$ is binary predicted-class scalar Jaccard, and
$\Delta_{\mathrm{shuf}}$ is the vector gap after label permutation. Scalar
results are decision-conditioned because the selected output coordinate can
differ across images. Values are means.}
\label{tab:app-exp4-class-gaps}
\scriptsize
\setlength{\tabcolsep}{4.1pt}
\begin{tabular}{@{}lrrrrrrr@{}}
\toprule[1.5pt]
Retained & $K$ & $\Delta_h$ & $\Delta_w$
& $J_{\mathrm{same}}^v$ & $J_{\mathrm{diff}}^v$ & $\Delta_v$
& $\Delta_{\mathrm{shuf}}$\\
\midrule
\multicolumn{8}{l}{\textit{ResNet-18}}\\
1.18\%  & 3  & 0.0362 & 0.0361 & 0.1911 & 0.0710 & 0.1200 &  0.0014\\
3.14\%  & 8  & 0.0305 & 0.0327 & 0.2535 & 0.1289 & 0.1246 &  0.0019\\
5.10\%  & 13 & 0.0325 & 0.0347 & 0.2994 & 0.1765 & 0.1229 &  0.0005\\
10.20\% & 26 & 0.0222 & 0.0279 & 0.3844 & 0.2804 & 0.1040 & -0.0009\\
12.55\% & 32 & 0.0179 & 0.0256 & 0.4195 & 0.3248 & 0.0947 &  0.0003\\
25.10\% & 64 & 0.0083 & 0.0226 & 0.5896 & 0.5232 & 0.0664 &  0.0000\\
\midrule
\multicolumn{8}{l}{\textit{ResNet-34}}\\
0.05\% & 32     & 0.0243 & 0.0257 & 0.0369 & 0.0163 & 0.0206 & 0.0016\\
1.00\% & 656    & 0.0321 & 0.0325 & 0.1152 & 0.0747 & 0.0406 & 0.0003\\
3.00\% & 1,967  & 0.0301 & 0.0311 & 0.1883 & 0.1444 & 0.0439 & 0.0002\\
5.00\% & 3,277  & 0.0257 & 0.0280 & 0.2364 & 0.1927 & 0.0437 & 0.0002\\
10.00\% & 6,554  & 0.0158 & 0.0220 & 0.3093 & 0.2658 & 0.0435 & 0.0003\\
13.00\% & 8,520  & 0.0137 & 0.0206 & 0.3383 & 0.2949 & 0.0434 & 0.0002\\
25.00\% & 16,384 & 0.0145 & 0.0199 & 0.4285 & 0.3881 & 0.0403 & 0.0001\\
\bottomrule[1.5pt]
\end{tabular}
\end{table*}

\begin{table}[t]
\centering
\caption{Signature diagnostics. Leave-one-out 1-NN uses binary
Jaccard at the fixed absolute budget $K=32$; this is not a fraction-matched
comparison across architectures, and deterministic tie handling makes the
values diagnostic. The final column is the largest absolute
confidence-tercile gap over the evaluated $K$ grid.}
\label{tab:app-exp4-diagnostics}
\small
\setlength{\tabcolsep}{3.7pt}
\begin{tabular}{@{}lrrrr@{}}
\toprule[1.5pt]
Model & Vector 1-NN & Scalar 1-NN & Chance & Max conf. gap\\
\midrule
ResNet-18 & 0.0039 & 0.2423 & 0.0010 & 0.0045\\
ResNet-34 & 0.0033 & 0.0991 & 0.0010 & 0.0031\\
\bottomrule[1.5pt]
\end{tabular}
\end{table}

\subsection{Results and Analysis}
\label{app:input-dependent-expert-results}

\paragraph{Vector interaction sets combine global reuse with class structure.}
As shown in Figures~\ref{fig:app-exp4-r18}(a)--(c) and
\ref{fig:app-exp4-r34}(a)--(c) and quantified in
Table~\ref{tab:app-exp4-vector-overlap}, same-class overlap exceeds
different-class overlap at every evaluated budget. For ResNet-18, the binary
vector gap is $0.0362$ at $K=3$ and decreases to $0.0083$ at $K=64$ as both
pair types approach saturation. ResNet-34 reproduces the pattern: the gap
peaks at $0.0321$ at $1\%$ retention and remains positive through $25\%$.
Weighted vector gaps are similarly positive
(Table~\ref{tab:app-exp4-class-gaps}), while the shuffled-label gaps remain
within $0.0019$ for ResNet-18 and $0.0016$ for ResNet-34. Thus, the class
signal is small in absolute magnitude but systematic and not explained by set
size alone.

At the same time, the vector sets are close to a dataset-level global set. At
approximately $13\%$ retention, mean per-image global overlap is $0.8705$ for
ResNet-18 and $0.8921$ for ResNet-34, as shown in panel (f) of each figure
and in Table~\ref{tab:app-exp4-vector-overlap}. Same-class and different-class pairs
also have high absolute overlap at these budgets. The label-free vector
analysis therefore supports a shared interaction core with input-specific,
class-aligned variation, rather than routing to nearly disjoint interaction
sets.

\paragraph{Decision-conditioned scalar signatures show stronger class
alignment.}
The scalar gaps in Figure~\ref{fig:app-exp4-r18}(b) and
Figure~\ref{fig:app-exp4-r34}(b) are larger than their vector counterparts over
most of the evaluated range. In ResNet-18, $\Delta_v$ is approximately
$0.12$ for the smallest three budgets and remains $0.0664$ at $25\%$
retention. In ResNet-34, it reaches $0.0439$ at $3\%$ retention and remains
near $0.04$ thereafter. Consistently, the scalar binary signatures attain
$24.23\%$ and $9.91\%$ leave-one-out 1-NN accuracy for ResNet-18 and
ResNet-34, compared with $0.10\%$ chance; vector signatures reach only
$0.39\%$ and $0.33\%$ (Table~\ref{tab:app-exp4-diagnostics}). These results
show that predicted-class interaction identities contain substantial class
information. Because the scalar coordinate is selected separately for each
image, and because the $K=32$ comparison is not fraction matched, however,
these values should not be read as a direct architectural ranking. The scalar
evidence is decision-conditioned and should not by itself be interpreted as a
label-free routing mechanism.

\paragraph{Set similarity is nearly independent of confidence.}
Panel (e) of each figure and Table~\ref{tab:app-exp4-diagnostics} show almost identical
within-tercile overlaps for high- and low-confidence images. Across the full
$K$ grid, the largest absolute difference is $0.0045$ for ResNet-18 and
$0.0031$ for ResNet-34, with no consistent direction in ResNet-18. We
therefore observe no material association between confidence tercile and how
broadly the vector-ranked interaction set is shared. This pattern is
consistent with the separate difficulty analysis: the identities of prominent
interactions vary with class more clearly than the amount of interaction
structure varies with sample difficulty.

\paragraph{Conclusion.}
Both checkpoints exhibit consistent class structure in their post-hoc
interaction sets, but the dominant vector interactions are also strongly
reused across inputs. The most defensible routing-like interpretation is
therefore a limited one: inputs modulate a class-aligned component around a
shared functional core, without an explicit router or sparse execution.
Decision-conditioned scalar signatures expose stronger class information,
while the vector results provide the more conservative label-free evidence.

\clearpage
\section{Sample Difficulty and Residual Interaction Complexity}
\label{app:difficulty-interaction-complexity}

This section examines whether difficult inputs induce a more complex residual
interaction profile. Across the studied ResNet-18 and ResNet-34 checkpoints,
sample difficulty is only weakly associated with interaction breadth,
interaction order, or magnitude-based reconstruction budgets. The hardest
loss quartile exhibits a small increase in effective interaction count, most
visibly for the decision-conditioned scalar readout of ResNet-18, but the
effective order and low-order coefficient mass remain nearly unchanged. The
same conclusion holds when misclassified and correctly classified images are
compared. Thus, input dependence in the identities of selected interactions
does not translate into substantial difficulty-dependent scaling of the
amount or order of functional interaction.

\subsection{Motivation}

The preceding experiment shows that the identities of high-magnitude
interactions vary across inputs and contain class structure. That result does
not establish that harder inputs use a broader or higher-order set of
interactions. A difficulty-adaptive interpretation would predict that hard
images recruit more interaction terms, shift coefficient mass toward higher
orders, or require a larger magnitude-ranked set to reconstruct the logits.
We test these possibilities by relating per-image difficulty to three
complementary aspects of the residual-interaction decomposition: breadth,
interaction order, and magnitude-based compressibility. As throughout the
paper, these quantities describe post-hoc functional contributions rather
than changes in the densely executed computation graph.

\subsection{Experimental Setup and Evaluation Metrics}

\paragraph{Models, Data, and Interaction Strengths.}
We evaluate torchvision ResNet-18 and ResNet-34 checkpoints with
IMAGENET1K\_V1 weights on the same class-balanced ImageNet-1K validation
subset used in Appendix~\ref{app:input-dependent-expert-sets}. The subset
contains ten images from each of $1{,}000$ classes ($N=10{,}000$). The
all-open Top-1 accuracies are $69.94\%$ for ResNet-18 and $73.55\%$ for
ResNet-34. ResNet-18 contains $L=8$ gateable residual branches and
$M=2^8-1=255$ non-empty interaction terms, whereas ResNet-34 contains $L=16$
branches and $M=2^{16}-1=65{,}535$ terms. For each image, exhaustive
residual-mask evaluation and M\"obius inversion yield scalar and vector
magnitudes
$a_S^v(x)=|\Delta_Sv_x|$ and
$a_S^h(x)=\|\Delta_Sh_x\|_2$, respectively. The scalar readout uses the class
predicted by the all-open model, fixed across all masks for that image.

\paragraph{Difficulty Measures.}
Let $h_x^{\mathrm{full}}$ be the all-open logits, $y$ the ground-truth class,
and $\hat y=\arg\max_c h_{x,c}^{\mathrm{full}}$. We use three label-aware
difficulty measures: negative log-likelihood
$d_{\mathrm{NLL}}=-\log\operatorname{softmax}(h_x^{\mathrm{full}})_y$,
the binary error indicator $d_{\mathrm{wrong}}=\mathbf 1[\hat y\ne y]$, and
the negative true-class logit margin
$d_{\mathrm{true}}=-(h_{x,y}^{\mathrm{full}}-
\max_{c\ne y}h_{x,c}^{\mathrm{full}})$. We additionally use two label-free
measures: negative confidence
$d_{\mathrm{conf}}=-\max_c\operatorname{softmax}(h_x^{\mathrm{full}})_c$
and the negative predicted-class logit margin
$d_{\mathrm{pred}}=-(h_{x,\hat y}^{\mathrm{full}}-
\max_{c\ne\hat y}h_{x,c}^{\mathrm{full}})$. All five scores are oriented so
that larger values indicate harder inputs.

\paragraph{Breadth and Interaction Order.}
For either readout $r\in\{v,h\}$, define
$q_S^r=a_S^r/\sum_Ta_T^r$ and
$p_S^r=(a_S^r)^2/\sum_T(a_T^r)^2$, where all sums are over non-empty
interactions. We measure breadth using the participation ratio and entropy
effective count,
\begin{equation}
N_{\mathrm{eff}}^r(x)
=\frac{1}{\sum_S(q_S^r(x))^2},
\qquad
N_{\mathrm{ent}}^r(x)
=\exp\!\left(-\sum_Sp_S^r(x)\log p_S^r(x)\right).
\label{eq:app-exp5-breadth}
\end{equation}
We summarize interaction order using the magnitude-weighted and
squared-magnitude-weighted averages
\begin{equation}
\bar k^r(x)=\sum_S|S|q_S^r(x),
\qquad
\kappa^r(x)=\sum_S|S|p_S^r(x),
\label{eq:app-exp5-order}
\end{equation}
and the low-order squared coefficient mass
$C_{\le3}^r(x)=\sum_{|S|\le3}p_S^r(x)$. The complementary high-order mass is
$T_{>3}^r(x)=1-C_{\le3}^r(x)$, so only one of the two is needed in the
tables.

\paragraph{Magnitude-based Reconstruction Budgets.}
For the vector readout, we also use the cumulative magnitude Top-$K$
reconstructions from Appendix~\ref{app:topk-sparse-reconstruction}. We define
$K_{\tau}^{\mathrm{err}}(x)$ as the first $K$ for which the relative logit
error is at most $\tau\in\{0.10,0.05\}$, and
$K_{\eta}^{\mathrm{mass}}(x)$ as the first $K$ that captures at least
$\eta\in\{0.90,0.95,0.99\}$ of the squared coefficient mass. These are
post-hoc magnitude-ranking budgets, not optimal subset sizes.

\paragraph{Evaluation Protocol.}
We recompute tie-correct Spearman correlations directly from the stored
per-image quantities using average ranks for tied values; these numbers match
the result-visualization figures. Because binary error is more naturally interpreted
as a group comparison, we additionally report the mean difference between
misclassified and correctly classified images. For NLL, the easiest and
hardest groups are the bottom and top quartiles, each containing $2{,}500$
images, and the hard--easy gap is the difference of their means. Finally, we
fit the class-fixed-effect model
$c(x)=\beta_0+\beta_1d_{\mathrm{NLL}}(x)+\gamma_{y(x)}+\varepsilon_x$
by demeaning difficulty and complexity within each ground-truth class. The
reported $\beta_1$ values use the original units and are therefore intended
as within-metric diagnostics rather than cross-metric effect sizes. All table
entries are point estimates; standard deviations and standard errors are
omitted. Because difficulty and complexity are measured on the same model
outputs, these analyses are observational and do not identify a causal effect
of difficulty.

\begin{figure*}[t]
\centering
\includegraphics[width=\textwidth]{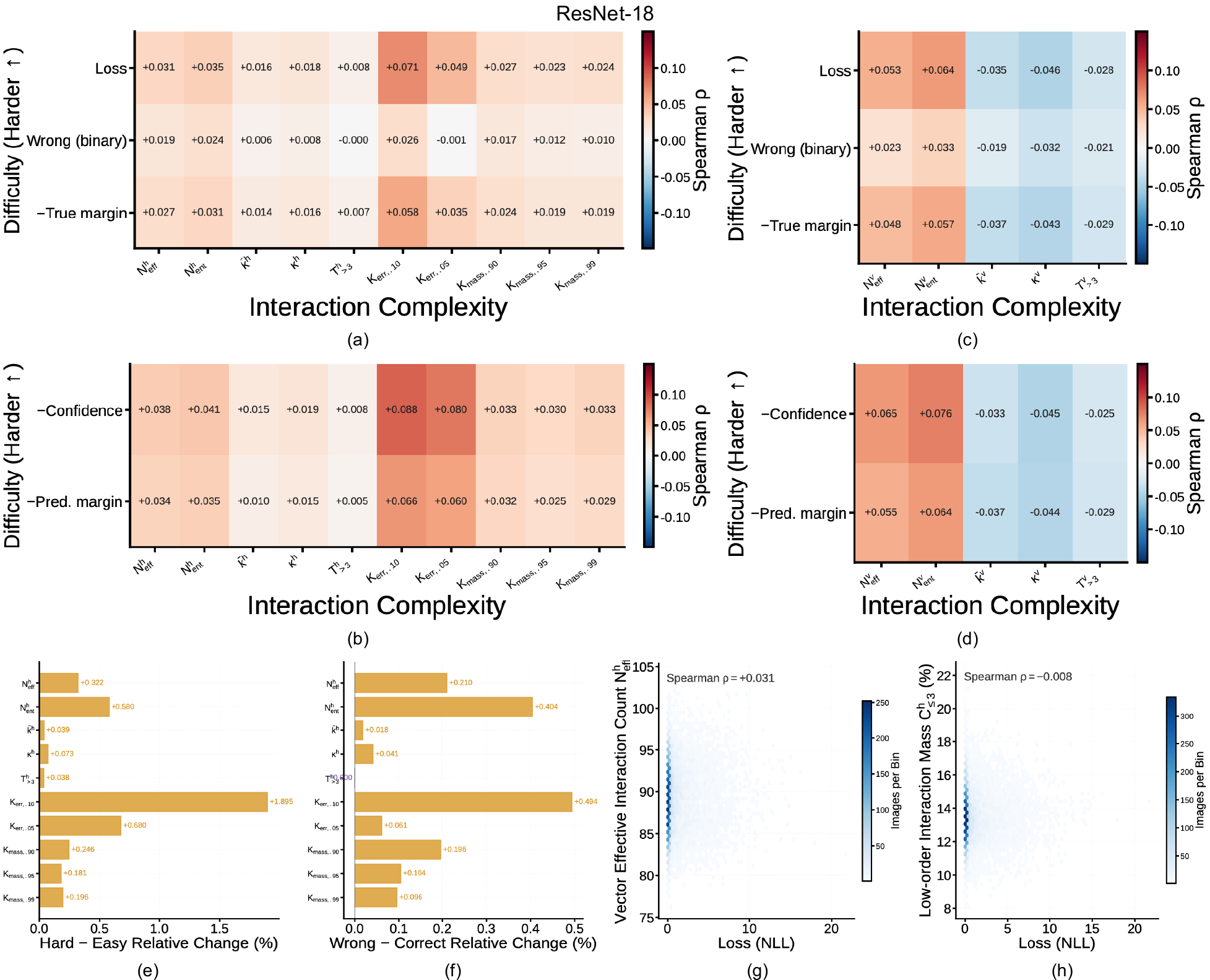}
\caption{Sample difficulty and residual interaction complexity for
ResNet-18. Results use $N=10{,}000$ ImageNet-1K validation images and
$M=255$ non-empty interactions. (a) Tie-correct Spearman correlations between
label-aware difficulty measures and vector-readout complexity. (b) The same
analysis for negative confidence and negative predicted-class margin, both
oriented so that larger values denote harder inputs. (c,d) Corresponding
results for the decision-conditioned scalar readout. (e) Relative hard--easy
change in the vector metrics, $100(\mu_{\mathrm{hard}}-\mu_{\mathrm{easy}})/
|\mu_{\mathrm{easy}}|$, between the top and bottom NLL quartiles. (f) Relative
change in the vector metrics between misclassified and correctly classified
images. Panels (e,f) display $T_{>3}^h$, whereas the tables report its
equivalent complement $C_{\le3}^h=1-T_{>3}^h$ to avoid redundant columns.
(g,h) NLL
versus the vector effective interaction count $N_{\mathrm{eff}}^h$ and
low-order mass $C_{\le3}^h$ (shown as a percentage), respectively.
Associations are uniformly weak;
the clearest change is a small broadening of the interaction-magnitude distribution,
not a shift toward higher interaction orders.}
\label{fig:app-exp5-r18}
\end{figure*}

\begin{figure*}[t]
\centering
\includegraphics[width=\textwidth]{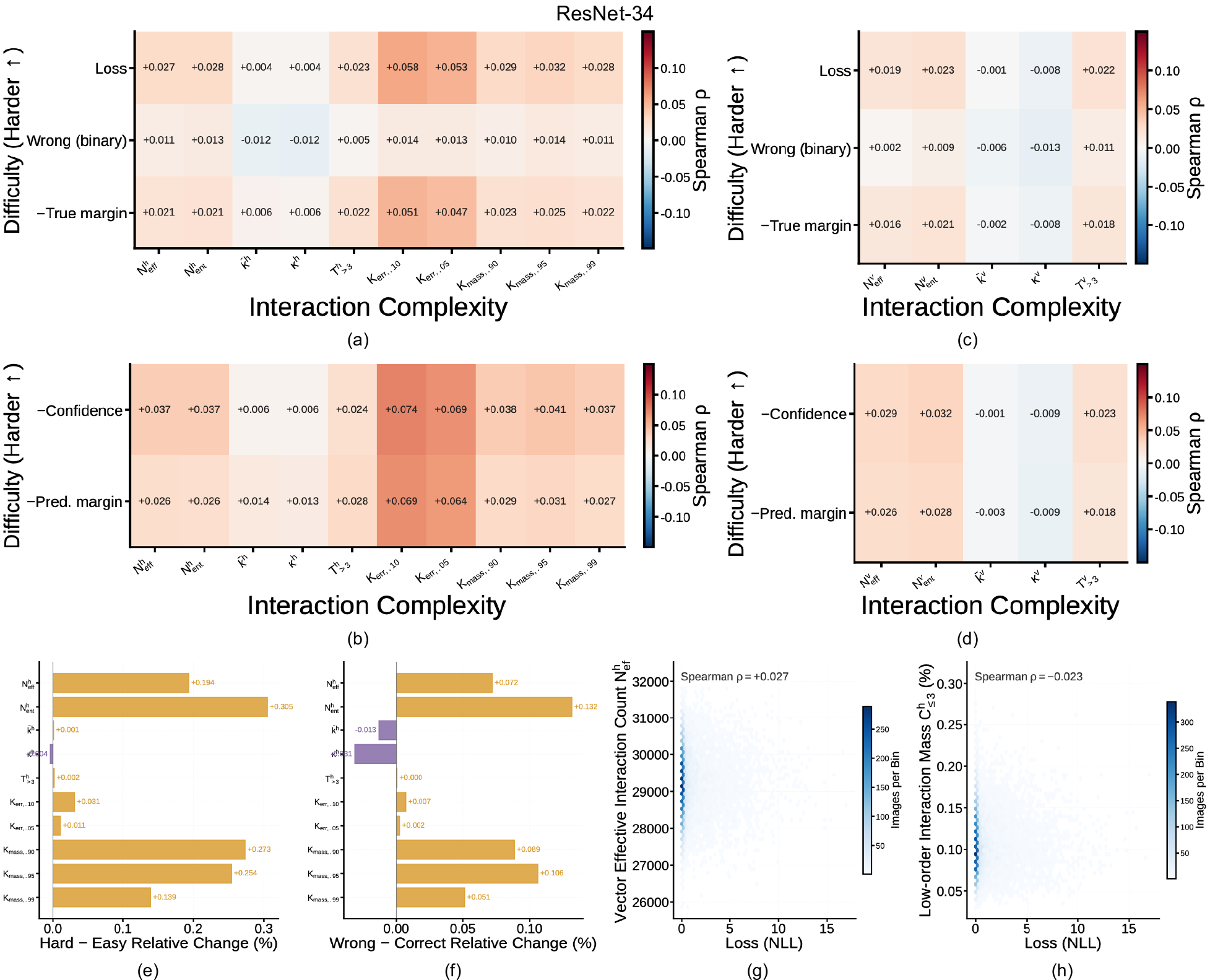}
\caption{Sample difficulty and residual interaction complexity for
ResNet-34. Results are computed over a class-balanced subset of
$N=10{,}000$ ImageNet-1K validation images and all $M=65{,}535$ non-empty
residual interactions. All difficulty scores are oriented so that larger
values indicate harder inputs. (a) Tie-correct Spearman correlations between
label-aware difficulty measures, including NLL, prediction error, and negative
true-class margin, and vector-readout measures of interaction breadth, order,
and Top-$K$ compressibility. (b) Corresponding correlations for the
label-free measures of negative confidence and negative predicted-class
margin. (c,d) Scalar-readout counterparts of (a,b), where the scalar
coordinate is the class predicted by the all-open model. (e) Relative changes
in vector-readout complexity between the highest- and lowest-NLL quartiles.
(f) Relative changes between misclassified and correctly classified images.
These comparisons cover effective interaction counts, interaction order,
high-order coefficient mass, and reconstruction budgets. (g,h) Binned
distributions of NLL against the vector effective interaction count
$N_{\mathrm{eff}}^h$ and low-order coefficient mass $C_{\le3}^h$,
respectively; color intensity denotes the number of images in each bin.
Across all analyses, the correlations and group differences remain small:
harder inputs neither recruit substantially broader interaction supports nor
shift coefficient mass toward higher orders, despite the large interaction
dictionary.}
\label{fig:app-exp5-r34}
\end{figure*}

\begin{table*}[t]
\centering
\caption{Tie-correct Spearman correlations for vector-readout
complexity. Every difficulty score is oriented so that larger values mean
harder inputs. $C_{\le3}^h$ is shown instead of the redundant
$T_{>3}^h=1-C_{\le3}^h$.}
\label{tab:app-exp5-vector-spearman}
\scriptsize
\setlength{\tabcolsep}{2.5pt}
\begin{tabular}{@{}lrrrrrrrrrr@{}}
\toprule[1.5pt]
Difficulty & $N_{\rm eff}^h$ & $N_{\rm ent}^h$ & $\bar k^h$
& $\kappa^h$ & $C_{\le3}^h$ & $K_{.10}^{\rm err}$
& $K_{.05}^{\rm err}$ & $K_{.90}^{\rm mass}$
& $K_{.95}^{\rm mass}$ & $K_{.99}^{\rm mass}$\\
\midrule
\multicolumn{11}{l}{\textit{ResNet-18}}\\
NLL                  & 0.0313 & 0.0352 & 0.0157 & 0.0182 & -0.0080 & 0.0711 & 0.0489 & 0.0267 & 0.0232 & 0.0238\\
Wrong                & 0.0186 & 0.0237 & 0.0063 & 0.0083 &  0.0002 & 0.0264 &-0.0009 & 0.0170 & 0.0116 & 0.0103\\
$-$True margin       & 0.0268 & 0.0306 & 0.0137 & 0.0162 & -0.0073 & 0.0575 & 0.0351 & 0.0237 & 0.0187 & 0.0187\\
$-$Confidence        & 0.0383 & 0.0407 & 0.0154 & 0.0191 & -0.0080 & 0.0883 & 0.0796 & 0.0334 & 0.0298 & 0.0335\\
$-$Predicted margin  & 0.0339 & 0.0355 & 0.0102 & 0.0145 & -0.0052 & 0.0665 & 0.0598 & 0.0315 & 0.0254 & 0.0290\\
\midrule
\multicolumn{11}{l}{\textit{ResNet-34}}\\
NLL                  & 0.0274 & 0.0280 & 0.0042 & 0.0042 & -0.0226 & 0.0580 & 0.0526 & 0.0288 & 0.0319 & 0.0280\\
Wrong                & 0.0110 & 0.0132 &-0.0115 &-0.0122 & -0.0045 & 0.0139 & 0.0133 & 0.0105 & 0.0141 & 0.0106\\
$-$True margin       & 0.0209 & 0.0212 & 0.0063 & 0.0059 & -0.0224 & 0.0513 & 0.0467 & 0.0225 & 0.0253 & 0.0217\\
$-$Confidence        & 0.0367 & 0.0370 & 0.0058 & 0.0057 & -0.0243 & 0.0744 & 0.0685 & 0.0384 & 0.0412 & 0.0372\\
$-$Predicted margin  & 0.0262 & 0.0255 & 0.0141 & 0.0128 & -0.0283 & 0.0693 & 0.0643 & 0.0289 & 0.0308 & 0.0273\\
\bottomrule[1.5pt]
\end{tabular}
\end{table*}

\begin{table*}[t]
\centering
\caption{Tie-correct Spearman correlations for scalar-readout
complexity. The scalar coordinate is the class predicted by the all-open
model and is fixed across masks for each image.}
\label{tab:app-exp5-scalar-spearman}
\small
\setlength{\tabcolsep}{5pt}
\begin{tabular}{@{}lrrrrr@{}}
\toprule[1.5pt]
Difficulty & $N_{\rm eff}^v$ & $N_{\rm ent}^v$ & $\bar k^v$
& $\kappa^v$ & $C_{\le3}^v$\\
\midrule
\multicolumn{6}{l}{\textit{ResNet-18}}\\
NLL                 & 0.0533 & 0.0637 & -0.0355 & -0.0456 &  0.0275\\
Wrong               & 0.0234 & 0.0333 & -0.0186 & -0.0320 &  0.0211\\
$-$True margin      & 0.0478 & 0.0571 & -0.0369 & -0.0435 &  0.0290\\
$-$Confidence       & 0.0654 & 0.0757 & -0.0326 & -0.0447 &  0.0250\\
$-$Predicted margin & 0.0550 & 0.0639 & -0.0365 & -0.0443 &  0.0286\\
\midrule
\multicolumn{6}{l}{\textit{ResNet-34}}\\
NLL                 & 0.0190 & 0.0234 & -0.0011 & -0.0084 & -0.0216\\
Wrong               & 0.0024 & 0.0090 & -0.0059 & -0.0129 & -0.0107\\
$-$True margin      & 0.0165 & 0.0205 & -0.0019 & -0.0080 & -0.0184\\
$-$Confidence       & 0.0290 & 0.0317 & -0.0013 & -0.0092 & -0.0225\\
$-$Predicted margin & 0.0264 & 0.0282 & -0.0025 & -0.0095 & -0.0184\\
\bottomrule[1.5pt]
\end{tabular}
\end{table*}

\begin{table*}[t]
\centering
\caption{Loss-quartile comparison. Easy and hard denote the bottom
and top NLL quartiles. Relative change is
$100(\mathrm{Hard}-\mathrm{Easy})/|\mathrm{Easy}|$. Values are dataset means;
no uncertainty columns are included. For $C_{\le3}$, the relative percentage
can appear comparatively large because the ResNet-34 baseline is close to
zero; the absolute gap should be used to assess its practical size.}
\label{tab:app-exp5-hard-easy}
\small
\setlength{\tabcolsep}{5pt}
\begin{tabular}{@{}lrrrr@{}}
\toprule[1.5pt]
Quantity & Easy & Hard & Hard $-$ Easy & Relative (\%)\\
\midrule
\multicolumn{5}{l}{\textit{ResNet-18}}\\
$N_{\rm eff}^h$        & 89.1949 & 89.4823 &  0.2874 &  0.322\\
$N_{\rm ent}^h$        & 62.7012 & 63.0648 &  0.3635 &  0.580\\
$\kappa^h$             &  4.9283 &  4.9319 &  0.0036 &  0.073\\
$C_{\le3}^h$           &  0.135677 & 0.135351 & -0.000326 & -0.241\\
$K_{.10}^{\rm err}$    & 148.7320 &151.5508 &  2.8188 &  1.895\\
$K_{.95}^{\rm mass}$   & 68.5012 & 68.6252 &  0.1240 &  0.181\\
$N_{\rm eff}^v$        & 63.4144 & 64.8123 &  1.3979 &  2.204\\
$N_{\rm ent}^v$        & 38.5587 & 39.9825 &  1.4238 &  3.693\\
$\kappa^v$             &  5.0352 &  4.9752 & -0.0600 & -1.191\\
$C_{\le3}^v$           &  0.136853 & 0.145008 & 0.008155 &  5.959\\
\midrule
\multicolumn{5}{l}{\textit{ResNet-34}}\\
$N_{\rm eff}^h$        & 29220.21 &29276.79 & 56.59 &  0.194\\
$N_{\rm ent}^h$        & 18305.46 &18361.31 & 55.85 &  0.305\\
$\kappa^h$             & 9.885070 &9.884682 & -0.000388 & -0.004\\
$C_{\le3}^h$           & 0.00112166 &0.00110484 & -0.00001682 & -1.500\\
$K_{.10}^{\rm err}$    & 65185.65 &65205.61 & 19.96 &  0.031\\
$K_{.95}^{\rm mass}$   & 25068.59 &25132.30 & 63.71 &  0.254\\
$N_{\rm eff}^v$        & 18856.96 &18932.03 & 75.07 &  0.398\\
$N_{\rm ent}^v$        & 9165.25 &9238.67 & 73.42 &  0.801\\
$\kappa^v$             & 9.888993 &9.882980 & -0.006012 & -0.061\\
$C_{\le3}^v$           & 0.00116970 &0.00115286 & -0.00001685 & -1.440\\
\bottomrule[1.5pt]
\end{tabular}
\end{table*}

\begin{table*}[t]
\centering
\caption{Correctness comparison. The ResNet-18 subset contains
$6{,}994$ correct and $3{,}006$ incorrect predictions; the ResNet-34 subset
contains $7{,}355$ correct and $2{,}645$ incorrect predictions. Relative
change is $100(\mathrm{Wrong}-\mathrm{Correct})/|\mathrm{Correct}|$.
Absolute gaps are the more informative scale for the small $C_{\le3}$
baseline.}
\label{tab:app-exp5-correctness}
\small
\setlength{\tabcolsep}{5pt}
\begin{tabular}{@{}lrrrr@{}}
\toprule[1.5pt]
Quantity & Correct & Wrong & Wrong $-$ Correct & Relative (\%)\\
\midrule
\multicolumn{5}{l}{\textit{ResNet-18}}\\
$N_{\rm eff}^h$        & 89.2743 & 89.4619 &  0.1876 &  0.210\\
$N_{\rm ent}^h$        & 62.7944 & 63.0482 &  0.2538 &  0.404\\
$\kappa^h$             &  4.9300 &  4.9320 &  0.0020 &  0.041\\
$C_{\le3}^h$           &  0.135411 &0.135413 & $2.22\!\times\!10^{-6}$ &  0.002\\
$K_{.10}^{\rm err}$    & 150.5648 &151.3084 &  0.7436 &  0.494\\
$K_{.95}^{\rm mass}$   & 68.5332 & 68.6045 &  0.0713 &  0.104\\
$N_{\rm eff}^v$        & 64.2505 & 64.6977 &  0.4472 &  0.696\\
$N_{\rm ent}^v$        & 39.3443 & 39.8978 &  0.5535 &  1.407\\
$\kappa^v$             &  5.0081 &  4.9755 & -0.0326 & -0.652\\
$C_{\le3}^v$           &  0.141030 &0.145338 & 0.004309 &  3.055\\
\midrule
\multicolumn{5}{l}{\textit{ResNet-34}}\\
$N_{\rm eff}^h$        & 29247.64 &29268.72 & 21.08 &  0.072\\
$N_{\rm ent}^h$        & 18328.32 &18352.43 & 24.11 &  0.132\\
$\kappa^h$             & 9.889496 &9.886413 & -0.003083 & -0.031\\
$C_{\le3}^h$           & 0.00110605 &0.00110169 & $-4.36\!\times\!10^{-6}$ & -0.394\\
$K_{.10}^{\rm err}$    & 65201.56 &65206.12 &  4.56 &  0.007\\
$K_{.95}^{\rm mass}$   & 25098.68 &25125.33 & 26.65 &  0.106\\
$N_{\rm eff}^v$        & 18919.93 &18933.02 & 13.09 &  0.069\\
$N_{\rm ent}^v$        & 9212.29 &9239.69 & 27.41 &  0.297\\
$\kappa^v$             & 9.889846 &9.882120 & -0.007726 & -0.078\\
$C_{\le3}^v$           & 0.00116570 &0.00115321 & $-1.249\!\times\!10^{-5}$ & -1.071\\
\bottomrule[1.5pt]
\end{tabular}
\end{table*}

\begin{table*}[t]
\centering
\caption{Class-controlled NLL slopes. Each entry is the raw
$\beta_1$ from a within-class fixed-effect regression. Because the quantities
have different units and scales, their signed magnitudes should be interpreted
within each column rather than compared directly across metrics or
architectures.}
\label{tab:app-exp5-class-controlled}
\small
\setlength{\tabcolsep}{5pt}
\begin{tabular}{@{}lrrrrrrr@{}}
\toprule[1.5pt]
Model & $N_{\rm eff}^h$ & $\kappa^h$ & $C_{\le3}^h$
& $K_{.10}^{\rm err}$ & $K_{.95}^{\rm mass}$
& $N_{\rm eff}^v$ & $\kappa^v$\\
\midrule
ResNet-18 & -0.0306 & -0.001003 & $9.652\!\times\!10^{-5}$
& 0.1311 & -0.0086 & 0.0913 & -0.007939\\
ResNet-34 & -5.7765 & -0.000242 & $2.762\!\times\!10^{-7}$
& 0.8280 & -5.1910 & 0.9821 & 0.000174\\
\bottomrule[1.5pt]
\end{tabular}
\end{table*}

\subsection{Results and Analysis}
\label{app:difficulty-interaction-complexity-results}

\paragraph{Difficulty has little rank association with interaction
complexity.}
Figures~\ref{fig:app-exp5-r18}(a)--(d) and
\ref{fig:app-exp5-r34}(a)--(d), together with
Tables~\ref{tab:app-exp5-vector-spearman} and
\ref{tab:app-exp5-scalar-spearman}, show that all associations are weak. For
ResNet-18, the largest absolute vector correlation is $0.0883$, between
negative confidence and $K_{.10}^{\mathrm{err}}$, while the largest scalar
correlation is $0.0757$, between negative confidence and
$N_{\mathrm{ent}}^v$. The corresponding maxima for ResNet-34 are $0.0744$
and $0.0317$. More importantly, the effective-order correlations are close to
zero: across the continuous difficulty measures, $|\rho(\kappa^h)|\le0.0191$
for ResNet-18 and $\le0.0128$ for ResNet-34. Difficulty therefore has little
monotonic association with either breadth or interaction order.

\paragraph{Hard images exhibit only small and model-dependent broadening.}
The loss-quartile comparison in Figure~\ref{fig:app-exp5-r18}(e),
Figure~\ref{fig:app-exp5-r34}(e), and
Table~\ref{tab:app-exp5-hard-easy} reveals a small increase in breadth. In
ResNet-18, $N_{\mathrm{eff}}^h$ increases from $89.1949$ to $89.4823$
($0.322\%$), while $K_{.10}^{\mathrm{err}}$ increases from $148.7320$ to
$151.5508$ ($1.895\%$). The decision-conditioned scalar breadth changes more,
with $N_{\mathrm{eff}}^v$ increasing by $2.204\%$. The same directions appear
in ResNet-34 but are smaller: the relative increases are $0.194\%$,
$0.031\%$, and $0.398\%$, respectively. Moreover, the within-class slopes in
Table~\ref{tab:app-exp5-class-controlled} are mixed: the vector effective
count has a small negative slope in both models, whereas the scalar effective
count has a small positive slope. The observed broadening is therefore weak
and not a readout-independent difficulty effect.

\paragraph{Interaction order remains essentially fixed.}
Neither architecture shifts toward higher-order interactions on hard images.
For ResNet-18, $\kappa^h$ changes by only $0.0036$ between the easiest and
hardest quartiles, and $C_{\le3}^h$ changes by $-0.000326$. For ResNet-34,
the corresponding changes are $-0.000388$ and
$-1.68\times10^{-5}$. The scalar effective order decreases slightly rather
than increasing, by $0.0600$ for ResNet-18 and $0.0060$ for ResNet-34. These
values, together with the near-zero correlations and class-controlled slopes,
provide no evidence that difficult inputs recruit systematically higher-order
residual interactions.

\paragraph{Misclassification yields the same qualitative conclusion.}
The wrong--correct comparisons in Figures~\ref{fig:app-exp5-r18}(f) and
\ref{fig:app-exp5-r34}(f), together with
Table~\ref{tab:app-exp5-correctness}, are also small. Misclassified ResNet-18
images increase $N_{\mathrm{eff}}^h$ by $0.1876$ and
$K_{.10}^{\mathrm{err}}$ by $0.7436$, while $\kappa^h$ changes by only
$0.0020$. For ResNet-34, the corresponding relative changes are $0.072\%$,
$0.007\%$, and $-0.031\%$. Thus, prediction failure is not accompanied by a
meaningful expansion of the magnitude-ranked reconstruction support or a
shift in interaction order.

\paragraph{Conclusion.}
For the studied checkpoints, residual-interaction complexity is approximately
stable across the difficulty spectrum. Hard images sometimes exhibit a small
increase in breadth or reconstruction budget, especially under the
decision-conditioned scalar readout, but the effect is modest, attenuates in
ResNet-34, and is not consistent across readouts after class control. The
stronger finding is the absence of an order shift. Combined with
Appendix~\ref{app:input-dependent-expert-sets}, these results distinguish two
forms of input dependence: the identities of prominent interactions vary with
the input and class, whereas the overall breadth and order of the interaction
profile remain nearly constant. This supports only a limited routing-like
interpretation at the level of functional contribution and does not support
a strong difficulty-adaptive interpretation or sparse execution.

\end{document}